\documentclass{article}
\usepackage[verbose=true,letterpaper]{geometry}
\AtBeginDocument{
  \newgeometry{
    textheight=9in,
    textwidth=6.5in,
    top=1in,
    headheight=14pt,
    headsep=25pt,
    footskip=30pt
  }
}

\usepackage[utf8]{inputenc}
\usepackage[T1]{fontenc}    
\usepackage{hyperref}      
\usepackage{url}          
\usepackage{booktabs}      
\usepackage{amsfonts}       
\usepackage{nicefrac}      
\usepackage{microtype}     
\usepackage{xcolor}   
\usepackage[section]{placeins}

\definecolor{gold}{rgb}{1.0, 0.84, 0.0}
\definecolor{silver}{rgb}{0.75, 0.75, 0.75}
\definecolor{bronze}{rgb}{0.8, 0.5, 0.2}
\definecolor{lightgrey}{rgb}{0.9, 0.9, 0.9}

\usepackage[english]{babel}
\usepackage{amssymb,amsmath,amsthm}
\usepackage{subfig}
\usepackage{graphicx}
\usepackage{colortbl}
\usepackage{authblk}
\usepackage{xspace}
\usepackage{cleveref}
\usepackage{thm-restate}

\newcommand{\Lap}{\text{Lap}}

\newcommand{\argmin}{\mathrm{argmin}}
\newcommand{\eps}{\epsilon}
\newcommand{\E}{\mathbb{E}}

\usepackage{mdframed}
\newmdtheoremenv{theomd}{Theorem}

\newtheorem{theorem}{Theorem}
\newtheorem*{theorem*}{Theorem}

\newtheorem{proposition}[theorem]{Proposition}
\newtheorem{lemma}[theorem]{Lemma}

\theoremstyle{definition}
\newtheorem{definition}[theorem]{Definition}
\newmdtheoremenv{defmd}[theorem]{Definition}

\newtheorem{example}[theorem]{Example}

\usepackage{algorithm}
\usepackage{algpseudocode}

\usepackage{wrapfig}
\usepackage{enumitem}
\usepackage{colortbl}

\usepackage{tikz}
\usetikzlibrary{arrows.meta}

\usepackage{graphicx}
\usepackage{subfig}

\usepackage[numbers]{natbib} 

\Crefname{fact}{Fact}{Fact}
\crefname{fact}{fact}{fact}

\usepackage{graphicx}
\usepackage{multirow}
\usepackage{booktabs}
\usepackage{adjustbox}

\newcommand{\bbeta}{\boldsymbol{\beta}}

\def\mbf{\mathbf}
\def\mc{\mathcal}
\def\SS{\mathcal{S}}

\def\argmin{\mathop{\rm argmin}}

\def\loss{\ell}

\def\obj{J}

\def\reg{\lambda}

\def\priveps{\epsilon}

\def\bbR{\mathbb{R}}

\def\loss{\ell}

\def\obj{J}

\def\reg{\lambda}

\def\surf{\textrm{surf}}
\def\priveps{\epsilon}

\def\extra{\Delta}
\def\c{c}

\def\D{\mathcal{D}}
\def\x{\mathbf{x}}

\def\b{\mathbf{b}}
\def\fpriv{\mathbf{\mbf{\beta}}_{\mathrm{priv}}}

\def\b{\mathbf{b}}
\def\D{\mathcal{D}}

\newcommand{\ignore}[1]{}

\newcommand{\kc}[1]{}
\newcommand{\ads}[1]{}
\newcommand{\cm}[1]{}

\newcommand{\norm}[1]{\left\lVert#1\right\rVert}

\title{Differential Privacy Under Class Imbalance:\\Methods and Empirical Insights\thanks{Y.L. supported in part by NSF grant CNS-2138834 (CAREER). M.A.M. supported in part by NSF grant DMS-2310973. R.C. supported in part by NSF grants CNS-2138834 (CAREER) and IIS-2147361. LR supported by the NSF under GRFP Grant No. DGE-2234660.
Work completed while Y.L. and E.T. were at Columbia University.\\ Emails: \texttt{lucas.rosenblatt@nyu.edu}, \texttt{\{ma3874,rac2239\}@columbia.edu}}}

\author{Lucas Rosenblatt$^1$, Yuliia Lut, Eitan Turok, Marco Avella‐Medina$^2$, Rachel Cummings$^2$ \\
$^1$New York University, $^2$Columbia University
}

\begin{document}

\maketitle

\begin{abstract}
Imbalanced learning occurs in classification settings where the distribution of class-labels is highly skewed in the training data, such as when predicting rare diseases or in fraud detection. This class imbalance presents a significant algorithmic challenge, which can be further exacerbated when privacy-preserving techniques such as differential privacy are applied to protect sensitive training data. Our work formalizes these challenges and provides a number of algorithmic solutions. We consider DP variants of \emph{pre-processing} methods that privately augment the original dataset to reduce the class imbalance; these include oversampling, SMOTE, and private synthetic data generation. We also consider DP variants of \emph{in-processing} techniques, which adjust the learning algorithm to account for the imbalance; these include model bagging, class-weighted empirical risk minimization and class-weighted deep learning. For each method, we either adapt an existing imbalanced learning technique to the private setting or demonstrate its incompatibility with differential privacy. Finally, we empirically evaluate these privacy-preserving imbalanced learning methods under various data and distributional settings. We find that private synthetic data methods perform well as a data pre-processing step, while class-weighted ERMs are an alternative in higher-dimensional settings where private synthetic data suffers from the curse of dimensionality.
\end{abstract}

\section{Introduction} \label{intro}
The problem of \emph{imbalanced learning} typically refers to classification tasks where one of the label-classes is substantially underrepresented in the training data. This occurs commonly in real-world applications, such as detecting fraudulent transactions \cite{makki2019experimental,mohammed2018scalable}, medical diagnostics for rare diseases \cite{singh2020imbalanced, yuan2018regularized}, or predicting natural disasters \cite{johnson2019survey}. Applying standard machine learning algorithms without adjustment can lead to poor predictions on rare events because these methods are designed for training data that are approximately balanced, or assume that false positives and false negatives have equal misclassification costs.  Consequently, such standard algorithms fail to both achieve good accuracy across binary classes and to represent the minority class in the resulting model. For example, consider the problem of detecting spam, which in 2009 accounted for 3\% of all posts on Twitter \cite{twitter2009}. A naive classification model could label every tweet as ``not spam,'' thus achieving 97\% predictive accuracy, but obviously failing to solve the core problem of spam detection.

In the machine learning community, the problem of imbalanced learning has been widely studied in non-private settings \cite{he2009learning, sun2009classification, chawla2004special}. This issue can be tackled with two main approaches. The first approach is to use \textit{pre-processing} techniques to balance the training dataset, such as oversampling \cite{chawla2002smote} or other data augmentation based methods. The second approach is to use \textit{in-processing} techniques to modify the machine learning model itself to account for the imbalance, such as bagging \cite{breiman1996bagging} or loss re-weighting \cite{karakoulas:ShaweTaylor1998,zhou:liu2005,rigollet:tong2011,scott2012,tong2013,xu:chen:khim:ravikumar2020}.

In applications when data are both sensitive and imbalanced -- for instance, in the detection of rare diseases \cite{ficek2021differential} -- we need machine learning tools that preserve privacy while maintaining high accuracy. Differential privacy (abbreviated DP) has emerged as a powerful technical definition in machine learning and theoretical computer science to address privacy concerns using formal algorithmic tools.  Many traditional machine learning algorithms have DP implementations \cite{gong2020survey,chaudhuri2011differentially,abadi2016deep, fletcher2017differentially}. However, extending such approaches to private data augmentation methods for imbalanced settings encounters two main challenges. Firstly, it has been shown that private classifiers can amplify minority group loss, magnifying bias and unfairness \cite{tran2021differentially, bagdasaryan2019differential, pujol2020fair, rosenblatt2024simple}. Secondly, pre-processing techniques such as oversampling run the risk of increasing sensitivity of the learning task with respect to the original database, thus increasing privacy loss. One must therefore be careful when designing privacy-preserving techniques for imbalanced learning that improve performance with respect to the minority class without over-inflating the privacy budget. 

\subsection{Our Contributions} 
In this work, we explore both \textit{pre-processing} and \textit{in-processing} methods for private imbalanced binary classification (and note that many of our methods have natural extensions to the multi-class settings). In Section~\ref{sec:preprocessing}, we account for the privacy degradation caused by the well-known data-augmentation technique SMOTE \cite{chawla2002smote}, showing that its privacy loss scales as $\Theta(\eps 2^d k)$ for $d$-dimensional data and $k$ new data points (unacceptably large for practical settings). This motivates an alternative: using black-box DP synthetic data techniques for augmenting minority data, which is trivially private via post-processing (Proposition~\ref{thm.post-proc}) and empirically effective. 

We then shift our focus to in-processing methods in Section \ref{sec:inprocessing}. We first consider model bagging \cite{breiman1996bagging}, a technique often used in imbalanced learning, which trains many weak learners on subsets of the data, and aggregates votes for overall predictions. Prior work on this method \cite{liu2020intrinsic} has claimed an ``intrinsic'' DP guarantee, due to the noise from sampling. We show that while this method satisfies $(\eps,\delta)$-DP for \emph{some} $(\eps,\delta)$, the resulting parameters are not useful in any practical settings (Proposition \ref{prop:bagging}). For a positive result, we then show how to leverage existing private learning to perform \textit{class-weighted} risk minimization. We adapt the canonical private ERM \cite{chaudhuri2011differentially} to a class-weighted variant in \Cref{alg:objective}, and show how DP-SGD trivially allows for class weights \Cref{alg:dpsgd}. We give privacy arguments for each, respectively, in \Cref{thm.ermprivacy} and \Cref{prop.ftt_private}.

Finally, we provide experimental results for several of the methods discussed above on synthetic multi-variate mixture models 
evaluated on standard imbalanced learning benchmarks \cite{lemaavztre2017imbalanced}. We find that a pre-processing method
(using a strong non-private model (XGBoost) trained on privatized synthetic data) performs best on average over the metrics we considered. The private weighted ERM (with a logistic regression model) outperforms its \textit{unweighted} variant in metrics suitable for imbalanced classification on average, while the private weighted neural model (FTTransformer) trained with DP-SGD  
underperforms despite its strong non-private performance, suggesting that neural models may not be ideal for small to medium-sized privacy-preserving imbalanced classification tasks on tabular data. 

\subsection{Related Work}

\textbf{Imbalanced learning and privacy.} The problem of imbalanced data often arises in machine learning when the size of one data class is considerably smaller than the other data class. Prior work on imbalanced learning without privacy constraints is extensive \cite{chawla2004special,he2009learning, sun2009classification, galar2011review, krawczyk2016learning,lopez2013insight, branco2016survey}, alongside work studying adjustments to common learning losses for imbalanced classification \cite{scott2012, menon2013statistical}. The challenge of handling imbalanced data in machine learning becomes much harder when privacy constraints are added, as accuracy for the minority class can be low even for non-private classification \cite{lau2021statistical}. Additionally, prior work shows that differentially private algorithms can disproportionally affect minority groups by amplifying the lost of accuracy of a minority class \cite{bagdasaryan2019differential,jaiswal2020privacy} as well as magnify bias and unfairness \cite{xu2020removing, pujol2020fair, farrand2020neither,tran2021differentially}. 
Work by \cite{jordon2019differentially} studies bagging under differential privacy with the assumption of a publicly available data sample; we do not make any such assumptions, and thus can operate in the most general settings.
In \cite{lau2021statistical} the authors emphasize a lack of private methodologies for pre-processing techniques. In their empirical assessment, they show that resampling on imbalanced data leads to privacy leakage, and that a higher ratio of oversampling corresponds to increased privacy loss. 

\textbf{Private synthetic data generation.} 
There has been much progress in recent years on  methods for differentially private data synthesis and generation \cite{aydore2021differentially,boedihardjo2022private,cai2021data,mckenna2019graphical,rosenblatt2020differentially,vietri2020new,zhang2021privsyn, boedihardjo2024private}. 
Some of the best-performing methods follow the \emph{Select-Measure-Project} paradigm \cite{tao2021benchmarking, mckenna2022aim}; these algorithms first \emph{select} highly representative queries to evaluate on the data, \emph{measure} those queries in a differentially private manner (often with standard additive noise mechanisms), and then \emph{project} the private measurements onto a parametric distribution, which can then be used to generate arbitrarily many new datapoints.

\textbf{Differential privacy and sampling.} 
It is known that randomly \emph{undersampling} the input database before running a private mechanism can improve the privacy guarantees (known as \emph{amplification by subsampling}) \cite{bun2015differentially, wang2016minimax, bassily2014private}. In \citeauthor{bun2022controlling}~\cite{bun2022controlling}, the authors showed that more complex, data-dependent subsampling can negatively affect the privacy guarantees.
Differentially private GANs have also been studied in the context of oversampling \cite{sun2022improving}. The evaluation of a number of state-of-the-art private generative models shows that stronger privacy guarantees can intensify the imbalance in the data or simply offer a lower quality synthetic data \cite{ganev2022robin, cheng2021can}.

\textbf{Differential privacy with deep learning.} The advent of differentially private gradient descent via gradient clipping and moments accounting \cite{bassily2014private, abadi2016deep} has led to privatized versions of many standard deep learning models that exhibit strong empirical performance \cite{gong2020survey, opacus}. Recent work has expressed skepticism over the performance of these methods, hypothesizing instead that their strength can be partially explained by the effect of unreported hyper-parameter tuning in a ``dishonestly'' private manner \cite{papernot2021hyperparameter, redberg2024improving}. Thus, we solely run our models with default hyper-parameter settings for a fair comparison.

\section{Preliminaries}\label{sec:prelims}

\paragraph{Imbalanced Learning.}
Let $D = (X, y)$ denote a dataset, where $X$ is a set of $d$-dimensional instances from a known range $[-R,R]^d$ and $y$ is a vector of binary labels. Each $(x_i, y_i) \in [-R,R]^d \times \{0, 1\}$ is a single labeled training example.\footnote{We assume that $X$ lies in a bounded range because this is necessary for differentially private regression (see, e.g., \cite{chaudhuri2011differentially}). If a bound on the data points is not known \emph{a priori}, then one can be guessed using domain knowledge or using other private methods such as Propose-Test-Release \cite{DL09}.} We partition $X$ into $X^0$ and $X^1$, respectively denoting the sets of entries of $X$ that are labeled with 0 and 1, where these sets are of size $|X^0|=n_0$ and $|X^1|=n_1$. To model the \emph{imbalanced} setting, we assume, without loss of generality, that $n_1 \ll n_0$. 

We define $r=\frac{n_0}{n_1} >1$ to be the \textit{imbalance ratio} between the positive and negative label classes in the sample. It is common to assume that $n_1$ is too small to enable learning directly on the minority class (e.g., $n_1$ is smaller than the sample complexity of the learning task of interest). While we do not explicitly make this assumption in our work, we use it as a motivation for studying oversampling methods. 

The goal of imbalanced learning is to develop a binary classifier that accurately learns from the imbalanced dataset $D$. In other words, we seek to learn a function $f: [-R, R]^d \to \{0, 1\}$ by minimizing a given loss function $\mathcal{L}$ weighted by class imbalance in the training label distribution, or maximizing an imbalanced performance metric of interest (e.g., F1 Score, Recall, etc.).

\paragraph{Differential Privacy.}
Differential privacy limits the effect of any individual's data on a computation and ensures that little can be inferred about the individual from an appropriately calibrated randomized output. Intuitively, it bounds the maximum amount that a single data entry can affect analysis performed on the database. Two databases $D,D'$ are \emph{neighboring} if they differ in at most one entry. In this work, we present results for the \textit{bounded} variant of neighboring datasets, i.e., neighboring datasets are the same size, $|D| = |D'|$, and are identical except for a single entry. All of our results can be extended to the \textit{unbounded} variant, i.e., where $D'$ can be constructed through addition/removal, so $|D| = |D'|\pm1$ \cite{kifer2011no}.
\begin{definition}[Differential Privacy \cite{dwork2006calibrating}]\label{def.dp}
	An algorithm $\mathcal{M}: \mathcal{D} \rightarrow  \mathbb{R}$ is \emph{$(\epsilon,\delta)$-differentially private} if for every pair of neighboring databases $D,D' \in \mathcal{D}$, and for every subset of possible outputs $\mathcal{S} \subseteq \mathbb{R}$,
	\begin{align*} \Pr[\mathcal{M}(D) \in \mathcal{S}] \leq \exp(\epsilon)\Pr[\mathcal{M}(D') \in \mathcal{S}]+\delta. \end{align*}
When $\delta=0$, $\mathcal{M}$ may be called $\epsilon$-differentially private.
\end{definition} 

One well-known technique for achieving $(\epsilon,0)$-DP is by adding Laplace noise. The \emph{Laplace distribution} with scale $b$ is the distribution with probability density function: $h(x|b) = \frac{1}{2b} \exp(-\frac{|x|}{b})$.  
The scale of noise should depend on the \emph{sensitivity} of a computation being performed, which is the maximum change in the function's value that can be caused by changing a single entry in the database. Formally, the sensitivity of a real-valued function $f$ is defined as: $\Delta f = \max_{\text{neighbors } D, D^\prime} | f(D) - f(D^\prime)|$. The Laplace Mechanism of \cite{dwork2006calibrating} takes in a real-valued function $f$, a database $D$, and a privacy parameter $\eps$, and produces the (random) output: $f(D) + \Lap(\Delta f/\eps)$. 

Alternatively, one can use the the Gaussian mechanism to achieve  $(\epsilon,\delta)$-DP by considering the L2 sensitivity of the function (i.e. $\Delta_2 f = \max_{\text{neighbors } D, D^\prime} \| f(D) - f(D^\prime)\|_2$, where $\| \cdot \|_2$ denotes the Euclidean norm) and adding noise sampled from $\mathcal{N}\left(\mu=0, \sigma^2 = (\Delta_2 f)^2 \cdot \frac{2 \log (\frac{1.25}{\delta})}{\epsilon^2}\right)$. The Gaussian mechanism further requires that $\epsilon < 1$ for the privacy guarantees to hold. In settings where data points can be unbounded, \emph{clipping} can be applied to project each $X_i$ in the range $[-R, R]$; doing so reduces the sensitivity of the function, and hence the scale of noise that must be added. 

Differential privacy has a number of helpful properties. It \emph{composes} (Theorem \ref{thm.basic}), meaning that the privacy parameter degrades gracefully as additional computations are performed on the same database. It is also robust to \emph{post-processing} (Theorem \ref{thm.post-proc}), meaning that any further analysis on the output of a differentially private algorithm cannot diminish the privacy guarantees.

\begin{theorem}[Basic Composition \cite{dwork2006calibrating}]\label{thm.basic}
Let $\mathcal{M}_1$ be an algorithm that is $(\epsilon_1,\delta_1)$-DP, and let $\mathcal{M}_2$ be an algorithm that is $(\epsilon_2,\delta_2)$-DP.  Then their composition $(\mathcal{M}_1,\mathcal{M}_2)$ is $(\eps_1 + \eps_2,\delta_1+\delta_2)$-DP.
\end{theorem}

\begin{theorem}[Post-processing \cite{dwork2006calibrating}]\label{thm.post-proc}
Let $\mathcal{M}: \mathcal{D}\rightarrow \mathcal{R}$ be an algorithm that is $\epsilon$-differentially private, and let $f : \mathcal{R} \rightarrow \mathcal{R}^\prime$ be an arbitrary function. Then $f \circ \mathcal{M} :\mathcal{D} \rightarrow R^\prime$ is $\eps$-differentially
private.
\end{theorem}

\section{Pre-processing Methods for Private Imbalanced Learning}\label{sec:preprocessing}
In this section, we consider applying pre-processing methods for data augmentation to address class imbalance: given a level of class imbalance in the training data, augment or replace the dataset to increase support for the minority class. After applying a pre-processing method, we can then privately learn a classifier on the augmented dataset. The first two methods we consider -- \emph{oversampling} in Section \ref{s.oversampling} and \emph{SMOTE} in Section \ref{s.smote} -- are non-private pre-processing methods; we show that both of these methods \textit{substantially} increase the sensitivity of the downstream private learning mechanism. This increase in sensitivity is due to the fact that these methods generate synthetic minority samples that are highly dependent on the original data, so changing one input point in the original database may lead to \textit{many} points being changed in the augmented database.
This motivates our consideration of \textit{private synthetic data generation} for data pre-processing in Section \ref{s.privsynthdata}. In the case of private synthetic data, we instead perform our privacy intervention upstream, learning a differentially private parameterization of the distribution of our data, from which we can draw arbitrary samples for downstream, 
non-private model training, while still maintaining an $(\epsilon,\delta)$-DP guarantee.

\subsection{Oversampling}\label{s.oversampling}
A common technique for dealing with class imbalance in data is to apply an \emph{oversampling} algorithm that first generates $N$ additional synthetic samples from the minority class, before performing learning on the augmented dataset. The learning algorithm then takes as input the original dataset $D=(X,y)$, concatenated with the $N$ new minority class (positive label) samples. While $N$ can be chosen freely by the analyst, a common parameter regime is to choose $N=n_0-n_1$ to equalize the size of the two classes. A simple oversampling method is to replicate each minority point in $X_1$ either $\lceil N/n_1 \rceil$ or $\lfloor N/n_1 \rfloor$ times to ensure $N$ total new points; we refer to this as \emph{deterministic oversampling}.\footnote{One could also randomly and independently sample a point to replicate from the minority class $N$ times, also exacerbating downstream sensitivity. For simplicity of presentation we stick with deterministic oversampling.
}

As formalized in Proposition \ref{prop.oversampling}, deterministic oversampling increases sensitivity of any downstream DP learning algorithm by a multiplicative factor of $\lceil N/n_1 \rceil +1$. This is because the maximum of $\lceil N/n_1 \rceil$ additional samples generated from each minority point, along with the minority point themselves, will all be used in the downstream learning algorithm.

\begin{proposition} \label{prop.oversampling}
Let $D=(X,y)$ be a dataset with $n_1$ minority instances, and let $\mathcal{M}$ be an arbitrary ($\epsilon,\delta$)-DP algorithm. Instantiating $\mathcal{M}$ on the dataset $D$ concatenated with the output of oversampling to generate $N$ additional minority samples is $(\eps  (\lceil\frac{N}{n_1}\rceil+1), \delta (\lceil\frac{N}{n_1}\rceil+1))$-differentially private. 
\end{proposition}

\subsection{SMOTE}\label{s.smote}

The Synthetic Minority Oversampling TEchnique (SMOTE, Algorithm \ref{algo:smote}) \cite{chawla2002smote} is a more advanced oversampling technique and has become a benchmark for imbalanced learning (see, e.g., \cite{fernandez2018smote} for a survey of the algorithm's impact in the field). For $N$ iterations, the algorithm: (1) selects an instance from the minority class, (2) finds the $k$ nearest neighbors of this point under $\ell_2$ distance and samples one uniformly at random, and (3) generates a new minority instance as a random convex combination of the original instance and its selected nearest neighbor. 

\begin{algorithm}[H]
\caption{SMOTE($X_1$, $N$, $k$) \cite{chawla2002smote}} \label{algo:smote}
\begin{algorithmic}
\State \textbf{Input:} minority class instances $X_1=\{x_1, \ldots x_n\}$, dataset dimension  $d$, number of points to be generated $N$, number of nearest neighbors $k$.
\State \textbf{Output:} $N$  synthetic minority class samples
	\For {$i=1,\ldots,n$}
	\State Compute $k$ nearest $\ell_2$ neighbors of $x_i$ from $X_1$:  $(x_i^1, \ldots , x_i^k)$
 \EndFor
	\For {$t=1,\ldots,N$}
        \State $i=t \mod n$, where $0\mod n$ is interpreted as $n$
	\State Randomly choose $x_i^\prime$, one of the $k$ nearest neighbors of $x_i$
	\For {$j=1,\ldots d$}
	\State Sample $u_j$ uniformly from $[0,1]$ 
	\State $z_{t,j}=(1-u_j) x_{i,j}'+u_j \cdot x_{i,j}$ 
 	\EndFor
	\State \textbf{return} $(z_t,1)$
	\EndFor
\end{algorithmic}
\end{algorithm}

Unfortunately, ~\Cref{thm.smote} shows that applying SMOTE as a pre-processing step before any differentially private algorithm substantially increases the sensitivity of the downstream computation: the increase in effective epsilon is exponential in $d$ and linear in $N$. This dramatic increase in the $\epsilon$ factor, if unaccounted for, leads to an overall $\epsilon'$-DP guarantee for extremely large $\epsilon'$ values that provide meaningless privacy guarantees. 

\begin{restatable}{theorem}{smote}\label{thm.smote}
Let $D=(X,y)$ be a $d$-dimensional dataset, with $n_1$ minority instances, and let $\mathcal{M}$ be an arbitrary $\epsilon $-DP algorithm. Then instantiating $\mathcal{M}$ on $D$ concatenated with the output of SMOTE$(X, N, k)$ is both $(\epsilon (2^{0.4042d}\lceil\frac{N}{n_1}\rceil+1),0)$-DP and $\left( \epsilon^\prime, \delta\right)$-DP, for any $\gamma\ge 0$ and for,
 \[\epsilon^\prime = \epsilon (1+\gamma) 2^{0.4042d} \left\lceil\frac{N}{n_1}\right\rceil\frac{1}{k}, \text{ and } \delta = e^{k 2^{0.4042d}\lceil\frac{N}{n_1}\rceil\left( \epsilon -\frac{\gamma^2}{k(2+\gamma)}\right)}.\]
\end{restatable}

A full proof of Theorem \ref{thm.smote} is deferred to Appendix \ref{app:smote}. As a brief sketch, we first define the quantity $Y=|SMOTE(X, N, k)\oplus SMOTE(X^\prime, N, k)|$ which gives the symmetric difference between SMOTE applied to two neighboring datasets $X,~X^\prime$, where $\oplus$ denotes symmetric difference. $Y$ can be fully described as a sum of Bernoulli random variables with parameters that depend on $k$, $N$, $n_1$, and the maximum number of times one point from $\mathbb{R}^d$ can appear among $k$-nearest neighbors of other points from $\mathbb{R}^d$. SMOTE only takes in the minority class data, and does not use majority class data at all in generating new synthetic data. Thus, without loss of generality, \Cref{thm.smote} only considers the modification of a minority class example that has a positive label; if the minority class was actually the negative label, this could be dealt with in the analysis simply by renaming. 

Now, to bound the maximum number of times one point can appear among $k$-nearest neighbors of other points, denoted $l(d,k)$, we require \Cref{lem.smotestab}. This lemma lower bounds $l(d,k)$ via a geometric argument that relies on the notion of a \emph{kissing number} $K(d)$, defined as the greatest number of equal non-overlapping spheres in $\mathbb{R}^d$ that can touch another sphere of the same size \cite{musin2008,jenssen:joos:perkins2018}.

\begin{restatable}{lemma}{smotestab}\label{lem.smotestab}
Let $l(d, k)$ be the maximum number of times one point from $\mathbb{R}^d$ can appear among the $k$-nearest neighbors of $n_1$ other points from $\mathbb{R}^d$. Then, $l(d, k) = \min\{k \cdot K(d), n_1\}$.
\end{restatable}

The exact value of the kissing number $K(d)$ for general $d$ is an open problem, but is known to be asymptotically bounded by $k 2^{0.2075d(1+o(1))} \le l(d,k) \le k 2^{0.4042d}$ \cite{wyner1965,musin2008,kabatiansky:levenshtein1978}. 
Returning to $Y$, we then apply a one-sided Chernoff bound constraining the probability that $Y$ is much greater than its mean. Plugging in \Cref{lem.smotestab} and appropriate parameters yields \Cref{thm.smote}.

\Cref{thm.smote} should be viewed as a negative result (i.e., SMOTE makes ensuring downstream privacy very difficult). With only the $(\epsilon,0)$-DP result, one might wonder whether the large increase in epsilon can be avoided by allowing a positive $\delta$. Thus, we include and highlight the $(\epsilon, \delta)$ result, which shows that this is not the case; even when a strictly positive failure probability $\delta>0$ is allowed, the explosion in $\epsilon$ is still present (albeit reduced by a $1/k$ factor). Intuitively, we frame the result as follows: introducing new, minority class examples based on linear interpolations of existing minority class examples leads to \textit{significantly} higher privacy sensitivity, which makes the method impractical to run. See \Cref{table:adjusted_eps} for an example of how large practical $\epsilon$ values can become, after adjusting for the sensitivity of SMOTE preprocessed data. We further note that this negative result has implications for more advanced class-imbalanced methods that embed the SMOTE algorithm, like SMOTEBoost \cite{chawla2003smoteboost} and SMOTEBagging \cite{wang2009diversity}. 

\begin{table}[tbh]
\centering
\resizebox{0.85\linewidth}{!}{
\begin{tabular}{@{}rccc|ccc@{}}
\toprule
 & \multicolumn{3}{c}{Input $\epsilon'$ required to achieve desired $\epsilon$} & \multicolumn{3}{c}{Resulting $\epsilon$ from unadjusted input $\epsilon'$} \\
\midrule \midrule
 & $\epsilon = 1$ & $\epsilon = 5$ & $\epsilon = 10$ & $\epsilon' = 1$ & $\epsilon' = 5$ & $\epsilon' = 10$ \\
\midrule
SMOTE & $0.00469$ & $0.02346$ & $0.04692$ & $213.21$ & $1066.06$ & $2132.1$ \\
\bottomrule
\end{tabular}
}
\vspace{0.2cm}
\caption{SMOTE requires a dramatic adjustment to the privacy parameter. The table shows (left) the adjusted values of input privacy parameter $\epsilon'$ to the differentially private algorithm for varying desired privacy budgets $\epsilon$, and (right) the resulting privacy budgets $\epsilon$ if $\epsilon'$ is unadjusted. Calculations make mild, practical assumptions like $\delta = 1/n^2$ with $n = 10000$, dimension $d = 25$, $k = 5$, $\gamma = 0$, and $\left\lceil \frac{N}{n_1} \right\rceil = 1$.}
\label{table:adjusted_eps}
\vspace{-0.3cm}
\end{table}

In \Cref{app:smote}, we also empirically demonstrate the poor performance of SMOTE with the proper sensitivity adjustment. 

\subsection{Private Synthetic Data}\label{s.privsynthdata}

We have shown that \textit{non-private} data augmentation techniques for imbalanced learning, like oversampling and SMOTE, explode downstream privacy parameters by amplifying sensitivity. \textit{Private} data pre-processing avoids this limitation. Specifically, we propose leveraging existing private synthetic data algorithms (e.g., \cite{tao2021benchmarking, mckenna2022aim, rosenblatt2024epistemic, sen2024diverse}) to produce a private balanced dataset that is usable for learning. 

Many existing methods for producing synthetic data with a differentially private guarantee follow the Select-Measure-Project paradigm. That is, these methods first select differentially private measurements to evaluate on the data (Select), compute these measurements on the sensitive data (Measure),  and finally fit a new distribution to those measurements (Project) \cite{liu2021iterative}. New samples can then be drawn from the new private distributional model to combat data imbalance -- one simple and general approach is to draw enough new samples of the minority class to balance the size of both classes in the dataset. Note that arbitrarily many samples can be drawn from the privately fitted distributional model without affecting the differential privacy guarantees due to a post-processing (Theorem \ref{thm.post-proc}).

A formal, general version of this procedure is given in~\Cref{algo:rejection_sampling}. Any private synthetic data generation method could be substituted in Stage 1 of \Cref{algo:rejection_sampling}. Note that \Cref{algo:rejection_sampling} is stated generally and is not entirely black box; the algorithm defaults to performing \textit{conditional} sampling to up-sample the minority class for parametric models (i.e. condition a new generated sample on a fixed positive or negative feature label), as this is sample efficient. For non-parametric models, one can take a more general rejection sampling approach, which is also given as alternate behavior in \Cref{algo:rejection_sampling}.

In \Cref{app:complete_experiments}, we provide an experimental comparison between the \emph{PrivBayes} and \emph{Generative Networks with the Exponential Mechanism} (GEM) methods, two state-of-the-art private synthetic data methods \cite{zhang2017privbayes, liu2021iterative}. Both GEM and PrivBayes are parametric models and thus permit conditional class sampling. Most of our empirical results in \Cref{sec:experiments} are then given with GEM for clarity of presentation, as we found that it outperformed PrivBayes across the board.

\begin{algorithm}[h!]
\caption{Balancing w/ Private Data Synthesizer}\label{algo:rejection_sampling}
\begin{algorithmic}
\State \textbf{Input:} $(\epsilon, \delta)$-differentially private data synthesizer $\mathcal{S}$, original dataset $D$, desired number of samples $N$, and any additional parameters for $\mathcal{S}$, $\mathcal{P}$.
\State \textbf{Output:} A balanced dataset $D'$ where $n_{0} = n_{1}$.
\State \textbf{Stage 1: Parameterize a Distribution} 
\State Learn/parameterize a differentially private distribution $\theta$ over the data domain i.e. $\theta \leftarrow \mathcal{S}(D, \mathcal{P})$. 
\State \textbf{Stage 2: Sample a New Dataset $D'$}
\If{$\theta$ is parametric}
    \State Sample $N/2$ minority examples $D'_{n_1} \sim \theta~|~n_{1}$, then sample $N/2$ majority examples $D'_{n_0} \sim \theta~|~n_{0}$. 
    \State \textbf{return} concatenation $[D'_{n_1}, D'_{n_0}]$.
\ElsIf{$\theta$ is non-parametric}
    \State Perform rejection sampling based on class label to draw balanced samples (i.e., ensure $n_0 = n_1 = N/2$ in the final dataset $D'$ by sampling from $\mathcal{S}(D)$ until target sizes are reached).
\EndIf
\State \textbf{return} $D'$
\end{algorithmic}
\end{algorithm}

The privacy of \Cref{algo:rejection_sampling} is straightforward to see: as long as the data synthesizer in Stage 1 is $(\epsilon,\delta)$-DP, then Stage 2 will retain the same privacy guarantee by post-processing (\Cref{thm.post-proc}).

\begin{proposition}\label{prop:privacy_synth}
    \Cref{algo:rejection_sampling} is $(\epsilon,\delta)$-differentially private. 
\end{proposition}

\section{In-processing Methods for Private Imbalanced Learning}\label{sec:inprocessing}

In-processing methods account for class imbalance by adjusting the learning process. They broadly fall into two main categories: \emph{ensemble-based} classifiers and \emph{cost-sensitive} classifiers. Our first in-processing method we consider in \Cref{s.bagging} is bagging, which is an ensemble-based classifier over splits of the training data.
We show that although bagging non-private learners does provide some inherent privacy, the resulting DP parameters are \textit{not} meaningful in practice. Cost-sensitive classification assumes a greater cost to misclassifying minority class examples in the training data \cite{chawla2004special}; the primary approach to accommodate asymmetric misclassification costs are weighting strategies during model training. In Section \ref{s.weightederm}, we revisit canonical results from \cite{chaudhuri2011differentially} on differentially private empirical risk minimization (ERM) and show how to introduce sample weights. Finally, in \Cref{s.weighteddpsgd}, we show that the widely-used differentially private stochastic gradient descent (DP-SGD) methods for deep learning can easily accommodate sample weighting based on class membership.

\subsection{Bagging and Private Bagging}\label{s.bagging}

Bagging is used widely in practice in imbalanced learning settings, as it has been shown to foster more diversity in model parameters and may help mitigate overfitting to the majority class by elevating minority class importance in the bootstrapped training subsets.  This empirical strength, robustness, and improved bias-variance tradeoff of bagging techniques in imbalanced learning is well known \cite{ueda1996generalization, moniz2017evaluation, haixiang2017learning}.

The standard bagging procedure \cite{breiman1996bagging} is as follows: create $m$ subsamples $\{D_1,...,D_m\}$ of a training dataset $D$ by randomly subsampling $k$ examples from $D$ (with or without replacement) to constitute each $D_i$. Then train a base model on each subsample $D_{i}$ using a base weak learner. To generate a prediction $\hat{y}_i$ for a given sample $X_i$, predict $\tilde{y}_i$ with each weak learner, and take the majority vote.

Since the bagging procedure is randomized, recent work by \cite{liu2020intrinsic} has suggested that it is \emph{intrinsically} differentially private, based on the randomness in sampling and in the predictions of the weak learners, which would imply that bagging is a potential in-processing method for handling imbalanced data. Specifically, \cite{liu2020intrinsic} showed that for a dataset of size $n$, bagging with parameters $(m,k)$ satisfies $(\eps,\delta)$-DP for $\epsilon = m \cdot k \cdot \ln(\frac{n+1}{n})$ and $\delta = 1 - (\frac{n-1}{n})^{m \cdot k}$. 

However, we highlight a significant issue with this approach, simply by inverting the parameter expressions, and solving for $m$ and $k$ given commonly desired settings of $\epsilon$ and $\delta$, namely that $\delta$ is polynomially small in $n$.\footnote{Many even prefer a stronger requirement, which is that $\delta$ is cryptographically small, or \emph{negligible}, in $n$.} In \Cref{prop:bagging}, we show that this re-parameterization reveals a major issue: we cannot set $\delta$ to be very small without setting $\epsilon$ to be exceedingly small as well; the simple proof of this is given in \Cref{app:bagging}.

\begin{restatable}{proposition}{bagging}\label{prop:bagging}
For a bagging classifier composed of non-differentially private learners to achieve $\delta = n^{-c}$, then it must also be that $\epsilon \leq \frac{1}{n}$, for all $c>1$.
\end{restatable}

Such a small $\eps$ value, paired with a constant-sensitivity function, would not allow the private output to sufficiently vary across different databases, even if they differ in many datapoints, meaning that the private output cannot provide meaningful accuracy. Therefore, non-private classifiers \textit{cannot} be used in bagging procedures to simultaneously provide meaningful privacy and accuracy guarantees. 

One approach to improving private bagging would be to use \textit{private} classifiers as the weak learners; in that setting, the privacy would follow easily via composition over all the private classifiers used. Given a dataset $D$ and a bagging procedure that trains $m$ $(\epsilon, \delta)$-DP regression models, then by advanced composition \cite{dwork2014algorithmic}, for any $\delta' > 0$, this version of private bagging would satisfy $(\epsilon', m\delta + \delta')$-DP for $\epsilon' = \sqrt{2m \ln(1/\delta')} \cdot \epsilon + m\epsilon(e^\epsilon - 1)$. As we show empirically in \Cref{app:bagging} (\Cref{fig:bagging_mammography_rebuttal}), this can still result in poor empirical performance in reasonable settings. One explanation is that since many private weak learners are needed, the privacy budget is ``spread too thinly'' over all the classifiers. That is, to satisfy a desired $\epsilon'$ privacy budget, the per-learner privacy parameter $\epsilon$ has to be small, thus significantly reducing performance.

Tighter composition analyses exist based on moments accountants \cite{abadi2016deep, wang2019subsampled}, where the dataset is also subsampled for each computation. These methods are most effective when only a small fraction of the dataset are included in each subsample; to contrast, many bagging procedures rely on much larger sub-samples disbursed among fewer learners \cite{sun2015novel}. In \Cref{app:bagging}, we also show that using a moments accountant for private bagging also did not result in good performance under class imbalance. Although the composition guarantees were improved, the subsampling created an additional issue in the presence of class imbalance: since so few minority class examples existed in the dataset, subsampling further reduced the number of minority examples available to each weak learner.

\subsection{Weighted Approaches}\label{s.weightederm}

Cost-sensitive classification assumes a greater \textit{cost} to misclassifying minority class examples and is a well-studied and practically effective method for combating class imbalance. \textit{Weighting} strategies during model training are the primary approach used to accommodate misclassification costs \cite{chawla2004special}. In \Cref{sec:warmup}, we motivate private cost-sensitive classification under a known distribution, and then in \Cref{sec:weighted_erm} we show how to adapt the private ERM given in \cite{chaudhuri2011differentially} under a bounded weighting scheme. Later in \Cref{s.weighteddpsgd} we show that DP-SGD can be modified to accommodate weights naturally.

\subsubsection{Warm-up: A Known Population}
\label{sec:warmup}
As a warm-up, we quantify the estimation error of the Bayes optimal classifier for a \textit{known} Gaussian mixture.

\begin{example}
\label{ex:warmup}
Let $\{X_i,y_i\}_{i=1}^n\in\mathbb{R}^{d=1}\times \{0,1\}^n$ be randomly sampled such that $X$ is a mixture of Gaussians and $y$ is a binary class label. Specifically, let $\{X_i~|~y_i=1\} \sim \mathcal{N}(\mu_1, \sigma^2)$ and $\{X_i~|~y_i=0\} \sim \mathcal{N}(\mu_0, \sigma^2)$.
The domain of $X$ here is \textit{a priori} unbounded, but we can later bound $X$ with clipping to reduce sensitivity.
\end{example}
This setting was also studied in \cite{yang2020rethinking}, who showed that the Bayes optimal classifier is given by $f_\theta(X)=\mathbb{I}(X \geq \theta)$
for $\theta =(\mu_0+\mu_1)/2$ (see \cite{hart2000pattern} for a textbook treatment). 
That is, assign the positive label if and only if $X>\theta$. We construct a private estimate of $\theta$ to build intuition for the effect of noise on imbalanced learning. 

The private classification mechanism $\mathcal{M}_{BOC}:\mathbb{R}\mapsto \{0,1\}$ makes private estimates of $\mu_0,~\mu_1$ by first clipping each $X_i$ to lie in the range $[-R,R]$ before applying the Gaussian mechanism to the clipped data to compute the empirical mean.\footnote{This is the canonical private mean estimator, but we note that improved methods exist \cite{biswas2020coinpress, kulesza2023mean, rosenblatt2024simple}.} Formally, define,
 \[
\hat{\mu}_b = \frac{1}{n_b} \sum_{i=1}^{n_b} \textsc{clip}(X_i, R) + \mathcal{N}\left(0, (\frac{2R}{n_b})^2 \cdot \frac{2 \log (\frac{1.25}{\delta})}{\epsilon^2}\right),
\]
for $b \in \{0,1\}$, where $\textsc{clip}$ denotes clipping $X_i$ into the range $[-R, R]$.
Then a natural mechanism for privately computing the Bayes Optimal Classifier is 
$\mathcal{M}_{BOC}(X) = \mathbb{I}(X\geq \hat{\theta}) \text{ for } \hat{\theta} = \frac{\hat{\mu}_1 + \hat{\mu}_2}{2}.$

\begin{restatable}{proposition}{warmup}\label{prop:warmup}
The mechanism $\mathcal{M}_{BOC}$ is $(2\epsilon,2\delta)$-differentially private. Assume $\max\{|\mu_1|, |\mu_2|\} \leq B$ for some known bound $B$ and $R>B+\sigma\sqrt{2\log(4n/\beta)}$. For any imbalance ratio $r \geq 1$, with probability at least $1-\beta/2$, the $\hat{\theta}$ produced by $\mathcal{M}_{BOC}$ satisfies
 \[
    \left|\hat{\theta} - \theta\right| \leq 2\sqrt{\log(4/\beta)} \sqrt{\frac{\sigma^2}{n_0}(1+r)+\frac{2R^2 \log(1.25/\delta)}{n_0^2 \epsilon^2} \cdot (1+ r^2)}~.
\]
Furthermore, for any estimator $\tilde\theta$ of $\theta$, with probability at least $1-\beta/2$,
$$|\tilde{\theta}-\theta|\geq \sigma\sqrt{\tfrac{(1+r)}{n_0}}\Phi^{-1}(1-\beta/2),$$
where $\Phi(\cdot)$ denotes the cumulative distribution function of a standard normal distribution.
\end{restatable}

The full proof of~\Cref{prop:warmup} is given in Appendix \ref{app:warmup}, although the proof is relatively straightforward. Privacy guarantees follow from the Gaussian Mechanism. For the accuracy guarantee, we first provide a high probability bound on the potential affects of clipping the data to $R$, and then provide a high-probability error bound accounting for the noise added to each of the $(\epsilon, \delta)$-differentially private estimates $\hat{\mu}_0$ and $\hat{\mu}_1$. The proof relies on known bounds for the population mean of $X$ (for example, if $X$ is a mixture of Gaussians over \textsc{Age}, one could assume a minimum of 0 and a maximum of 120).

Proposition~\ref{prop:warmup} tells us that in \Cref{ex:warmup}, a private classifier from the ideal model class has privacy error that scales linearly in the class imbalance parameter $r$, which is minimized under no class imbalance. Unfortunately, imbalanced data often has $r \gg 1$; for example, when detecting spam on Twitter \cite{twitter2009}, $r\approx 32$, and in the datasets used in Section \ref{sec:experiments} in our empirical evaluations, $r$ ranges between 8.6 and 130. 

\paragraph{Linking to Imbalanced Metrics} A natural question, building on \Cref{prop:warmup}, is how we might weight samples when calculating $\theta$
to improve performance on imbalanced metrics, such as Recall, under this simple population model. 
We consider a re-weighted classifier, $f_{\theta_\gamma}$, where the weights $\gamma$ are tied to class prevalence. We can reason about weights under this classifier and show, for example, that the true positive rate (TPR) can be written as
$\mathrm{TPR} = \Phi\left(\frac{(1-\gamma)(\mu_1-\mu_0)}{\sigma}\right).$
Through careful analysis, we can show that as class imbalance increases, Recall tends to worsen, but that choosing a weight $\gamma < 1/2$ improves performance relative to the standard Bayes classifier. In practice, we propose setting weights based on class prevalence estimates (e.g., $\gamma = 1/\Pr(y_i=1)$) 
to better target imbalance-focused metrics like Recall. Full details, as well as analysis for other imbalanced classification metrics (such as F1 Score and Precision) are presented in \Cref{app:imbalanced_metrics}.

\subsubsection{Weighted private ERM } \label{sec:weighted_erm}

Standard Empirical Risk Minimization (ERM) trains a model by minimizing an average loss function over a dataset, i.e., optimizing parameters of some model class to reduce the gap between predicted and true data values \cite{vapnik1991principles, devroye2013probabilistic}. Many cost-sensitive approaches to class imbalance rely on sample-weighted objective minimization in the ERM framework, where the minority class samples are up-weighted in the loss function relative to the sample majority \cite{tang2008svms}. We show in \Cref{thm.ermprivacy} that the \textit{differentially private} empirical risk minimization (ERM) procedure of \cite{chaudhuri2011differentially} can be adapted to accommodate such minority sample weights, which we outline in \Cref{alg:objective}. Weighting samples in the objective function allows us to tune the impact of the minority class on the final model parameters.

We instantiate Algorithm~\ref{alg:objective} with the weight function $\mathcal{W}(D)$ as the inverted class frequency for each sample in our experiments in \Cref{sec:experiments}. More formally, for a dataset $\mathcal{D} = \{(x_i, y_i)\}_{i=1}^{n}$, where $y_i \in \{0,\ldots,k\}$ represents the class label of each sample, we compute the class frequencies for class $k$ as $\hat\pi_k=\frac{1}{n}\sum_{i=1}^n\mathbb{I}[y_i=k]$. The inverted class frequency vector $\hat{\pi}^{-1}=(1/\hat\pi_0,\ldots,1/\hat\pi_k)$ gives the sample weights $w_i=\frac{\|\hat\pi^{-1}\|_1}{\pi_{y_i}} \in [0,1]$, where each sample is weighted according to the inverse frequency of its class in the dataset. We choose this weighting scheme to align with our results in \Cref{sec:warmup} along with prior work \cite{chawla2004special,galar2011review}.

\begin{algorithm}[tbh]
\caption{Weighted ERM w/ Objective Perturbation}\label{alg:objective}
\begin{algorithmic}
\State \textbf{Inputs:} Data $\mc{D} = \{x_i, y_i\}$ with $y_i \in \{0,\ldots,k\}$, parameters $\priveps$, $\reg$, $\c$, loss $\ell(\mbf{y}_i, \mbf{x}_i^T\bbeta)$, weight function $\mathcal{W}: \mathcal{D} \rightarrow [0,1]^n$
\State \textbf{Output:} Approximate minimizer $\bbeta_{priv}$.
\State Let $\mathbf{w} = \mathcal{W}(\mathcal{D})$ and $\priveps' = \priveps - \log(1 + \frac{2\c}{n \reg} + \frac{\c^2}{n^2 \reg^2})$
\State If $\priveps' > 0$ then $\extra = 0$ else $\extra = \frac{c}{n(e^{\priveps/4} - 1)} - \reg$, $\priveps' = \priveps/2$.
\State Draw vector $\b$ according to PDF $\nu(\b) \propto e^{- \frac{\priveps'\norm{\b}}{2}}$.
\State Compute $\bbeta_{priv} = \argmin_{\bbeta} \{\frac{1}{n} \sum_{i=1}^{n} w_i \cdot \ell(\mbf{y}_i,\mbf{x}_i^T\bbeta) + \frac{1}{n} \b^T \bbeta + \frac{1}{2} \extra ||\bbeta||^2\}$.
\end{algorithmic}
\end{algorithm}

\Cref{thm.ermprivacy} states that Algorithm~\ref{alg:objective} is still DP, with a full proof deferred Appendix ~\ref{sec:privacy_proof_erm}.

\begin{restatable}{theorem}{ermpriv}\label{thm.ermprivacy}
Algorithm~\ref{alg:objective} instantiated with a loss function $\ell(y,\eta)$ 
that is convex and twice differentiable with respect to $\eta$, with $|\frac{\partial}{\partial\eta}\ell(y,\eta)|\leq 1$ and $|\frac{\partial^2}{\partial\eta^2}\ell(y,\eta)|\leq c$ for
all $y$, is $\priveps$-differentially private.
\end{restatable}

Although our theoretical (and empirical in Section \ref{sec:experiments}) results focus on a logistic regression ERM algorithm, our results directly apply to the kernel method and SVM given in \cite{chaudhuri2011differentially}. 
Surprisingly, no adaptation of private ERM under sample weights was previously known; \cite{giddens2023differentially} had recently considered the problem for more complicated weighting schemes, but under some undesirable assumptions. Their privacy proof works only for loss functions that take in a single argument, which excludes standard models like logistic regression, SVM, and others. Additionally, they made the assumption that the difference of weights across neighboring datasets goes to 0 as $n \rightarrow \infty$, which is too strong for our inverse proportional weights strategy. We also note that in differential privacy, sensitivity is analyzed under worst case assumptions even if the influence of a single data point diminishes as $n$ grows large. One therefore should avoid privacy statements that rely on asymptotic assumptions.

\subsection{Weighted DP-SGD}\label{s.weighteddpsgd}

Competitive approaches to many private classification problems are given with deep learning models, often tuned using a variant of the differentially private stochastic gradient descent (DP-SGD) algorithm \cite{bassily2014private, abadi2016deep} (canonical version given in \Cref{alg:dpsgd}, but with \textit{weighted} cross-entropy loss). 
DP-SGD follows an iterative process of sampling mini-batches of the data, computing gradients on the sampled points, clipping the gradients to have a bounded $\ell_2$-norm to reduce sensitivity, adding noise that scales with $\epsilon$ and the clipping parameter to preserve privacy, and finally updating the model using the resulting clipped noisy gradients. 

\begin{algorithm}[tbh]
\caption{Differentially Private SGD (with weighted Cross-Entropy Loss)}\label{alg:dpsgd}
\begin{algorithmic}
\State \textbf{Inputs:} Database $\mathcal{D} = \{x_i, y_i\}$ with $n$ entries where each $y_i \in \{0, 1\}$, privacy parameters $(\epsilon$, $\delta)$, learning rate $\eta$, clipping norm $C$, minibatch size $B$, batch sampling probability $q = L/n$, number of iterations $T$, initial random model parameters $\theta$.
\State \textbf{Output:} Model parameters $\theta_{\text{priv}}$.
\For{iteration $t = 1$ to $T$}
    \State Construct a batch of expected size $L$ by sampling each  point into the batch with probability $q$
    \State Partition the batch into minibatches of size $B$
    \For{each minibatch $b$}
        \State Compute model predictions $\hat{y}_i = f(x_i; \theta)$ for each $i \in b$.
        \State Compute binary weighted cross-entropy loss as \\ \quad\quad\quad\quad\quad\quad $\mathcal{L}(y, \hat{y}; \mathbf{w}) = - \frac{1}{B} \sum_{i=1}^{B} w_i \left[ y_i \log(\hat{y}_i) + (1 - y_i) \log(1 - \hat{y}_i) \right]$
        \State Compute per-sample gradients $\nabla \mathcal{L}_i = w_i \left( \hat{y}_i - y_i \right) \mathbf{x}_i$ 
        \State Clip gradients $\tilde{\nabla} \mathcal{L}_i = \nabla \mathcal{L}_i \cdot \min\left(1, \frac{C}{\|\nabla \mathcal{L}_i\|_2}\right)$
        \State Parameterize $\sigma^2$ for $(\epsilon', \delta')$-DP, where $\epsilon' = O\left(\epsilon/\sqrt{T \log\left(\frac{1}{\delta}\right)}\right)$, for $(\epsilon,\delta)$-DP overall \cite{abadi2016deep}.
        \State Add noise: $\tilde{\nabla} \mathcal{L}_i = \tilde{\nabla} \mathcal{L}_i + \mathcal{N}(0, \sigma^2 C^2 \mathbf{I})$ 
        \State Update model parameters $\mathbf{\theta} = \mathbf{\theta} - \eta \cdot \frac{1}{B} \sum_{i=1}^{B} \tilde{\nabla} \mathcal{L}_i$
    \EndFor
\EndFor
\State Return differentially private model parameters: $\theta_{\text{priv}} = \theta$.
\end{algorithmic}
\end{algorithm}

For cost-sensitive gradient updates under class imbalance, it is straightforward to show that weights can be incorporated into a standard binary classification loss $\mathcal{L}(y, \hat{y}; \mathbf{w})$ (e.g. cross-entropy) 
while maintaining privacy. Proposition~\ref{prop.ftt_private} formalizes this claim; thus, we are free to re-weight our gradient updates by class prevalence while maintaining privacy.

\begin{proposition}\label{prop.ftt_private}
    \Cref{alg:dpsgd}, 
    a standard DP-SGD procedure with \textit{weighted} cross-entropy loss given by $\mathcal{L}(y, \hat{y}; \mathbf{w}) = - \frac{1}{n} \sum_{i=1}^{n} w_i \left[ y_i \log(\hat{y}_i) + (1 - y_i) \log(1 - \hat{y}_i) \right]$, is $(\epsilon,\delta)$-differentially private.
\end{proposition}
\begin{proof}
    $\mathcal{L}(y, \hat{y}; \mathbf{w})$ does not effect the sensitivity of the gradient $\nabla \mathcal{L}_i$ with respect to each sample; the gradient is bounded by the norm bound $C$ due to clipping. When each per-sample gradient $\nabla \mathcal{L}_i$ is clipped to $\tilde{\nabla} \mathcal{L}_i = \nabla \mathcal{L}_i \cdot \min\left(1, \frac{C}{\|\nabla \mathcal{L}_i\|_2}\right)$, the sensitivity of the gradient is limited to $C$. Adding Gaussian noise calibrated to this sensitivity ensures that the overall training procedure satisfies $(\epsilon,\delta)$-DP. Re-weighting of samples in the loss function \textit{pre-clipping} does not affect these privacy guarantees.
\end{proof}

\section{Experiments}\label{sec:experiments}
In our experiments, we evaluate a range of methods to understand their performance under different privacy and class imbalance conditions. To maintain consistency with previous sections, we categorize methods as pre-processing or in-processing methods.

We evaluate: (1) a private synthetic data method (GEM) as a pre-processing step, generating a class-balanced sample for a downstream, non-private XGBoost model (\textit{GEM + NonPriv. XGBoost}, \Cref{s.privsynthdata}), (2) a private ERM logistic regression model as an in-processing step \textit{without} class weights (\textit{Priv. LogReg}, exact method from \cite{chaudhuri2011differentially}, see \Cref{s.weightederm}), (3) a private ERM logistic regression model as an in-processing step \textit{with} sample weights (\textit{Priv. Weighted LogReg}, our modified algorithm under class weighting, \Cref{alg:objective} in \Cref{s.weightederm}), and (4) a DP-SGD trained FTTransformer model as an in-processing step \textit{with} sample weights in the cross-entropy loss (\textit{Priv. Weighted FTT}, \Cref{s.weighteddpsgd}). All additional details of these methods are given in \Cref{app:complete_experiments}.

We also compare the performance of these methods against the following non-private baselines:  (1) a vanilla XGBoost model with in-processing sample weights (\textit{NonPriv. Weighted XGBoost}), (2) an XGBoost model \textit{without} class sample weights, using SMOTE as a pre-processing step (\textit{SMOTE + NonPriv. XGBoost}), (3) a logistic regression model with and without sample weights (as in-processing) (\textit{NonPriv. Weighted LogReg} / \textit{NonPriv. LogReg}), and (4) a non-private FTTransformer model with and without sample weights in the cross-entropy loss (as in-processing) (\textit{NonPriv. Weighted FTT} / \textit{NonPriv. FTT}). These methods serve as baselines for comparison to measure the effects of adding differential privacy, and the role of weighting in model performance.

In \Cref{sec:eval_real}, we present our main empirical results, from an extensive evaluation conducted on real datasets from the \texttt{imblearn}~\cite{lemaavztre2017imbalanced} repository (summarized in \Cref{tab:imb-learn-tasks}). In \Cref{sec:eval_syn}, we briefly build intuition for the effect of private noise on each classifier's decision boundary using 2-dimensional synthetic data. 
 
\begin{table*}[tbh]
\centering
\begin{tabular}{@{}cllccr@{}}
\toprule
\midrule
ID & Name            & Repository \& Target    & $r=\frac{n_0}{n_1}$  & Size $n$  & $\#$ Features  \\ \midrule
\midrule
1  & ecoli           & UCI, target: imU        & 8.6  & 336    & 7   \\
2  & yeast\_me2      & UCI, target: ME2        & 28   & 1,484  & 8   \\
3  & solar\_flare\_m0 & UCI, target: M-0      & 19   & 1,389  & 32  \\
4  & abalone         & UCI, target: 7          & 9.7  & 4,177  & 10  \\
5  & car\_eval\_34   & UCI, target: good, v good & 12  & 1,728  & 21  \\
6  & car\_eval\_4    & UCI, target: vgood      & 26   & 1,728  & 21  \\
7  & mammography     & UCI, target: minority   & 42   & 11,183 & 6   \\
8  & abalone\_19     & UCI, target: 19         & 130  & 4,177  & 10  \\
\midrule
\bottomrule
\end{tabular}
\caption{Imbalanced learning datasets used from the \href{https://imbalanced-learn.org/stable/references/generated/imblearn.datasets.fetch_datasets.html}{\texttt{imblearn}} package.}
\label{tab:imb-learn-tasks}
\end{table*}

\subsection{Evaluations on Real Data}
\label{sec:eval_real}
We next empirically evaluate the performance of our methods for private binary classification under class imbalanced data using eight datasets from the Imbalanced-learn \cite{lemaavztre2017imbalanced} repository. These datasets represent a variety of settings, with imbalance ratios $r \in [8.6,130]$ and sizes $n\in[336, 11183]$; see \Cref{tab:imb-learn-tasks} for complete details.

All datasets we chose were purposefully low-dimensional enough to be run with GEM. Neural models (GEM and FTTransformer) were trained using an NVIDIA T4 GPU, with $\epsilon \in \{0.05, 0.1, 0.5, 1.0, 5.0\}$ (privacy budget range following guidance from \cite{mckenna2022aim}). Private models were trained for 20 epochs, while non-private models were trained for 100 epochs with early stopping. FTTransformer was initialized with default architecture hyper-parameters (dimension=32, depth=6, 8 heads, dropout of 0.1). DP-SGD was performed with the Opacus pytorch library using recommended parameters \cite{opacus}. No hyperparameter tuning was performed for the private models to ensure ``honest'' comparisons \cite{papernot2021hyperparameter}; hyperparameters were lightly tuned for non-private models using randomized cross-validation. Results are given with standard deviations over 10 randomly seeded data splits and parameter initializations. GEM models are computationally expensive \cite{liu2021iterative, rosenblatt2024epistemic}; they were trained in parallel on the same NVIDIA T4 and took over 50 compute hours. XGBoost and LogReg models trained within seconds, while FTTransformer models required minutes.

In \Cref{fig:mammography_body}, we show how performance varies with privacy level; our performance metrics include general metrics like AUC, F1, and Precision, as well as metrics that are more tailored to imbalanced classification, such as Recall, Worst Class Accuracy, etc. The macro-average accuracy (Macro-Avg-ACC) helps evaluate performance across both classes without bias toward the majority class, while the geometric mean (G-Mean) provides insight into the balance between sensitivity and specificity. Higher is better for all metrics. \Cref{fig:mammography_body} presents results on the \textit{mammography} dataset, which was representative of general trends for all datasets. Complete plots are presented in Figures~\ref{fig:ecoli_private} to~\ref{fig:abalone_private} in \Cref{app:complete_experiments}.

\begin{figure}[htb]
    \centering
    \includegraphics[width=0.87\linewidth]{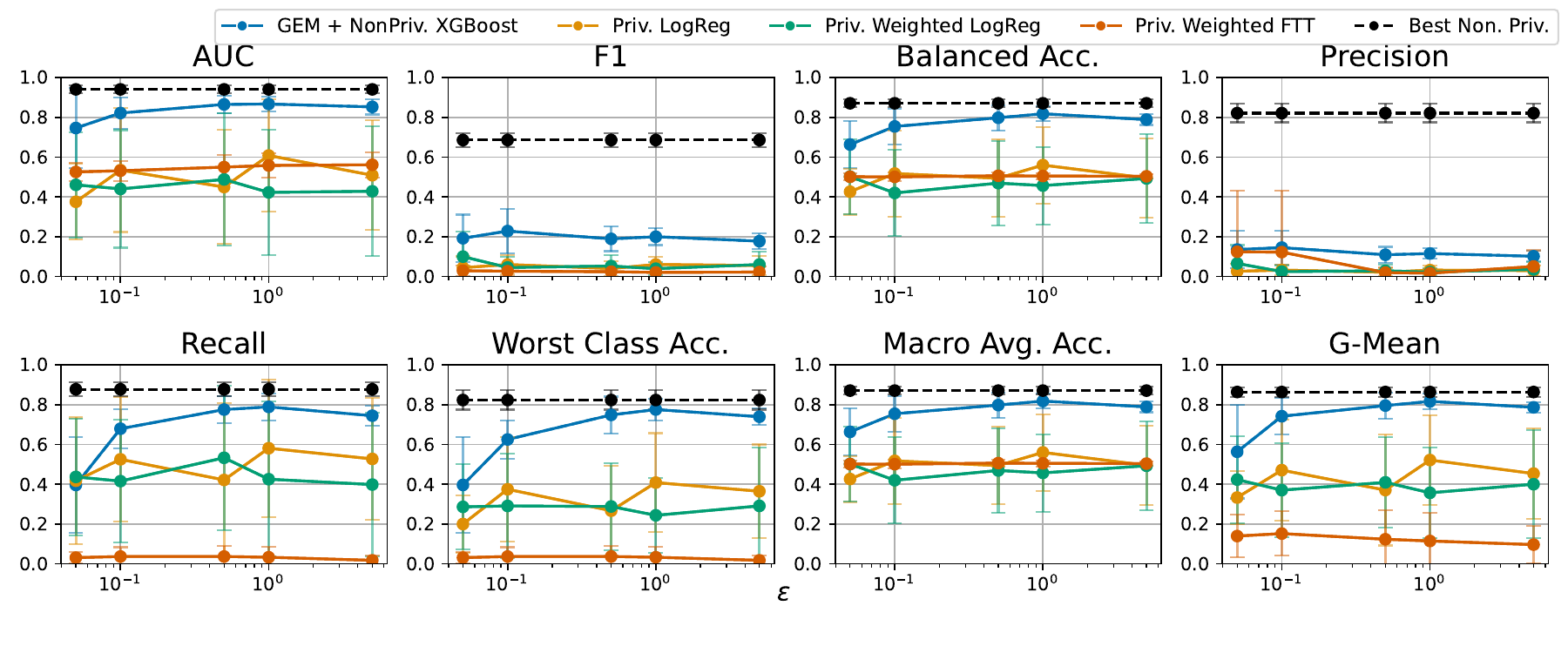}
    \caption{Performance for \textit{mammography} dataset under varying $\epsilon$ parameters for overall performance metrics (AUC, F1, Balanced Accuracy, Precision) and metrics appropriate for imbalanced classification settings (Recall, Worst Class Accuracy, Macro Average Accuracy, Geometric Mean).}
    \label{fig:mammography_body}
\end{figure}

\paragraph{Varying Privacy Budget} We observed that for all datasets, the \textit{GEM+XGBoost} method improved with increased privacy budget. \Cref{fig:mammography_body} presents results on the \textit{mammography} dataset, which is representative of general trends. Higher dimensionality increased the difficulty across the board (e.g., there was a larger difference between non-private performance and private performance with, for example, the $car\_eval\_4$ dataset (\Cref{fig:car_eval_4_private})), but we did not find a meaningful trend or interaction between imbalance ratio and dimensionality. Absolute dataset size correlated with the classification performance, as expected. Complete plots are presented in Figures~\ref{fig:ecoli_private}-\ref{fig:abalone_private} in Section~\ref{app:complete_experiments}.

Additionally, in our experiments, we found that more minority examples led to more stable improved performance from the \textit{GEM+XGBoost} model. For example, the \textit{mammography} (\Cref{fig:mammography_body}) and \textit{abalone} (\Cref{fig:abalone_private}) datasets, both of which have the highest number of minority class examples, also exhibited the best performance for the \textit{GEM+XGBoost} synthesizer at low levels of epsilon, and the most stable performance overall across varied privacy parameters.

\begin{table}[htb]\label{tab:ranking_performance}
\centering
\caption{Average performance rankings of the DP imbalanced learning approaches, across all $\epsilon$ settings and datasets. Average ranks are in $[1,4]$ and in descending order, so lower is better. We adopt the Olympic medal convention: \colorbox{gold!30}{gold}, \colorbox{silver!30}{silver} and \colorbox{bronze!30}{bronze} cells signify first, second and third best performance, respectively.}
\label{tab:rankings}
\resizebox{\columnwidth}{!}{
\begin{tabular}{lcccccccc}
\toprule
\multicolumn{1}{c}{\multirow{2}{*}{\textbf{Model}}} & \multicolumn{4}{c}{\textbf{Overall}} & \multicolumn{4}{c}{\textbf{(Im)Balanced}} \\
\cmidrule(lr){2-5} \cmidrule(lr){6-9}
& \textbf{AUC} & \textbf{F1} & \textbf{Bal-ACC} & \textbf{Precision} & \textbf{Recall} & \textbf{Worst-ACC} & \textbf{Macro-Avg-ACC} & \textbf{G-Mean} \\
\midrule
GEM + XGBoost & \cellcolor{gold!30}1.45 & \cellcolor{gold!30}1.45 & \cellcolor{gold!30}1.48 & \cellcolor{gold!30}1.45 & \cellcolor{bronze!30}2.26 & \cellcolor{gold!30}1.45 & \cellcolor{gold!30}1.48 & \cellcolor{gold!30}1.45 \\
Priv. LogReg & \cellcolor{silver!30}2.77 & \cellcolor{silver!30}2.62 & \cellcolor{silver!30}2.89 & \cellcolor{silver!30}2.62 & \cellcolor{silver!30}2.20 & \cellcolor{bronze!30}2.89 & \cellcolor{bronze!30}2.89 & \cellcolor{bronze!30}2.86 \\
Priv. Weighted LogReg & \cellcolor{lightgrey!30}3.19 & \cellcolor{bronze!30}2.65 & \cellcolor{bronze!30}2.59 & \cellcolor{bronze!30}2.65 & \cellcolor{gold!30}2.11 & \cellcolor{silver!30}2.59 & \cellcolor{silver!30}2.59 & \cellcolor{silver!30}2.59 \\
Priv. Weighted FTT & \cellcolor{bronze!30}2.89 & \cellcolor{lightgrey!30}3.58 & \cellcolor{lightgrey!30}3.34 & \cellcolor{lightgrey!30}3.58 & \cellcolor{lightgrey!30}3.70 & \cellcolor{lightgrey!30}3.37 & \cellcolor{lightgrey!30}3.34 & \cellcolor{lightgrey!30}3.40 \\
\bottomrule
\end{tabular}
}
\end{table}

In \Cref{tab:rankings}, we present average rankings across all datasets and epsilon values for the four privacy-preserving imbalanced learning approaches we explore; here, lower is better, and highest average performance in each row is highlighted according to the Olympic medal convention (gold, silver, bronze).
\textit{GEM + XGBoost} performs best, ranking highest across 7 of the 8 metrics on average. As expected, \textit{Priv. Weighted LogReg} performs worse than its unweighted counterpart on overall metrics. Overall metrics are well known to be poor indicators in imbalanced learning, as many of them weight negative and positive class performance equally \cite{he2009learning}. 
However, on the metrics more appropriate for imbalanced classification, \textit{Priv. Weighted LogReg} outperforms the unweighted variant in 3 out of 4 metrics, and has the best average Recall among all private models. In stark contrast, \textit{Priv. Weighted FTT} consistently under-performed.

\paragraph{Empirical Takeaways} Private variants of neural models (\textit{Priv. Weighted FTT}, for example)  may be inappropriate in general for relatively low-data regimes under class imbalance due to minority example sparsity, especially when weighted ERM based methods like \textit{Priv. Weighted LogReg} perform well and are less expensive to train. Moreover, pre-processing with private synthetic data (\textit{GEM + XGBoost}) displayed the most robust performance across varying privacy levels and imbalanced datasets in our experiments, consistently ranking highest across nearly all metrics. Unfortunately, this method is limited to low-dimensional datasets, and is computationally expensive, even intractable in certain data settings. \textit{Priv. Weighted LogReg} performed best in terms of Recall, and performed second best on average in terms of the other imbalanced classification metrics. Our empirical results lead us to recommend these two methods, depending on the metric of interest, data context, and computational resources available.

\subsection{Visualizing Decision Boundaries}
\label{sec:eval_syn}
Next we explore the effect of differential privacy on decision boundaries by presenting visualizations on 2-dimensional synthetic data. These visualizations of decision boundaries help develop intuition for how private noise impacts model predictions, particularly in class-imbalanced settings. 

We generate a small ($n=1000$) synthetic 2-dimensional mixture of Gaussians, where majority (negative) and minority (positive) classes are separable in the feature space. Specifically, the random vector $[X_1, X_2]$ is sampled from the following process: with probability 0.9, $[X_1, X_2] \sim \mathcal{N}([0, 0], \begin{bmatrix} 4 & 0 \\ 0 & 4 \end{bmatrix})$, and with probability 0.1, $[X_1, X_2] \sim \mathcal{N}([4, 4], \begin{bmatrix} 4 & 0 \\ 0 & 4 \end{bmatrix})$. Thus, the mixture has two components: one centered at $[0, 0]$ and the other at $[4, 4]$, both independent and with variance 4.

\Cref{fig:synth_data_boundaries} compares the decision boundaries of non-differentially private and differentially private classifiers on this data, allowing us to directly observe the impact of the privacy preserving methods on how the model makes decisions. The blue points represent majority (negative) class examples, while the red points represent minority (positive) class examples. The blue region denotes where the model will predict a negative label, and the red region denotes where the model will predict a positive label. The underlying data distributions are also visible in these figures, represented as an mean-centered ellipse capturing 2 standard deviations of the 2d-Gaussian. 

Inspecting \Cref{fig:synth_data_boundaries} helps build intuition for the effect of DP on decision boundaries. We observe that \textit{Priv. Weighted FTT} fails to learn a meaningful decision boundary (labeling everthing negative), while \textit{Priv. LogReg} is catastrophically noisy (flipping the decision boundary).
\textit{GEM + NonPriv. XGBoost} (\Cref{algo:rejection_sampling}) is lossy relative to \textit{SMOTE + NonPriv. XGBoost}, but maintains a class separating boundary.

\begin{figure}[htb]
    \begin{minipage}{0.24\linewidth}
        \centering
        \includegraphics[width=\linewidth]{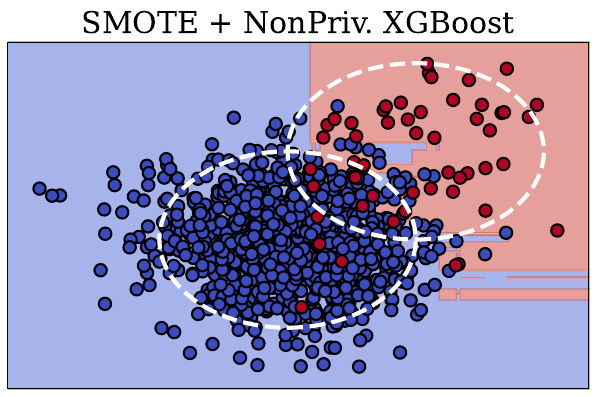}
    \end{minipage}\hfill
    \centering\begin{minipage}{0.24\linewidth}
        \centering
        \includegraphics[width=\linewidth]{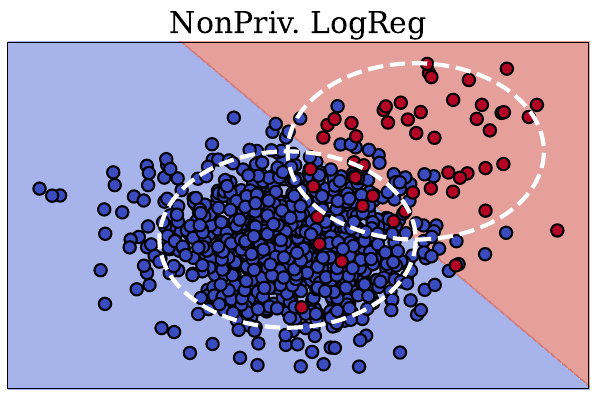}
    \end{minipage}\hfill
    \begin{minipage}{0.24\linewidth}
        \centering
        \includegraphics[width=\linewidth]{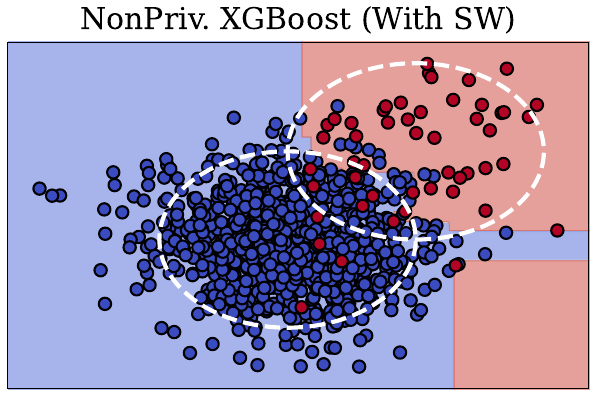}
    \end{minipage}\hfill
    \begin{minipage}{0.24\linewidth}
        \centering
        \includegraphics[width=\linewidth]{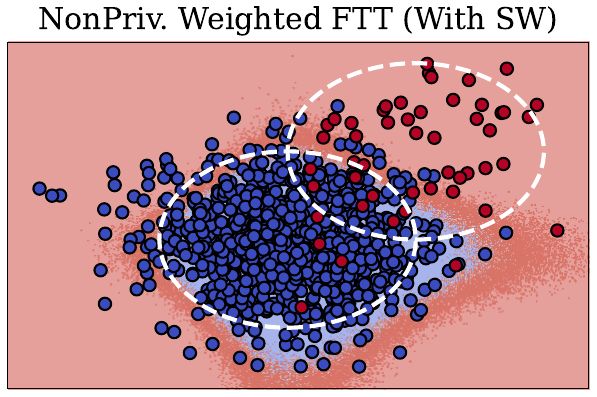}
    \end{minipage}\hfill
    
    \begin{minipage}{0.24\linewidth}
        \centering
        \includegraphics[width=\linewidth]{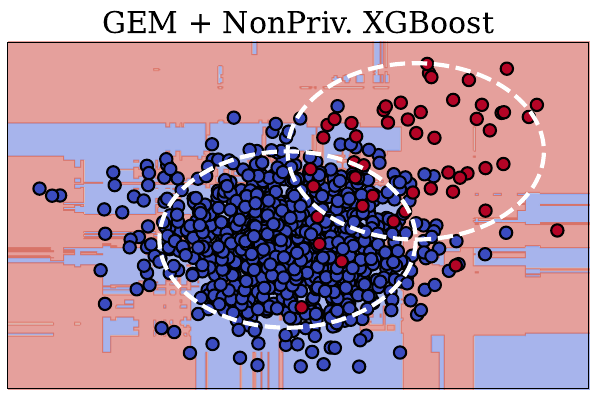}
    \end{minipage}\hfill
    \begin{minipage}{0.24\linewidth}
        \centering
        \includegraphics[width=\linewidth]{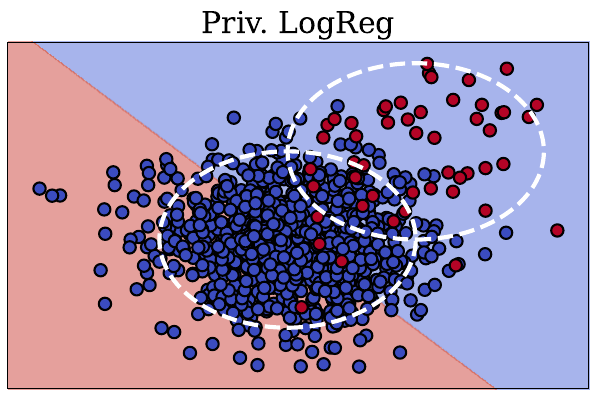}
    \end{minipage}\hfill
    \begin{minipage}{0.24\linewidth}
        \centering
        \includegraphics[width=\linewidth]{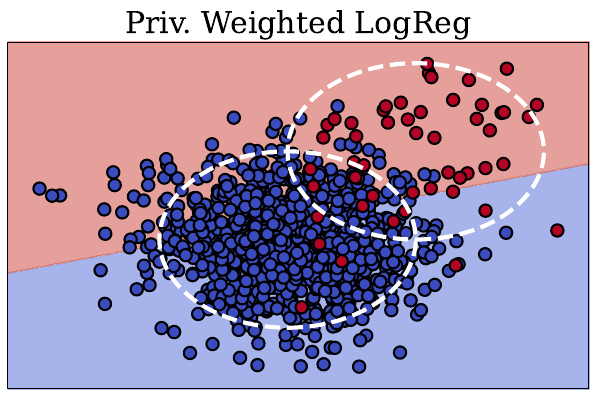}
    \end{minipage}\hfill
    \begin{minipage}{0.24\linewidth}
        \centering
        \includegraphics[width=\linewidth]{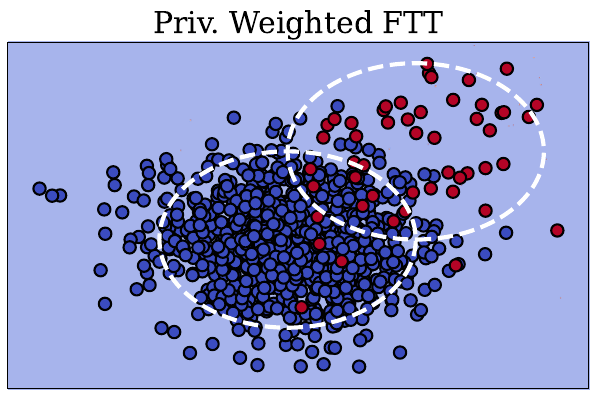}
    \end{minipage}
    \caption{
    Top row shows decision boundaries of non-DP classifiers (high performance on the task, $AUC \in [0.94,0.97]$). Bottom row illustrates the decision boundaries of DP classifiers ($\epsilon=1.0$, $\delta=1\texttt{e-5}$ where applicable), which perform worse. The underlying true data generating function for each class is represented as an ellipse (dotted white line), where the center of the ellipse is the mean and each point on the dotted line represents 2 standard deviations from the mean.}
    \label{fig:synth_data_boundaries}
\end{figure}

\section{Conclusion}
Private binary classification under class imbalance is especially challenging. We show that commonly used non-private imbalanced learning approaches like SMOTE and bagging are inappropriate in a DP setting. We show that instead, cost-sensitive versions of ERM and deep learning can be adapted for DP, and existing DP synthetic data methods can be used to generate balanced data for training, with strong empirical performance. Unfortunately, DP synthetic data methods suffer from the curse of dimensionality \cite{mckenna2019graphical,liu2021iterative}, which is a significant limitation in most practical settings. To address this limitation, future work could explore \textit{hybrid} algorithms that leverage components of DP synthetic data algorithms in a lower dimensional space before solving a weighted ERM problem in the higher dimensions. Additionally, it would be interesting to explore recently proposed imbalanced-learning-specific loss functions (\cite{cui2019class, cao2019learning, li2021autobalance}) for DP-SGD methods.

\bibliographystyle{abbrvnat}
\bibliography{ref}

\appendix
\newpage

\section{Pre-processing Methods and Analysis}

\subsection{SMOTE}\label{app:smote}

\smote*

\begin{proof}
Let $D=(X,y)$ and $D'=(X',y')$ be two neighboring datasets such that $D^\prime = D \cup \{(x,1)\}$, and let $M$ be an arbitrary $\epsilon$-DP algorithm. For the remainder of the proof, fix SMOTE  parameters $N\in \mathbb{N}$ and $k\in \mathbb{Z}^{+}$. Define $l(d,k)$ to be the maximum number of times one point from $\mathbb{R}^d$ can appear among the $k$-nearest neighbors of an arbitrary set of other points in $\mathbb{R}^d$. To simplify notation, we may denote $l(d,k)$ as simply $l$ when $d$ and $k$ are clear from context.

To compare the outputs of SMOTE($X, N, k$) and SMOTE($X^\prime, N, k$), we fix the internal randomness of SMOTE between these two runs, which includes randomly choosing a nearest neighbor and randomly sampling $u$ inside the for-loop. That is, an output point will be different only if the new point $x$ in $X'$ replaces the selected nearest neighbor $x'_i$ that was chosen under $X$. For each iteration where $x$ replaces a previous nearest neighbor, there is a $1/k$ probability that $x$ is the selected nearest neighbor.

Define the random variable $Y=|SMOTE(X, N, k)\oplus SMOTE(X^\prime, N, k)|$, where $\oplus$ denotes a symmetric difference. Then $Y$ can be described as the sum of independent Bernoulli random variables that are 1 if $x$ is the selected $k$-nearest neighbor. Each trial has success probability $1/k$, and there are $l \lceil\frac{N}{n_1}\rceil$ total trials, corresponding to the $l$ datapoints that can be neighbors to $x$ and the $\lceil\frac{N}{n_1}\rceil$ iterations through the database. For simplicity of presentation, we drop the ceiling notation for the remainder of the proof, but it is implied if $N$ is not divisible by $n_1$. Thus $Y \sim Binomial\left( \frac{l\cdot N}{n_1},1/k\right)$  and $\E{[Y]} = \frac{l\cdot N}{n_1 \cdot k}$. 

Note that using the upper bound $Y\leq \frac{lN}{n_1}$, we can obtain an immediate DP guarantee for $M$ applied to the output of SMOTE using group privacy.  Specifically, since $Y\leq \frac{lN}{n_1}$, we know that changing one entry of $X$ would change up to $\frac{lN}{n_1}$ entries of the output of SMOTE, which is equivalent to changing $\frac{lN}{n_1}+1$ entries of the input to $M$ (since the input to $M$ is the original database $X$ concatenated with the output of SMOTE). Thus by the group privacy property of DP, these $\frac{lN}{n_1}+1$ entries that depend on $x$ would jointly receive a $(\epsilon (\frac{lN}{n_1} + 1),0)$-DP guarantee.

However, since $Y$ is a random variable, one can instead use a high probability bound on $Y$ as it may lead to an improved $\epsilon$ bound. There is some chance that $Y$ will fail to satisfy this bound, and this failure probability will later be incorporated into the $\delta$ parameter of DP. Using a one-sided Chernoff bound, we bound the probability that $Y$ is significantly greater than its mean:
\begin{equation}\label{eq.ychernoff}
\Pr\left[Y\ge (1+\gamma)\frac{l N}{n_1 k}\right] \le e^{-\frac{\gamma^2}{2+\gamma}\frac{l N}{n_1 k}},
\end{equation}
for any $\gamma \geq 0$. 

For ease of notation, let $N_l=l\cdot N/n_1$, and let $T(D)=(X,y)\cup SMOTE(X, N, k)$. Then for an arbitrary set of outputs $S\subset Range(M)$, we can obtain the following bounds on the output of $M \circ T$ on $D$ and $D'$:
\begin{align*}
\Pr&[M(T(D))\in S] = \sum_{j=1}^{N_l} \Pr\left[M(T(D))\in S|Y=j\right]\cdot\Pr[Y=j]
\\
&\le \sum_{j=1}^{N_l} e^{\epsilon \cdot j}\Pr[M(T(D^\prime))\in S|Y=j]\cdot\Pr[Y=j]
\\
&=  \sum_{j=1}^{\frac{(1+\gamma)N_l}{k}} e^{\epsilon \cdot j}\Pr[M(T(D^\prime))\in S|Y=j]\cdot\Pr[Y=j] +\sum_{j=\frac{(1+\gamma)N_l}{k}+1}^{N_l} e^{\epsilon \cdot j}\Pr[M(T(D^\prime))\in S|Y=j]\cdot\Pr[Y=j] \\
&\le e^{\epsilon \cdot (1+\gamma)N_l/k} \sum_{j=1}^{(1+\gamma)N_l/k} \Pr[M(T(D^\prime))\in S |Y=j]\Pr[Y=j]
+  e^{\epsilon \cdot N_l} \Pr[Y\ge (1+\gamma)\frac{N_l}{k}]
\\
&\le e^{\epsilon \cdot (1+\gamma)N_l/k} \Pr[M(T(D^\prime))\in S]+  e^{\epsilon \cdot N_l-\frac{\gamma^2}{2+\gamma}\frac{N_l}{k}}. 
\end{align*}
The first equality is due to the law of total probability, the second step is due to the group privacy property of DP, and the third step separates the sum into small and large $j$ values based on the parameter $\gamma$. In the fourth and fifth steps, we bound each coefficient $e^{\epsilon j}$ by the largest value of $j$ in the respective sum. For small $j$ values we then the apply the law of total probability; for large $j$, we upper bound each term $\Pr[M(T(D^\prime))\in S |Y=j]$ by 1, so the remaining sum is simply the probability that $Y$ is greater than the smallest ``large'' $j$ value, which is then bounded by the one-sided Chernoff bound of Equation \eqref{eq.ychernoff}.

Therefore, the composition of first applying SMOTE and then applying $M$ to the original dataset along with the output of SMOTE is $\left( \eps (1+\gamma) \frac{l N}{n_1 k} , e^{\epsilon \frac{l N}{n_1}-\frac{\gamma^2}{2+\gamma}\frac{l N}{k n_1}} \right)$-differentially private. 

Next, we prove Lemma \ref{lem.smotestab}, which gives a lower bound for a parameter $l(d,k)$ that describes the maximum number of times one point from $\mathbb{R}^d$ can appear among $k$-nearest neighbors of other points from $\mathbb{R}^d$. The proof is a geometric argument that relies on the notion of \emph{kissing number} $K(d)$, which is the greatest number of equal sized non-overlapping spheres in $\mathbb{R}^d$ that can touch another sphere of the same size \cite{musin2008,jenssen:joos:perkins2018}.

\smotestab*

\begin{proof}
Trivially $l(d, k) \leq n_1$ so in the following, we will consider the case where $n_1$ is sufficiently large. Also w.l.o.g. we will consider $K(d)$ kissing number spheres of radius $r=1$.

Consider constructing a point set $S$ around the origin $O = (0,...,0) \in \mathbb{R}^d$ such that $S$ contains points whose $1$-nearest neighbor is $O$. We next define the points $x_i \in \mathbb{R}^d$, s.t. $S = \{ x_1,...,x_{K(d)} \}$ where each $x_i$ is the center point in each of the $K(d)$ kissing point spheres around the unit sphere centered at $O$. By construction, each $||x_i||_2 = 2$ and $||x_i - x_j||_2 \geq 2$ for every other $x_j \in S$. 

We have so far a set $S$ with cardinality $|S| = K(d)$, which contains the $K(d)$ centroids whose 1-nearest neighbor is $O$. Recall that we allow ourselves to break ties in distance arbitrarily. Ties are often broken probabilistically in $k$-nearest implementations, but we are considering the ``worst-case'' scenario for our analysis of $S$. 

We now demonstrate that we cannot locally increase $S$. That is, $K(d)$ is the maximum number of points who can share a 1-nearest neighbor with $O$. We show this by contradiction. 

Consider adding a new point $x_{K(d)+1}$ into the set $S$ of 1-nearest neighbors with $O$. How can $x_{K(d)+1}$ be a valid 1-nearest neighbor of $O$? If $||x_{K(d)+1}||_2 > 2$, $O$ is certainly not its 1-nearest neighbor; instead, for some $x_i \in S$, $||x_{K(d)+1} - x_i||_2 < ||x_{K(d)+1} - O||_2 $ by construction. If $||x_{K(d)+1}||_2 \leq 2$ then it would either:
\begin{itemize}
    \item Have $O$ as its 1-nearest neighbor, implying that $||x_{K(d)+1} - O||_2 \leq ||x_{K(d)+1} - x_i||_2$ for all $x_i \in S$, which then implies that $||x_j - x_{K(d)+1}||_2 \leq ||x_j - O||_2$ for at least one point $x_j \in S$, thus either shrinking $|S|$ or leaving it the same size.
    \item Have a fixed $x_j \in S$ as its 1-nearest neighbor, implying that $||x_{K(d)+1} - x_j||_2 \geq ||x_{K(d)+1} - x_i||_2$ for all $x_i \in S$ and $O$, but then implying that $||x_j - x_{K(d)+1}||_2 \leq ||x_a - O||_2$ for some other $x_a \in S$ by the triangle inequality, again shrinking $|S|$.
\end{itemize}
Thus, a new point $x_{K(d)+1}$ cannot be added to $S$ when $k=1$, and $l(d, 1) = K(d)$.

Next we generalize this result from 1-nearest neighbors to $k$ nearest neighbors, demonstrating that $l(d,k) = kK(d)$. We will do this by duplicating points in $S$ from the 1-nearest neighbor construction to create a set $S'$, and then again show by contradiction that this set $S'$ cannot locally increase in size.

For $k$-nearest neighbor, we construct a set of points $S'$ as follows, where each $x_i^j$ for $j \in \{1,...,k\}$ is an exact replica of $x_i$ from $S$. Thus, $S' = \{x_1^1,x_2^1,...,x_{K(d)}^1\} \cup ... \cup \{x_1^k,x_2^k,...,x_{K(d)}^k\}$. Note that $|S'| = k K(d)$, where we have $k$ duplicates of the set $S$ from the 1-nearest neighbor example.

For each point $x_i^j \in S'$, there are $k-1$ points $\{x_i^1,...,x_i^k\} \neq x_i^j$ for which $||x_i^j - x_i^c||_2 = 0$. As before, the distance from the origin to each point $||x_i^j - O|| = 2$ by construction, and for each $x_i^j$ and the $kK(d) - k + 1$ points $x_a^b$ that are not duplicates of $x_i^j$, $||x_i^j - x_a^b|| \geq 2$. Thus for $S'$ of size $k K(d)$, then $O$ is a $k$-nearest neighbor of every point in $S'$, using worst-case tie breaking. 

We next show that the number of points with $O$ as a nearest neighbor cannot increased by adding a new point $x_{K(d)+1}^{j+1}$ to $S'$. The argument is analogous to the argument for $1$-nearest neighbor. If $||x_{K(d)+1}^{j+1}||_2 > 2$, then $O$ is certainly not its $k$-nearest neighbor; instead, for some size $k$ set $\{x_i^1,...,x_i^k\}\subset S'$, $||x_{K(d)+1}^{j+1} - x_i^j||_2 < ||x_{K(d)+1}^{j+1} - O||_2 $ by construction. If $||x_{K(d)+1}^{j+1}||_2 \leq 2$ then it would either:
\begin{itemize}
    \item Have $O$ in its set of $k$-nearest neighbors, implying that $||x_{K(d)+1}^{j+1} - O||_2 \leq ||x_{K(d)+1}^{j+1} - x_i^j||_2$ for at least one $x_i^j \in S'$, which then implies that $||x_i^a - x_{K(d)+1}^{j+1}||_2 \leq ||x_i^a - O||_2$ for at least one point $x_i^a \in S$, thus either shrinking $|S'|$ or leaving it the same size.
    \item Have an entire set $\{x_i^1,..., x_i^k\}\in S'$ as its $k$-nearest neighbors, again shrinking $|S|$ by the triangle inequality as in the $1$-nearest neighbor argument.
\end{itemize}
Thus, we have shown that $|S'|$ cannot be locally improved, and that $l(d, k) = kK(d)$.
\end{proof}

The exact value of the kissing number $K$ in general $d$ dimensions is an open problem, but is known to be lower bounded by $K \geq 2^{0.2075d(1+o(1))}$ \cite{wyner1965,musin2008} and upper bounded by $K \leq 2^{0.4042d}$ \cite{kabatiansky:levenshtein1978}. Thus when $n_1$ is not too small, $k 2^{0.2075d(1+o(1))} \le l(d,k) \le k 2^{0.4042d}$. We note that even though the exact value of the kissing number is unknown, its bounds are asymptotically tight, with exponential dependence on $d$.

Plugging in the maximum value of $k 2^{0.4042d}$ for $l(d,k)$ into the differentially privacy bounds derived above recovers the guarantees of the theorem.
\end{proof}

Here we present simple empirical results illustrating that SMOTE as a pre-processing step before differentially private learning results in extremely poor performance. Figure \ref{fig:smote_mammography_rebuttal} presents the performance of SMOTE as a pre-processing method before DP logistic regression with three different $\epsilon$ values, compared with non-private logistic regression and DP logistic regression without SMOTE applied. The evaluation is performed on the \textit{mammography} dataset (see Section \ref{sec:experiments}) with a variety of imbalance ratios created by subsampling.

As predicted, downstream performance degrades significantly after SMOTE-induced $\epsilon$ adjustments as described in \Cref{table:adjusted_eps}. Note how proper privacy adjustments after SMOTE (dotted lines) negatively impact performance compared to DP logistic regression without SMOTE (solid red line). This empirically confirms our negative result of Theorem \ref{thm.smote}, that SMOTE should not be a preferred pre-processing method for differentially private imbalanced learning.

\begin{figure}[htb]
    \centering
    \includegraphics[width=0.97\linewidth]{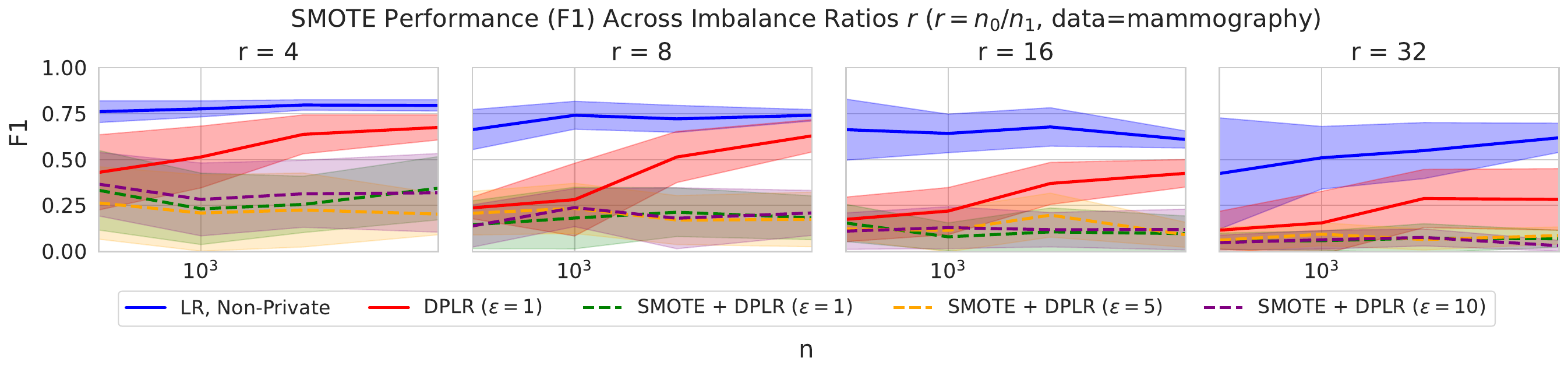}
    \caption{SMOTE pre-processing on downstream DP logistic regression (with adjusted $\epsilon$) on the \textit{mammography} dataset. Data was subsampled (log-scale x-axis: $n \in [500, 1000, 2000, 5000, 10000]$) and evaluated across imbalance ratios $r \in [4, 8, 16, 32]$.}
    \label{fig:smote_mammography_rebuttal}
\end{figure}

\section{In-processing Methods and Analysis}

\subsection{Bagging and Private Bagging}\label{app:bagging}

\bagging*

\begin{proof}
From Theorem 3 in \cite{liu2020intrinsic}, given a
training dataset of size $n$ and an arbitrary non-private base learner, bagging with replacement with $m$ base models and a subsample size of $k$ has privacy parameters $\epsilon = m \cdot k \cdot \log(\frac{n+1}{n})$ and $\delta = 1 - \left(\frac{n-1}{n}\right)^{m \cdot k}$. Solving for $m\cdot k$ in the $\delta$ equation and plugging in $\delta = n^{-c}$ yields, 
\begin{align*}
    m \cdot k = \frac{\log(1 - n^{-c})}{\log(n-1) - \log(n)}~.
\end{align*}
Plugging this in to the expression for $\epsilon$ gives, for $n > 1$,
\begin{align*}
    \epsilon &= \log(1 - n^{-c}) \frac{\log (n+1) - \log(n)}{\log(n-1) - \log(n)} \\
    &= \log(1 - n^{-c}) \frac{\log \left(1 + n^{-1}\right)}{\log\left(1 - n^{-1}\right)} \\
    &\leq \log(1 - n^{-1}) \frac{\log \left(1 + n^{-1}\right)}{\log\left(1 - n^{-1}\right)} \\
    &= \log\left(1 + n^{-1}\right)
\end{align*}

Thus, for $c>1$, applying the bound of $\log (1 + x) \leq x$ yields the result. 
\end{proof}

We also present simple empirical results for DP Bagging to illustrate its poor performance as an in-processing method for DP imbalanced learning. Figure \ref{fig:bagging_mammography_rebuttal} presents the performance of DP bagging using DP logistic regression as a weak learner, compared against the two baselines of non-private logistic regression and DP logistic regression without bagging. For each DP-LR in the Bagged classifier, the privacy budget was split among the estimators using advanced composition (setting $\epsilon=1/2$ and $\delta = 1/n^2$). The evaluation is performed on the \textit{mammography} dataset (see Section \ref{sec:experiments}) with a variety of imbalance ratios created by subsampling. As predicted, we observe that private bagging underperformed relative to a single DP logistic regression classifier across sample sizes and imbalance ratios.

\begin{figure}[htb]
    \centering
    \includegraphics[width=0.97\linewidth]{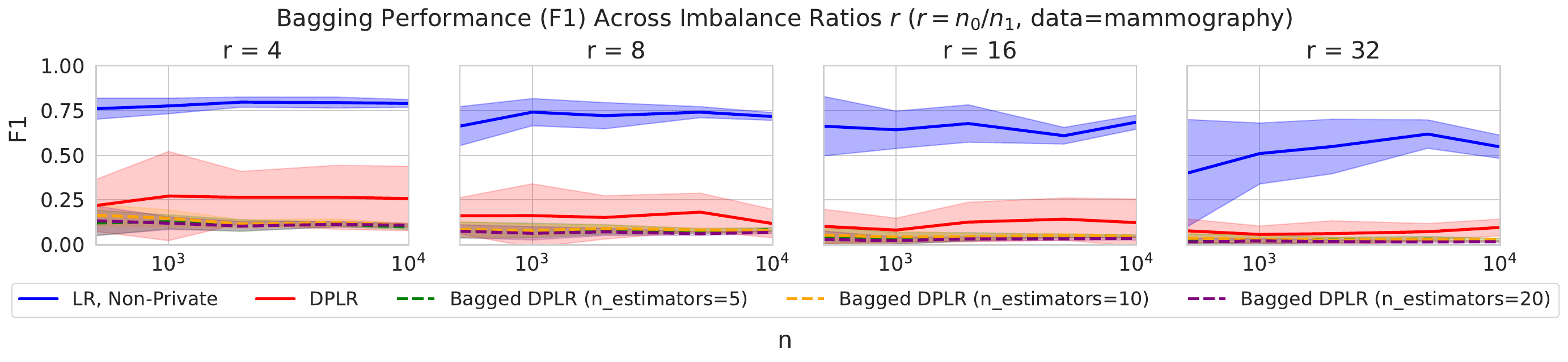}
    \caption{$F1$ score performance on subsamples of the \textit{mammography} dataset (\texttt{imblearn}) comparing differentially private logistic regression (DPLR) and DP bagging (DPLR as weak learner). Data was subsampled (log-scale x-axis: $n \in [500, 1000, 2000, 5000, 10000]$) and evaluated across imbalance ratios $r \in [4, 8, 16, 32]$. }
    \label{fig:bagging_mammography_rebuttal}
\end{figure}

\subsection{Warm-up: A Known Population}
\label{app:warmup}

\warmup*

\begin{proof}
Recall our private mean estimates for each class are, 
\begin{align*}
    \hat{\mu}_0 = \frac{1}{n_0}\sum_{i \in \{Y_i=0\}}\textsc{clip}(X_i, R) + G_0~, \quad \hat{\mu}_1 = \frac{1}{n_1}\sum_{i \in \{Y_i=1\}}\textsc{clip}(X_i, R) + G_1,
\end{align*}
where $G_0\sim\mathcal{N}(0, \sigma_{DP_0}^2)$, $G_1\sim \mathcal{N}(0, \sigma_{DP_1}^2)$ and  $\textsc{clip}(x,R)=\max\{-R,\min(x,R)\}.$ This clipping function guarantees the sensitivity of our private mean computation function is $\Delta f = \frac{2R}{n_{y}}$. Thus, the Gaussian mechanism gives $(\epsilon, \delta)$-DP with variance $\sigma_{DP}^2 = \Delta f^2 \frac{~2\log(1.25/\delta)}{\epsilon^2} = \frac{8R^2\log(1.25/\delta)}{n_{y}^2\epsilon^2}$.

First, we'll argue that with our choice of $R$ with probability $1-\beta/2$ clipping does not bias the data. To see this  let $\{Z_i\}_{i=1}^n\overset{iid}{\sim}N(0,\sigma^2)$. We will use the following standard concentration result for the maximum of sub-Gaussian random variables to bound them.
\begin{lemma}[\cite{rigollet:hutter2023}]\label{lem:maxGaussian}
Let $Z_1,\dots,Z_n\overset{iid}{\sim}N(0,\sigma^2)$. Then,
$\Pr(\max_{1\leq i \leq n}|Z_i| >t)\leq 2n e^{-\frac{t^2}{2\sigma^2}}.$
\end{lemma}

Consecutively using  the  triangle inequality, Lemma \ref{lem:maxGaussian}, and $R>B+\sigma\sqrt{2\log(4n/\beta)}$ we see that,
\begin{equation*}
\label{eq:max_ineq}
    \Pr[\max_{1 \leq i \leq n} |X_i| > R] \leq  
   \Pr[\max_{1\leq  i \leq n}|Z_i| + B > R]  \leq 2n e^{-\frac{(R-B)^2}{2\sigma^2}}\leq \beta/2.
\end{equation*}
Therefore, with probability at least $1-\beta/2$, 
\begin{align}
   \nonumber \hat{\theta} - \theta &= \frac{1}{2}\Bigg(\frac{1}{n_0}\sum_{i \in \{Y_i=0\}}\textsc{clip}(X_i, R)-\mu_0 + G_0 + \frac{1}{n_1}\sum_{i \in \{Y_i=1\}}\textsc{clip}(X_i, R) -\mu_1+ G_1\Bigg) \nonumber   \\
  \nonumber  &=~\frac{1}{2}\left(\frac{1}{n_0}\sum_{i \in \{Y_i=0\}}(X_i-\mu_0) + G_0 + \frac{1}{n_1}\sum_{i \in \{Y_i=1\}}(X_i -\mu_1)+ G_1\right) \\
 \nonumber   &\sim~\frac{1}{2}\mathcal{N}\left(0,\frac{\sigma^2}{n_0}+\frac{\sigma^2}{n_1}+ \sigma_{DP_1}^2 + \sigma_{DP_0}^2\right)\\
    &= \frac{1}{2}\mathcal{N}\left(0, \frac{\sigma^2}{n_0}(1+r)+\frac{2R^2 \log(1.25/\delta)}{n_0^2 \epsilon^2} \cdot (1+ r^2)\right)
\label{eq:dist_DP_theta}
\end{align}
Combining \eqref{eq:dist_DP_theta} with the fact that for  $Z \sim \mathcal{N}(0,\nu^2)$ we have the inequality $
\Pr(\left|Z\right| > t) \leq 2e^{-\frac{t^2}{2 \nu^2}}$ we conclude that with probability $1-\beta$,
\begin{equation*}
   |\hat{\theta} - \theta|  \leq \sqrt{\frac{\sigma^2}{n_0}(1+r)+\frac{2R^2 \log(1.25/\delta)}{n_0^2 \epsilon^2} \cdot (1+ r^2)} \sqrt{2\log(4/\beta)}.
\end{equation*}
This completes the utility proof of the proposed estimator.

Let's now turn to the question of the best achievable deviation. We note that in this example the MLE of $\theta$ is available in closed form, namely  $\hat{\theta}_{\mathrm{MLE}}=(\bar{X}^0+\bar{X}^1)/2$, where $\bar{X}^0=\frac{1}{n_0}\sum_{i=1}^n(1-y_i)X_i$ and $\bar{X}^1=\frac{1}{n_1}\sum_{i=1}^ny_iX_i$. Furthermore $\hat\theta_{\mathrm{MLE}}\sim N(0,\frac{(1+r)\sigma^2}{4n_0})$ is the minimum variance unbiased estimator. It is well know that in this case the narrowest confidence interval for $\theta$ is $[\hat\theta-\sigma \sqrt{(1+r)/n_0}\Phi^{-1}(1-\beta/2),\hat\theta+\sigma \sqrt{(1+r)/n_0}\Phi^{-1}(1-\beta/2)]$. 

\end{proof}

\subsubsection{Linking to Imbalanced Metrics} \label{app:imbalanced_metrics}
A natural question given \Cref{prop:warmup} is whether we can \textit{re-weight} minority class examples to improve the performance of our classifier on specific, imbalanced class focused metrics, like Recall.

To do this, we consider an alternative, re-weighted classifier $f_{\theta_\gamma}$ where $\theta_\gamma=\gamma\mu_1+(1-\gamma)\mu_0.$ Note that the optimal Bayes classifier is in this model class, and can be written $f_{\theta_{1/2}}(X)$. We will denote the Gaussian random variables for the majority and minority classes as $Z_0 \sim \mathcal{N}(\mu_0,\sigma_0^2)$ and $Z_1 \sim \mathcal{N}(\mu_1,\sigma_1^2)$ respectively, and the Gaussian random variable for the added, zero-centered noise for privacy as $Z \sim \mathcal{N}(0,\sigma^2)$.

Then, we can derive a population version of \textit{true positive rate} under $f_{\theta_\gamma}$ as,
\begin{align*}
    \mathrm{TPR}&=\Pr(f_{\theta_{\gamma}}(X)=1|Y=1)\\
    &=\Pr(Z_1 \geq \gamma\mu_1+(1-\gamma)\mu_0)\\
    &=\Phi(\frac{(1-\gamma)(\mu_1-\mu_0)}{\sigma})~.
\end{align*}

Our original imbalance ratio $r$ is a sample-specific estimate of the true population imbalance ratio, $r^*=\Pr(Y=0)/\Pr(Y=1)$. Here, we will consider $r^*$, alongside a population version of positive rate  $\mathrm{PR}=\Pr(Y=1)=\frac{1}{1+r^*}$. This gives us insight into an exact form of the Recall metric for \Cref{ex:warmup}, which is $\frac{\mathrm{TP}}{\mathrm{P}}=\Phi(\frac{(1-\gamma)(\mu_1-\mu_0)}{\sigma})(1+r^*).$ Note that Recall gets \textit{worse} as the imbalance ratio increases. However taking $\gamma<1/2$ \textit{improves} Recall relative to the standard optimal Bayes classifier. One way to choose such a $\gamma$ is to take the inverse probability weight $\gamma=1/\Pr(Y=1)=1+r^*$. These population parameter considerations motivate an empirical weighted counterpart to $f_{\theta_{1/2}}(X)$. 
Such precise distributional knowledge is rarely known in practice and unverifiable under differential privacy (where the data cannot be accessed directly without noise mechanisms). 
Instead we will by default account for $r$ 
by setting weights inversely proportional to class prevalence in our weighted methods, as motivated by this reasoning and prior work \cite{chawla2004special}. 

\subsubsection{Quantifying the Benefits of Re-weighting in Imbalanced Metrics}
\vspace{0.2cm}

Here we provide detailed calculations showing that the re-weighted classifier $f_{\theta_\gamma}$ outperforms the optimal Bayes classifier $f_{\theta_{1/2}}$ in various imbalanced metrics, summarized in Table \ref{tab.imbalancedmetrics}. Recall that $f_{\theta_\gamma} $ is a threshold-based classifier defined as 
\begin{align}
f_{\theta_\gamma}(X) = \mathbb{I}(X \geq \theta_\gamma),
\end{align}
where $ \theta_\gamma = \gamma\mu_1 + (1-\gamma)\mu_0 $. The parameter $ \theta_\gamma $ is a weighted average of the means for each class, $ \mu_1 $ and $ \mu_0$, with $ \gamma $ controlling the weight given to each class. 

\begin{table}[htb]
    \centering
    \begin{tabular}{@{}lcl@{}}
        \toprule
        \textbf{Metric} & \textbf{Formula} \\
        \midrule
        Recall ($\textrm{Re}(\gamma)$) & 
        $\left(1 + r^*\right) \cdot \Phi\left((1-\gamma)\Delta\right)$ \\
        \midrule
        Precision ($\textrm{Pre}(\gamma)$) & 
        $\frac{\Phi\left((1-\gamma)\Delta\right)}{\Phi\left((1-\gamma)\Delta\right) + (1 + r^*) \cdot \left[1 - \Phi\left(\gamma \Delta\right)\right]}$ \\
        \midrule
        Balanced Accuracy ($\textrm{BA}(\gamma)$) & 
        $\frac{\Phi\left((1-\gamma)\Delta\right) + \Phi\left(\gamma\Delta\right)}{2}$ \\
        \midrule
        F1 Score ($\textrm{F1}(\gamma)$) & 
        $\frac{\Phi\left((1-\gamma)\Delta\right)}{\Phi\left((1-\gamma)\Delta\right) + \frac{1}{2} \left[1 - \Phi\left(\gamma\Delta\right)\right]}$ \\
        \bottomrule
    \end{tabular}
    \caption{Some imbalanced classification metrics, defined as functions of the imbalance weight parameter $\gamma$ for the reweighted classifier $f_{\theta_\gamma}$, where $\Delta = \frac{\mu_1 - \mu_0}{\sigma}$ for ease of presentation.}\label{tab.imbalancedmetrics}
\end{table}

\paragraph{Recall Metric.}
The True Positive Rate (TPR) is defined as the probability that the classifier correctly identifies a positive instance. For a given $ X $ sampled from the positive class ($ Y = 1 $), we have,
\begin{align}
\mathrm{TPR} &= \text{Pr}(f_{\theta_{\gamma}}(X) = 1 \mid Y = 1) \nonumber \\
&= \text{Pr}(Z_1 \geq \gamma \mu_1 + (1-\gamma) \mu_0)~.\nonumber 
\end{align}
Since $ X \mid Y = 1 $ is distributed as $ \mathcal{N}(\mu_1, \sigma^2) $, we can standardize this normal variable as,
\begin{align}
\mathrm{TPR} = \text{Pr}\left(\frac{X - \mu_1}{\sigma} \geq \frac{\gamma \mu_1 + (1-\gamma) \mu_0 - \mu_1}{\sigma}\right). \nonumber 
\end{align}
Thus, the TPR can be written using the cumulative distribution function (CDF) of the standard normal, denoted here as $\Phi$, given
\begin{align}
\mathrm{TPR} = \Phi\left(\frac{(1-\gamma)(\mu_1 - \mu_0)}{\sigma}\right)~. \nonumber 
\end{align}

We define the population imbalance ratio $r^*$ as the ratio of the probability of the negative class to the probability of the positive class i.e.
\begin{align}
r^* = \frac{\text{Pr}(Y = 0)}{\text{Pr}(Y = 1)}. \nonumber 
\end{align}
The total probability of positives (i.e. positive rate $\mathrm{PR}$ is just
\begin{align}
\mathrm{PR} = \text{Pr}(Y = 1) = \frac{1}{1 + r^*}~. \nonumber 
\end{align}
Recall is defined as the ratio of true positives to the total actual positives, or
\begin{align}
\text{Recall} = \frac{\mathrm{TPR}}{\mathrm{PR}} = \frac{\Phi\left(\frac{(1-\gamma)(\mu_1 - \mu_0)}{\sigma}\right)}{\frac{1}{1 + r^*}}~. \nonumber 
\end{align}
Simplifying gives,
\begin{align}
\text{Recall} = \textrm{Re}(\gamma) = \left(1 + r^*\right) \cdot \Phi\left(\frac{(1-\gamma)(\mu_1 - \mu_0)}{\sigma}\right)~. \nonumber 
\end{align}
This shows that $\textrm{Re}(\gamma)$ decreases as the imbalance ratio $ r^* $ increases, because the term $ \left(1 + r^*\right) $ magnifies the effect of the Gaussian term.

This implies that $\textrm{Re}(\gamma)$ can be improved relative to the Bayes Optimal Classifier by choosing  $ \gamma < \frac{1}{2} $. This adjustment shifts the threshold $ \theta_\gamma $ to be more inclusive of the positive class, thereby increasing the true positive rate.

\paragraph{Precision Metric.}
Precision is defined as the ratio of true positive rate to all positive predictions, or
\begin{align}
\text{Precision} = \frac{\mathrm{TPR}}{\mathrm{TPR} + \mathrm{FPR}}, \nonumber 
\end{align}
where FPR denotes False Positive Rate.

False Positive Rate (FPR) are defined as the probability that the classifier incorrectly identifies a negative instance as positive. For $ X $ sampled from the negative class ($ Y = 0 $),
\begin{align}
\mathrm{FPR} &= \text{Pr}(f_{\theta_{\gamma}}(X) = 1 \mid Y = 0) \nonumber \\
&= \text{Pr}(Z_0 \geq \gamma \mu_1 + (1-\gamma) \mu_0)~. \nonumber 
\end{align}
We can similarly standardize this normal variable:
\begin{align*}
\mathrm{FPR} = \text{Pr}\left(\frac{X - \mu_0}{\sigma} \geq \frac{\gamma \mu_1 + (1-\gamma) \mu_0 - \mu_0}{\sigma}\right),
\end{align*}
which simplifies to,
\begin{align*}
\mathrm{FPR} = \text{Pr}\left(Z \geq \frac{\gamma (\mu_1 - \mu_0)}{\sigma}\right),
\end{align*}
where $ Z \sim \mathcal{N}(0,1) $. Thus, FPR can be written as:
\begin{align*}
\mathrm{FPR} = 1 - \Phi\left(\frac{\gamma (\mu_1 - \mu_0)}{\sigma}\right).
\end{align*}
This yields
\begin{align*}
\text{Precision} &= \textrm{Pre}(\gamma) = \frac{\mathrm{TPR}}{\mathrm{TPR} + \mathrm{FPR}} = \frac{\Phi\left(\frac{(1-\gamma)(\mu_1 - \mu_0)}{\sigma}\right)}{\Phi\left(\frac{(1-\gamma)(\mu_1 - \mu_0)}{\sigma}\right) + (1 + r^*) \cdot \left[1 - \Phi\left(\frac{\gamma (\mu_1 - \mu_0)}{\sigma}\right)\right]}.
\end{align*}

As $ r^* $ increases, the FPR becomes larger, leading to a potential decrease in $\textrm{Pre}(\gamma)$. This result shows that in adjusting $ \gamma $, better performance can be achieved on either $\textrm{Re}(\gamma)$ or $\textrm{Pre}(\gamma)$.

\paragraph{Balanced Accuracy.}
Balanced accuracy is defined as the average of TPR and TNR.
We require the following explicit formulas for TPR and TNR, where TPR was previously defined for Recall:
\begin{align*}
\text{TPR} = \Phi\left(\frac{(1-\gamma)(\mu_1 - \mu_0)}{\sigma}\right),
\end{align*}
and 
\begin{align*}
 \text{TNR} &=\text{Pr}\left(Z_0 \leq \theta_\gamma\right)= \text{Pr}\left(Z \leq \gamma\left(\frac{\mu_1 - \mu_0}{\sigma}\right){}\right) = \Phi\left(\frac{\gamma(\mu_1 - \mu_0)}{\sigma}\right).
\end{align*}
Therefore, Balanced Accuracy is simply
\begin{align*}
\textrm{BA}(\gamma)=\frac{\Phi\left(\frac{(1-\gamma)(\mu_1 - \mu_0)}{\sigma}\right) + \Phi\left(\frac{\gamma(\mu_1 - \mu_0)}{\sigma}\right)}{2}.
\end{align*}

\paragraph{F1 Score.}
F1 Score can be written as:
\begin{align*}
\text{F1 Score} = \frac{2 \cdot \text{TPR}}{2 \cdot \text{TPR} + \text{FPR} + \text{FNR}}.
\end{align*}
The expressions for TPR, FPR, and FNR have been previously derived as follows:
\begin{align*}
\text{TPR} &= \Phi\left(\frac{(1-\gamma)(\mu_1 - \mu_0)}{\sigma}\right),\\
\text{FPR} &= 1 - \Phi\left(\frac{\gamma(\mu_1 - \mu_0)}{\sigma}\right),\\
\text{FNR} &= 1 - \Phi\left(\frac{(1-\gamma)(\mu_1 - \mu_0)}{\sigma}\right).
\end{align*}
Substituting these yields an expression for the F1 Score:
\begin{align*}
\textrm{F1}(\gamma) = \frac{2 \cdot \Phi\left(\frac{(1-\gamma)(\mu_1 - \mu_0)}{\sigma}\right)}{2 \cdot \Phi\left(\frac{(1-\gamma)(\mu_1 - \mu_0)}{\sigma}\right) + \left[1 - \Phi\left(\frac{\gamma(\mu_1 - \mu_0)}{\sigma}\right)\right] + \left[1 - \Phi\left(\frac{(1-\gamma)(\mu_1 - \mu_0)}{\sigma}\right)\right]}.
\end{align*}
Simplifying yields:
\begin{align*}
\textrm{F1}(\gamma) = \frac{\Phi\left(\frac{(1-\gamma)(\mu_1 - \mu_0)}{\sigma}\right)}{\Phi\left(\frac{(1-\gamma)(\mu_1 - \mu_0)}{\sigma}\right) + \frac{1}{2} \left[1 - \Phi\left(\frac{\gamma(\mu_1 - \mu_0)}{\sigma}\right)\right]}.
\end{align*}

\subsection{Weighted Private ERMs}\label{app:weighted_erm_proof}

\textbf{Assumptions from \cite{giddens2023differentially}.}~~~We list here (for completeness) the undesirable assumptions from \cite{giddens2023differentially} that we overcome. Their privacy proof works only for loss functions that take in a single argument, which excludes standard models like logistic regression, SVM, and others. Additionally, they made the  assumption that the difference of weights across neighboring datasets goes to 0 as $n \rightarrow \infty$, which is too strong for our inverse proportional weights strategy. We also note that in differential privacy, sensitivity is analyzed under worst case assumptions even if the influence of a single data point diminishes as $n$ grows large. One therefore should avoid privacy statements that rely on asymptotic assumptions.

\paragraph{Notation for ERM Proof.} For parity and ease of comparison, we will use mostly overlapping notation with \cite{chaudhuri2011differentially}. We will denote the $\ell_2$-norm by $\norm{\mbf{x}}$. For an integer $n$, the notation $[n]$ will represent the set $\{1, 2, \ldots, n\}$. Boldface will be used for vectors, and calligraphic type for sets. For a square matrix $A$, the induced $L_2$ norm will be indicated by $\norm{A}_2$. Algorithms will accept as input \textit{training data} $\mc{D} = {(\mbf{x}_i, \mbf{y}_i) \in \mc{X} \times \mc{Y} : i = 1, 2, \ldots, n}$, consisting of $n$ data-label pairs. In binary classification, the data space is $\mc{X} = \mathbb{R}^d$ and the label set is $\mc{Y} = {0, 1}$. It will be assumed throughout that $\mc{X}$ is the unit ball, hence $\norm{\mbf{x}_i}_2 \le 1$. Note that the extension of the proof to $\|\mathbf{x}_i\|\leq q$ is straightforward and commonly implemented in practice. This is also how we implemented our code.

We aim to construct a \textit{predictor} $\mbf{f} : \mc{X} \to \mc{Y}$. The quality of our predictor on the training data is assessed using a nonnegative \textit{loss function} $\loss : \mc{Y} \times \mc{Y} \to \mathbb{R}$. In regularized empirical risk minimization (ERM), we select a predictor $\mbf{f}$ that minimizes the regularized empirical loss, optimizing over $\mbf{f}$ within a hypothesis class $\mc{H}$. The regularizer $\reg N(\mbf{f})$ is used to prevent over-fitting, for some function $N$ of the predictor. Altogether, this yields the ERM loss function:
\begin{align*}
    \obj(\mbf{f},\mc{D}) = \frac{1}{n} \sum_{i=1}^{n} \loss(\mbf{y}_i,
    \mbf{f}(\mbf{x}_i)) + \reg N(\mbf{f})~.
\end{align*}

We can slightly modify the regularized ERM by introducing a weighting scheme to correct for class imbalance. Let $\mathbf{w} = [w_1, w_2, \ldots, w_n]$ be a vector of sample weights, where each $w_i$ corresponds to a weight assigned to the $i$-th sample in the dataset $\mathcal{D}$. This yields,
\begin{align*}
    J(\mathbf{f},\mathcal{D}, \mathbf{w}) = \frac{1}{n} \sum_{i=1}^{n} w_i \cdot \ell(\mbf{y}_i,\mathbf{f}(\mathbf{x}_i)) + \reg N(\mathbf{f}).
\end{align*}
We consider weights $w_i$ that do not explicitly affect the regularization term $\reg N(\mathbf{f})$, as is standard, as regularization should penalize model complexity independent of class imbalance or weighting. 

\paragraph{Ridge Regression.} From here on, we will focus on \textit{ridge regression}, so instead of a penalty of the form $\reg N(\mathbf{f})$ we will use $\frac{\lambda}{2}\|\boldsymbol{\beta}\|^2$, where our predictor is $\mbf{x}^T\bbeta$ and $\bbeta$ is a vector of coefficients that can be multiplied with a sample vector $\mbf{x}$ to produce a prediction.

A common choice of weight vector $\mathbf{w} = [w_1, w_2, \ldots, w_n]$ is to compose weights such that they correspond to the inverse frequency of the class label in the training set \cite{provost1997analysis}. In other words, $w_i$ is inversely proportional to the prevalence of the class label $\mathbf{y}_i$ associated with each sample $(\mathbf{x}_i, \mbf{y}_i)$. Let $n$ be the number of total samples in dataset $D$, and $Y$ be the set of unique class labels. Then let $\hat{\pi}_k=\frac{1}{n}\sum_{i=1}^n\mathbb{I}[y_i=k]$, $\hat{\pi}=(\hat{\pi}_0,\hat{\pi}_1)$ and define the weights $w_k=\frac{\|\hat{\pi}\|_1}{\pi_k}$  for $k\in\{0,1\}$.

For completeness, we reproduce standard definitions in the form they appear in \cite{chaudhuri2011differentially}, including a slightly stronger variation of Definition \ref{def.dp} then what is described in \Cref{sec:prelims}.

\paragraph{Assumptions on loss.}
We make almost the same loss assumptions as \cite{chaudhuri2011differentially}. Here, we restate definitions of \textit{strictly convex} and \textit{$\tau$-strongly convex} from their paper for convenience. We also require the convex loss function $\loss(\cdot,\cdot)$ to be \textbf{twice} differentiable functions with respect to $\bbeta$, and that $|\frac{\partial}{\partial\eta}\ell(y,\eta)|\leq 1$ and $|\frac{\partial^2}{\partial\eta^2}\ell(y,\eta)|\leq c$ for some fixed $c$.  
\begin{definition}
A function $H(\bbeta)$ over $\bbeta \in \bbR^d$ is {\em strictly convex} if for all $\alpha \in (0,1)$, $\bbeta$, and $\bbeta'$,
	\begin{align*} 
	H\left( \alpha \bbeta + (1 - \alpha) \bbeta' \right) < \alpha H(\bbeta) + (1 - \alpha) H(\bbeta').
	\end{align*}
It is {\em $\tau$-strongly convex} if for all $\alpha \in (0,1)$, $\bbeta$, and $\bbeta'$, 
	\begin{align*}
	H\left( \alpha \bbeta + (1 - \alpha) \bbeta' \right) 
		\leq \alpha H(\bbeta) + (1 - \alpha) H(\bbeta')
			- \frac{1}{2} \tau \alpha (1 - \alpha) \norm{\bbeta - \bbeta'}_2^2.
	\end{align*}
\end{definition}

\paragraph{Privacy model.}
Assume $\mc{A}(\mc{D})$ generates a classifier, and let $\mc{D}'$ be a dataset that differs from $\mc{D}$ in one entry (assumed to be the private value of one individual). They are neighboring datasets in the standard sense, e.g. $\mc{D}'$ and $\mc{D}$ share $n-1$ points $(\mbf{x}_i, y_i)$. The algorithm $\mc{A}$ ensures DP if, for any set $\SS$, the probability that $\mc{A}(\mc{D}) \in \SS$ is close to the probability that $\mc{A}(\mc{D}') \in \SS$, with the probability taken over the randomness in the algorithm.
\begin{definition}\label{def:stronger_dp}
An algorithm $\mc{A}(\mc{B})$ taking values in a set $\mc{T}$ provides $\priveps$-DP if 
\begin{align*}
\sup_{\mc{S} \subseteq \mc{T}} \sup_{\mc{D}, \mc{D}'} \frac{ \mu\left( \mc{S} ~|~ \mc{B} = \mc{D} \right) 
	}{
\mu\left( \mc{S} ~|~ \mc{B} = \mc{D}' \right) 
	}	
	\le 
	e^{\priveps},
\end{align*}
where the first supremum is over all measurable $\mc{S} \subseteq \mc{T}$, the second is
over all datasets $\mc{D}$ and $\mc{D}'$ differing in a single entry, and
$\mu(\cdot|\mc{B})$ is the conditional distribution (measure) on $\mc{T}$
induced by the output $\mc{A}(\mc{B})$ given a dataset $\mc{B}$.  The ratio is interpreted to be 1 whenever the numerator and denominator are both 0.
\label{def:densitypriv}
\end{definition}
We also restate sensitivity, as it appears in \cite{chaudhuri2011differentially}. Consider $g: (\bbR^m)^n \rightarrow \bbR$, a scalar function of $z_1, \ldots, z_n$, where each $z_i \in \bbR^m$ represents the private value of individual $i$; the sensitivity of $g$ is defined as follows.
\begin{definition}\label{def:sens_from_chaud}
The sensitivity of a function $g: (\bbR^m)^n \rightarrow \bbR$ is the maximum change in the value of $g$ when one entry of the input database changes. More formally, the sensitivity $S(g)$ of $g$ is defined as:
	\begin{align*}
	S(g) = \max_{i\in[n]} \max_{z_1, \ldots, z_n, z'_i} 
		\left|
			g(z_1, \ldots, z_{i-1}, z_i, z_{i+1},\ldots, z_n) 
			- g(z_1, \ldots, z_{i-1}, z'_i, z_{i+1}, \ldots, z_n)
		\right|.
	\end{align*}
\end{definition}
For the function $A(\mc{D}) = \argmin \obj(\bbeta, \mc{D})$, the output is a vector $A(\mc{D}) + \b$, where $\b$ is random noise with a density of $\nu(\b) = \frac{1}{\alpha} e^{- \gamma\norm{\b}}$, where $\alpha$ is the normalizing constant. The parameter $\gamma$ depends on $\priveps$ and the $L_2$-\textit{sensitivity} of $A(\cdot)$.
\begin{definition}
The $L_2$-sensitivity of a vector-valued function
is defined as the maximum change in the $L_2$ norm of the value of $g$ when one entry of the input database changes. More formally,
	\begin{align*} 
	S(A) = \max_i \max_{z_1, \ldots, z_n, z'_i} \left\|
		A(z_1, \ldots, z_i,\ldots) - A(z_1, \ldots, z'_i, \ldots)
		\right\|.
	\end{align*}
\end{definition}

\paragraph{Objective perturbation.} The approach to private ERM first proposed by \cite{chaudhuri2011differentially} adds noise to the objective function itself and then produces the minimizer of the perturbed objective.  The perturbed objective is:
\begin{align*}
    \obj_{\mathrm{priv}}(\bbeta,\mc{D}) 
        &= \obj(\bbeta, \mc{D}) + \frac{1}{n} \b^T \bbeta,
\end{align*}
Note that the privacy parameter here does not depend on the sensitivity of the of the classification algorithm.
That is, the privacy parameter $\epsilon$ is determined by the amount of noise added to the objective function through $\frac{1}{n} \mathbf{b}^T \boldsymbol{\beta}$, and it depends on the properties of the loss function and the regularizer rather than on the sensitivity of the classification algorithm's output. With the addition of a weight vector $\mathbf{w}$, this is perturbed objective becomes:
\begin{align*}
    \obj_{\mathrm{priv}}(\bbeta,\mathcal{D}, \mathbf{w}) &= \obj(\bbeta, \mc{D}, \mathbf{w}) + \frac{1}{n} \b^T \bbeta,
\end{align*}

\subsubsection{Privacy of \Cref{alg:objective}}
\label{sec:privacy_proof_erm}
In this section, we show that \Cref{alg:objective} using the weighted ERM objective function $\obj_{\mathrm{priv}}(\bbeta,\mc{D}, \mathbf{w})$ is $\priveps$-differentially private.
e.g. the output of the weighted $\obj_{\mathrm{priv}}(\bbeta,\mc{D}, \mathbf{w})$ is $(\priveps,0)$-differentially private. We assume for each $w_i \in \mathbf{w}$, $|w_i| \leq 1$. 
Note in particular that our analysis covers the case of \textit{logistic regression}, which as stated, \cite{chaudhuri2011differentially} does not. Still, much of what follows is a adapted directly from the proof given by \cite{chaudhuri2011differentially}, with careful accounting for the weights vector $\mathbf{w}$; for sake of completeness and ease of comparison, all steps are stated as closely as possible to what appears in the prior work.

\ermpriv*

\begin{proof}
Consider $\bbeta_{priv}$ output by \Cref{alg:objective}. We observe that given {\em any} fixed
$\bbeta_{priv}$ and a fixed dataset $\D$, there always exists a $\b$ such that
\Cref{alg:objective} outputs $\bbeta_{priv}$ on input $\D$.  Because $\loss$ is differentiable 
and convex, and $N(\cdot)$ is differentiable, 
we can take the gradient of the objective function and set it to
$\mbf{0}$ at $\bbeta_{priv}$. Therefore, we set
\begin{align*}
\nonumber 0~&=\nabla \obj_{\mathrm{priv}}(\bbeta_{priv},\mc{D}, \mathbf{w}) \\
\nonumber &=\nabla \obj(\bbeta_{priv}, \mc{D}, \mathbf{w}) +  \frac{1}{n} \b +\Delta\bbeta_{priv}\\
\nonumber~&= \frac{1}{n} \sum_{i=1}^{n} w_i \cdot \nabla \ell(y_i,\mbf{x}_i^T\bbeta_{priv}) + (\lambda+\Delta) \bbeta_{priv} +  \frac{1}{n} \b,
\end{align*}
and therefore
\begin{equation}
    \b =  - \sum_{i=1}^{n} w_i~\cdot~\loss'(y_i,\mbf{x}_i^T\bbeta_{priv}) \mbf{x}_i - n \extra \bbeta_{priv} \label{eqn:bf}.
\end{equation}

We claim that as $\loss$ is \textit{twice} differentiable and $\obj(\bbeta, \D) + \frac{\extra}{2}
||\bbeta||^2$ is strongly convex, given a dataset $\mc{D} = (\x_1, y_1),
\ldots, (\x_n, y_n)$, there is a bijection between $\b$ and $\bbeta_{priv}$.  Equation \eqref{eqn:bf} shows that two different $\b$ values cannot result in the same $\bbeta_{priv}$.  Furthermore, since the objective is strictly convex, for a
fixed $\b$ and $\mc{D}$, there is a unique $\bbeta_{priv}$; therefore the map
from $\b$ to $\bbeta_{priv}$ is injective.  The relation Equation \eqref{eqn:bf} also shows that for any $\bbeta_{priv}$, there exists a $\b$ for which $\bbeta_{priv}$ is the minimizer, so the map from $\b$ to $\bbeta_{priv}$ is surjective.

To show $\priveps$-DP, we need to compute the ratio $g( \bbeta_{priv} | \mc{D} )/g( \bbeta_{priv} |
\mc{D}' )$ of
the densities of $\bbeta_{priv}$ under the two datasets $\mc{D}$ and $\mc{D'}$. This ratio can be
written as:
\begin{align*}
	\frac{ g( \bbeta_{priv} | \mc{D} )}{ g( \bbeta_{priv} | \mc{D}' ) } 
	&= \frac{ \mu( \b | \mc{D} )}{ \mu( \b' | \mc{D}' ) } 
		\cdot \frac{ |\det(\mbf{J}(\bbeta_{priv} \to \b |
		\mc{D}))|^{-1} }{ |\det(\mbf{J}(\bbeta_{priv} \to \b' |
		\mc{D}'))|^{-1} },
\end{align*}
where $\mbf{J}(\bbeta_{priv} \to \b | \mc{D})$, $\mbf{J}(\bbeta_{priv} \to \b
| \mc{D'})$ are the Jacobian matrices of the mappings from $\bbeta_{priv}$ to $\b$,
and $\mu(\b | \mc{D})$ and $\mu(\b | \mc{D'})$ are the densities
of $\b$ given the output $\bbeta_{priv}$, when the datasets are $\mc{D}$ and $\mc{D'}$ respectively.

First, we bound the ratio of the Jacobian determinants.   Let $\b^{(j)}$ denote the $j$-th coordinate of $\b$.  From Equation \eqref{eqn:bf} we have,
	\begin{equation*}
	\b^{(j)} =  
		- \sum_{i=1}^{n} w_i~\cdot~\loss'(y_i, \bbeta_{priv}^T \mbf{x}_i) \mbf{x}_i^{(j)} 
		- n(\lambda+\Delta) \bbeta_{priv}^{(j)}~.
	\end{equation*}

Given a dataset $\D$, the $(j,k)$-th entry of the Jacobian
matrix $\mbf{J}(\mbf{f} \to \b | \mc{D})$ is
	\begin{align*}
	\frac{\partial \b^{(j)} }{\partial \bbeta_{priv}^{(k)}} 
		= 
			- \sum_{i} w_i^2~\cdot~\loss''(y_i,\bbeta_{priv}^T \mbf{x}_i) \mbf{x}_i^{(j)} \x_i^{(k)} 
			- n (\lambda+\extra) \mathbb{I}(j=k),
	\end{align*}
where $\mathbb{I}(\cdot)$ is the indicator function. We note that the Jacobian is
defined for all $\fpriv$ because $\|\bbeta\|^2$ and $\loss$ are globally
twice differentiable.
 
Let $\mc{D}$ and $\mc{D'}$ be two datasets which differ in the value of the $n$-th
item such that \\
$\mc{D} = \{ (\x_1, y_1), \ldots, (\x_{n-1}, y_{n-1}), (\x_n, y_n)
\}$ and $\mc{D'} = \{ (\x_1, y_1), \ldots, (\x_{n-1}, y_{n-1}), (\x'_n,
y'_n) \}$.
Moreover, we define matrices $A$ and $E$ as follows:
	\begin{align*}
	A & =  n \reg \nabla^2 N(\fpriv) 
		+ \sum_{i=1}^{n} w_i^2~\cdot~\loss''(y_i, \fpriv^T \mbf{x}_i) \mbf{x}_i \mbf{x}_i^T  
		+ n \extra I_d\\
	E & = - w_n^2~\loss''(y_n, \fpriv^T \mbf{x}_n) \mbf{x}_n \mbf{x}_n^T 
		+ (w'_n)^2~\loss''(y'_n, \fpriv^T \mbf{x'}_n) \mbf{x'}_n \mbf{x'}_n^T.
	\end{align*}
Then, $\mbf{J}(\fpriv \to \b | \mc{D}) = -A$, and $\mbf{J}(\fpriv \to
\b | \mc{D'}) = -(A + E)$. 

Let $\lambda_1(M)$ and $\lambda_2(M)$ denote the largest and second largest eigenvalues of a matrix $M$.  As $E$ has rank at most $2$, then,
	\begin{align*}
	\frac{ |\det( \mbf{J}(\fpriv \to \b | \mc{D'}))| } { |\det(
	\mbf{J}(\fpriv \to \b | \mc{D}))|} 
	&=
	\frac{|\det(A+E)|}{|\det{A}|} \\
	&= |1 + \lambda_1(A^{-1}E) +
	\lambda_2(A^{-1}E) + \lambda_1(A^{-1}E) \lambda_2(A^{-1}E)|.
	\end{align*}
Since we have
assumed $\loss$ is twice differentiable and convex, any eigenvalue of $A$ is
therefore at least $n\reg + n \extra$; therefore, for $j=1,2$,
$|\lambda_j(A^{-1}E)| \leq \frac{|\lambda_j(E)|}{n (\reg + \extra)}$. 
Applying the triangle inequality to the trace norm:
	\begin{align*}
	|\lambda_1(E)| + |\lambda_2(E)| 
		\le |w_n^2 \loss''(y_n, \fpriv^T \x_n)| \cdot \norm{\x_n}
			+ | -(w'_n)^2 \loss''(y'_n, \fpriv^T \x'_n) | \cdot \norm{\x'_n}.
	\end{align*}
Then upper bounds on $|w_i|$, $|y_i|$, $||\x_i||$, and $|\loss''(z)|$ yield,
	\begin{align*}
	|\lambda_1(E)| + |\lambda_2(E)| \leq 2\c.
	\end{align*}

 So, $|\lambda_1(E)|\cdot |\lambda_2(E)| \leq c^2$, and 
	\begin{align*}
	\frac{|\det(A+E)|}{|\det(A)|} \leq 1 + \frac{2\c}{n (\reg + \extra)} +
\frac{c^2}{n^2 (\reg + \extra)^2}  = \left(1 + \frac{\c}{n (\reg +
\extra)}\right)^2.
	\end{align*}
We now consider two cases. In the first case, $\extra = 0$, and thus by
definition, $1 + \frac{2 \c}{n \reg} + \frac{\c^2}{n^2
\reg^2} \leq e^{\priveps - \priveps'}$. In the second case, $\extra > 0$,
and in this case, by definition of $\extra$, $(1 + \frac{\c}{n (\reg +
\extra)})^2 = e^{\priveps/2} = e^{\priveps - \priveps'}$.
 
Next, we bound the ratio of the densities of $\b$. We observe that as
$|\loss'(z)| \leq 1$, for any $z$ and $|w_i|$, $|y_i|, ||\x_i|| \leq 1$, for datasets
$\mc{D}$ and $\mc{D'}$ which differ by one value,
	\begin{align*}
	\mbf{b'} - \b =  w_n \loss'(y_n, \fpriv^T \mbf{x}_n) \mbf{x}_n 
		-  w'_n \loss'(y_n, \fpriv^T \mbf{x'}_n) \mbf{x'}_n.
	\end{align*}
This implies that:
	\begin{align*} 
		\norm{\b} - \norm{\mbf{b'}} 
		\leq 
		\norm{\b - \mbf{b'}} 
		\leq 2.
	\end{align*}
We can write:
	\begin{align*} 
	\frac{ \mu(\b | \mc{D})}{\mu(\mbf{b'} | \mc{D'})} 
	=
	\frac{||\b||^{d-1} e^{-\priveps' ||\b||/2} 
			\cdot \frac{1}{\surf(||\b||)}
		}{
		||\mbf{b'}||^{d-1} e^{-\priveps' ||\mbf{b'}||/2} 
			\cdot \frac{1}{\surf(||\mbf{b'}||)}} 
	\leq e^{\priveps'(||\b|| - ||\mbf{b'}||)/2} \leq e^{\priveps'},
	\end{align*}
where $\surf(x)$ denotes the surface area of the sphere in $d$ dimensions with
radius $x$. Here the last step follows from the fact that $\surf(x) =
s(1)x^{d-1}$, where $s(1)$ is the surface area of the unit sphere in $\bbR^d$.   

Finally, we are ready to bound the ratio of densities:
	\begin{align*}
	\frac{ g( \fpriv | \mc{D} )}{ g( \fpriv | \mc{D}' ) } 
	&= \frac{ \mu( \b | \mc{D} )}{ \mu( \b' | \mc{D}' ) } 
		\cdot \frac{ |\det(\mbf{J}(\fpriv \to \b | \mc{D}'))| }{
		|\det(\mbf{J}(\fpriv \to \b' | \mc{D}))| } \\
	&= \frac{ \mu( \b | \mc{D} )}{ \mu( \b' | \mc{D}' ) } 
		\cdot \frac{|\det(A + E)|}{|\det{A}|} \\
	&\leq e^{\priveps'} \cdot e^{\priveps - \priveps'} \\
	&\leq e^{\priveps}.
	\end{align*}
Thus, \Cref{alg:objective} satisfies Definition~\ref{def:densitypriv}.
\end{proof}

\section{Additional Experimental Results and Details }\label{app:complete_experiments}

\paragraph{GEM Summary.}  
GEM is an $(\epsilon,\delta)$-DP neural method that fits a private, parameterized weight distribution $G_{\theta}$, where $\theta$ represents the learnable parameters of the model. It follows the Select-Measure-Project paradigm, and its main novelty lies in the \textit{project} step: the method fits a neural network, denoted as $G_{\theta}$, to approximate a distribution over the data domain in a differentially private manner. This network generates a product distribution $P_{\theta}$, where $P_{\theta}$ represents the output distribution over a discretized version of the data domain.

The process works by sampling random Gaussian noise vectors $z$, which are passed through the neural network $G_{\theta}$ to output a distribution $P_{\theta}(z)$ in the same domain as the target data. This product distribution is normalized to ensure it behaves as a valid marginal probability vector. Once fit, arbitrarily many samples can be generated from the fully specified distribution $P_{\theta}$.

Any statistical query $q$ can be described as a function mapping $P_{\theta}$ to a value in $[0,1]$, i.e., $q(P_{\theta}) = \sum_{x \in X} \phi(x) P_{\theta}(x)$, where $\phi(x)$ is the predicate function defining the query. Any query $q$ is differentiable with respect to the parameters $\theta$ of the model. Given a set of queries $\tilde{q}_i \in \tilde{Q}_{1:T}$, which are privately selected using the Exponential Mechanism, and answers $\tilde{a}_i \in \tilde{A}_{1:T}$ privately computed using an additive noise mechanism, a natural loss function for the parameterization $\theta$ is given by:
\begin{align*}
\mathcal{L}_{GEM} \left( \theta, \tilde{Q}_{1:T}, \tilde{A}_{1:T} \right) = \sum_{i \in [T]} \left| \tilde{q}_i (P_{\theta}) - \tilde{a}_i \right|.
\end{align*}
GEM iteratively updates $\theta$ to minimize this loss function, incorporating the observed queries and answers.

\paragraph{PrivBayes Summary.}  
PrivBayes builds a Bayesian network to approximate the joint distribution of the data by factorizing it into a sequence of conditional probabilities, which it can then sample from to create differentially private synthetic data. To ensure DP, it first selects an attribute ordering using mutual information (privatized by an additive noise mechanism) to determine parent-child relationships. Then, for each attribute, it estimates the attribute's conditional probability distribution given its parent attributes using a DP noise-perturbed frequency table. Once the Bayesian network is constructed, synthetic data points are generated by sampling from the learned network.

\paragraph{FTTransformer Summary.}
We adapt a recently proposed transformer-based model, FTTransformer \cite{gorishniy2021revisiting}, to the DP setting, which involves minor adjustments to the model architecture for compatibility with Opacus \cite{opacus}. FTTransformer is a neural tabular data classifier that is competitive with well-known gradient boosting tree-based methods like XGBoost \cite{chen2016xgboost}; its novelty lies in data transformations for attenuation by the attention layers in a transformer architecture \cite{wolf2020transformers, tay2022efficient, khan2022transformers}. Our empirical results rely on modifications to implementations for DP-SGD from the Opacus library \cite{opacus} and the base implementation for FTTransformer from \cite{huang2020tabtransformer}.

After experimenting with different neural architectures in the non-private setting, we found that the FTTransformer architecture was significantly better than other methods on tabular data tasks, even for imbalanced classification. However, when transitioning to the private setting, we found that all of the neural methods using DP-SGD had trouble under class imbalance. FTTransformer still performed best among these (albeit poorly relative to other model classes), so we included the Private FTTransformer implementation to represent the class of neural models trained with DP-SGD (using a weighted cross-entropy loss, which helped a little on imbalanced classification metrics). 

\begin{figure}[h!]
    \centering
    \vspace{-0.3cm}
    \includegraphics[width=0.9\linewidth]{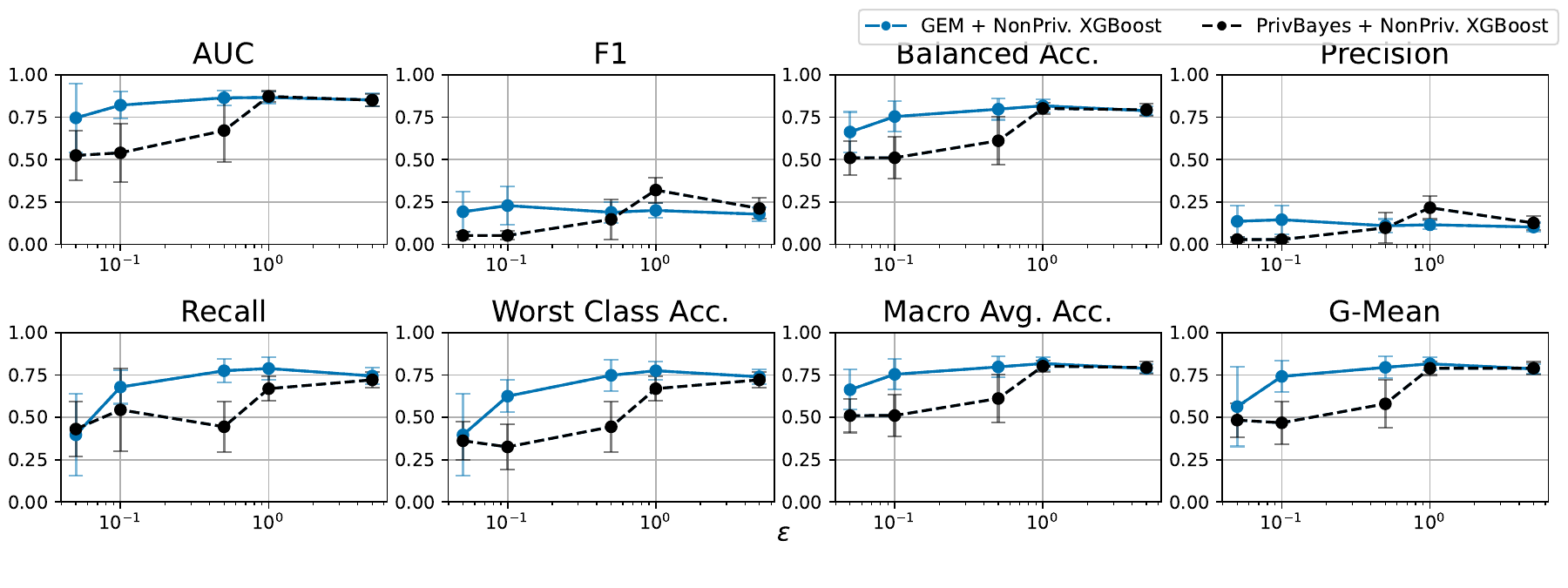}
    \vspace{-0.5cm}
    \caption{Comparison of PrivBayes \cite{zhang2017privbayes}  and GEM \cite{liu2021iterative} as private preprocessing steps on the \textit{mammography} dataset, with XGBoost as the downstream non-private classifier. PrivBayes, while generally weaker, shows similar performance trends to GEM as $\epsilon$ increases and is a strong private pre-processing step for imbalanced classification.}
    \label{fig:privbaye_vs_gem_mammography_rebuttal}
\end{figure}

\subsection{Performance of models on all \texttt{imblearn} datasets}
This section presents figures that detail exhaustive performance across privacy parameter (i.e., varying $\epsilon$ from 0.01 to 5.0) for all the datasets listed in \Cref{tab:imb-learn-tasks}.

\begin{figure}[h!]
    \centering
    \includegraphics[width=0.9\linewidth]{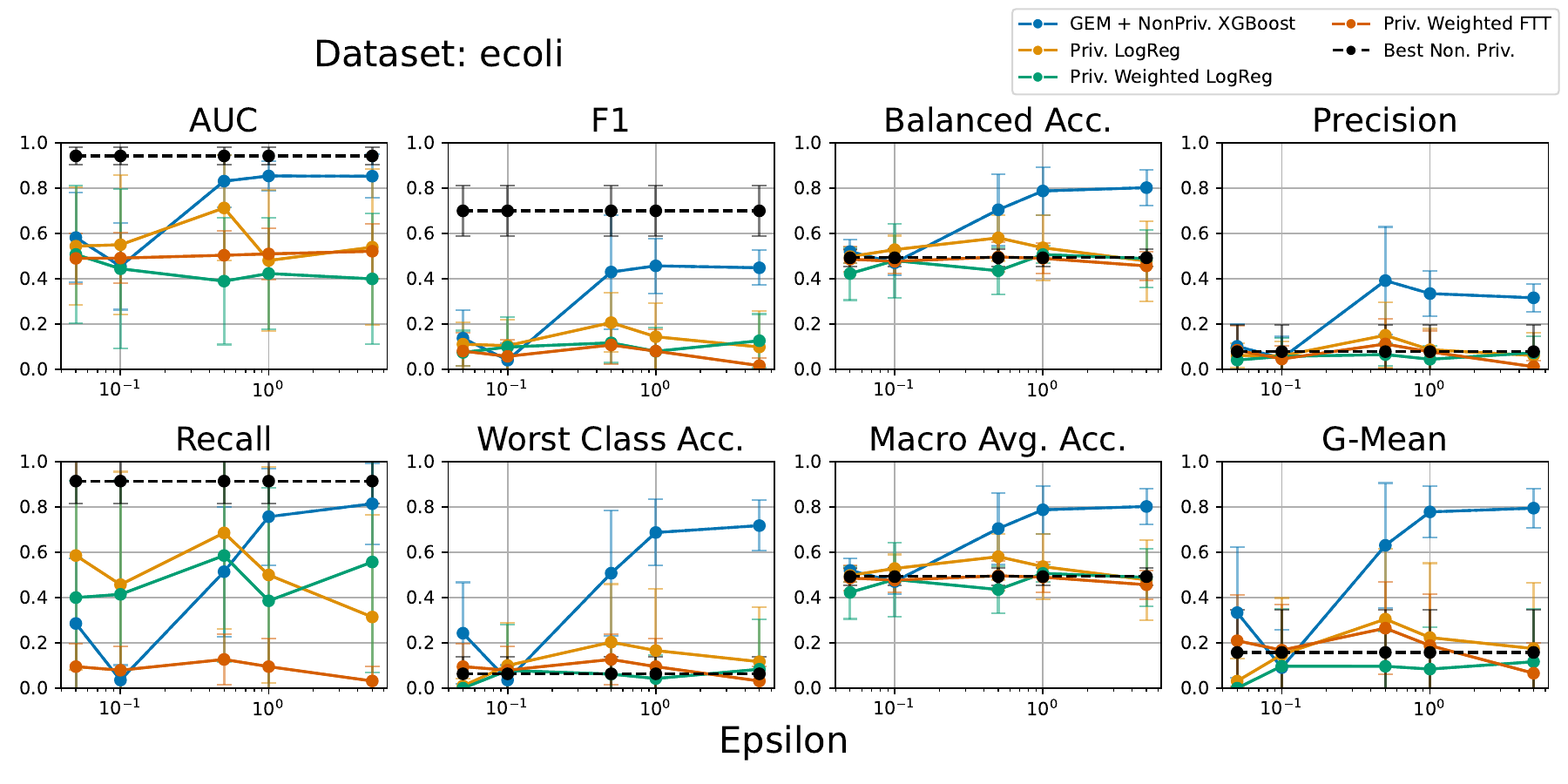}
    \caption{Privacy-preserving predictors across $\epsilon$ settings for \textit{ecoli} dataset.}
    \label{fig:ecoli_private}
\end{figure}

\begin{figure}[h!]
    \centering
    \includegraphics[width=0.9\linewidth]{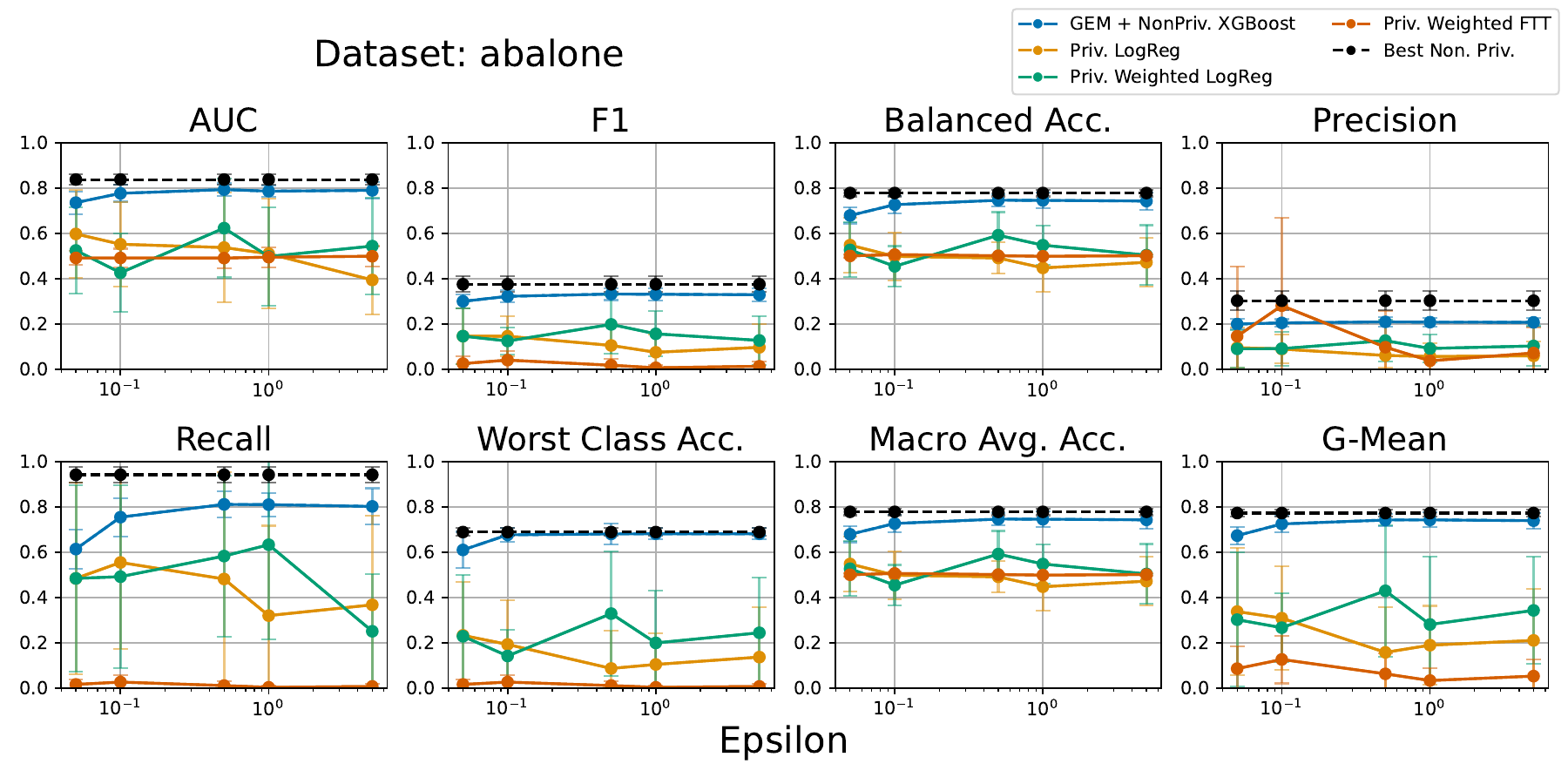}
    \caption{Privacy-preserving predictors across $\epsilon$ settings for \textit{abalone} dataset.}
    \label{fig:abalone_private}
\end{figure}

\begin{figure}[h!]
    \centering
    \includegraphics[width=0.9\linewidth]{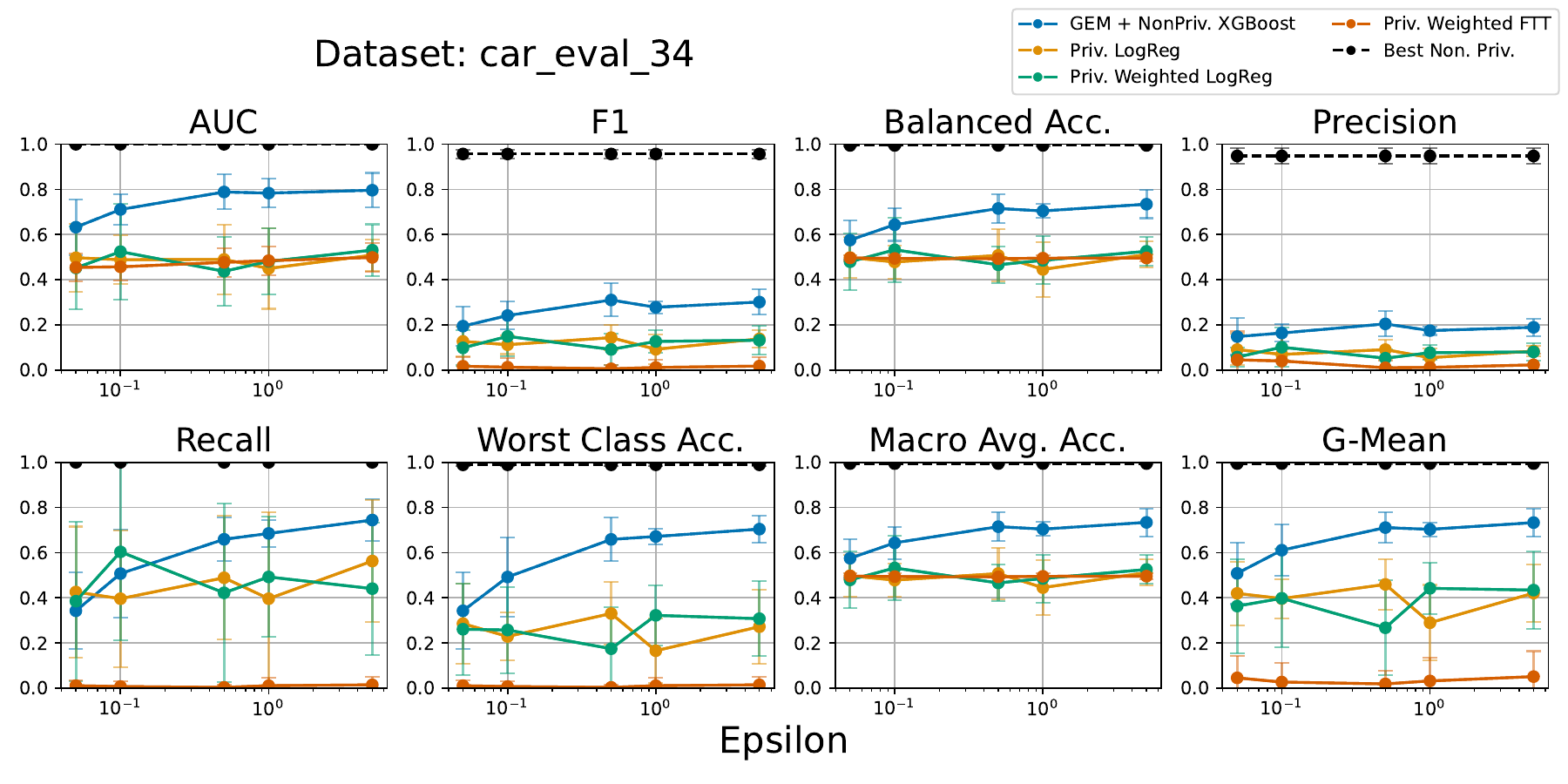}
    \caption{Privacy-preserving predictors across $\epsilon$ settings for $car\_eval\_34$ dataset.}
    \label{fig:car_eval_34_private}
\end{figure}

\begin{figure}[h!]
    \centering
    \includegraphics[width=0.9\linewidth]{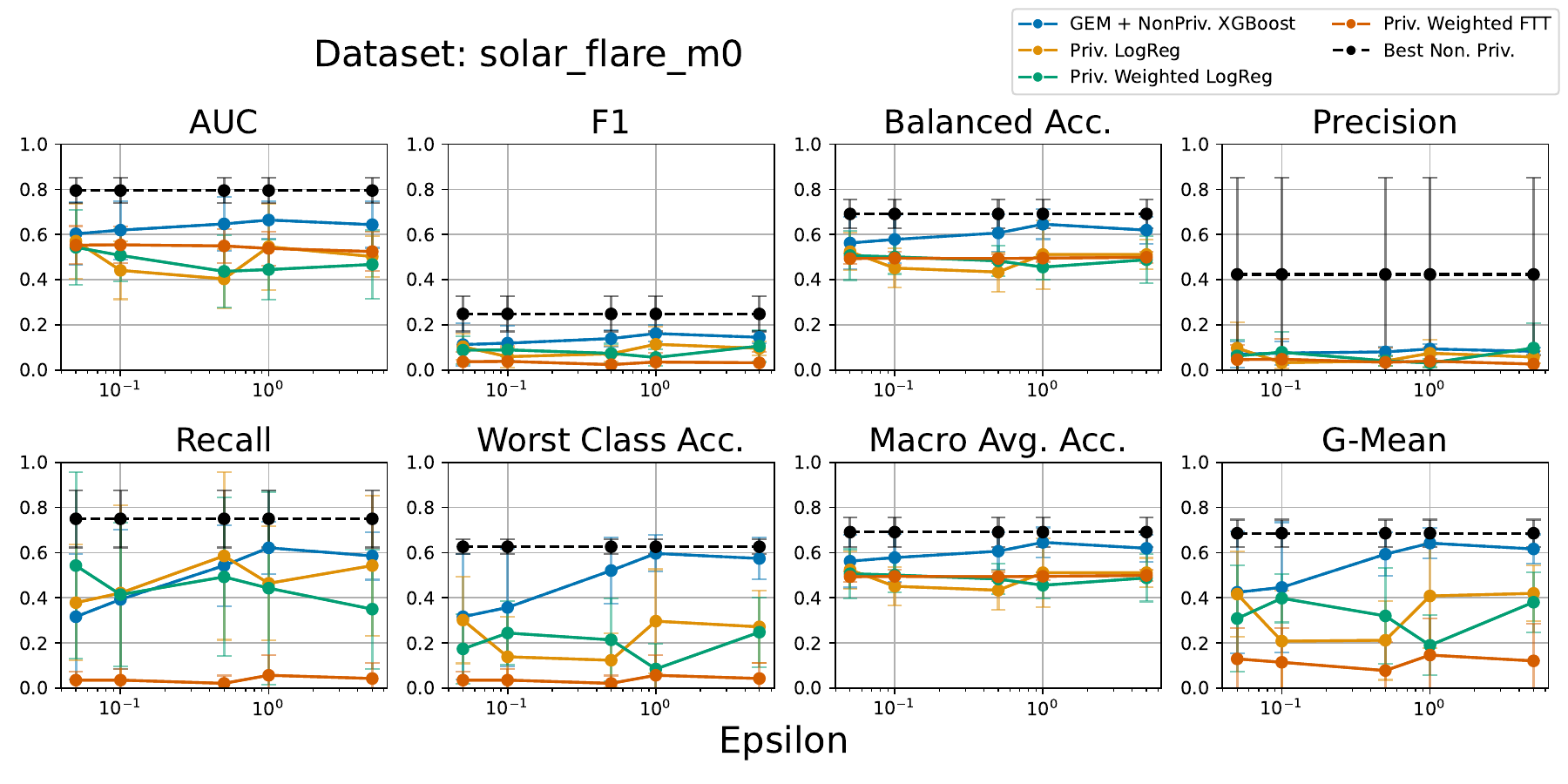}
    \caption{Privacy-preserving predictors across $\epsilon$ settings for $solar\_flare\_m0$ dataset.}
    \label{fig:solar_flare_m0_private}
\end{figure}

\begin{figure}[h!]
    \centering
    \includegraphics[width=0.9\linewidth]{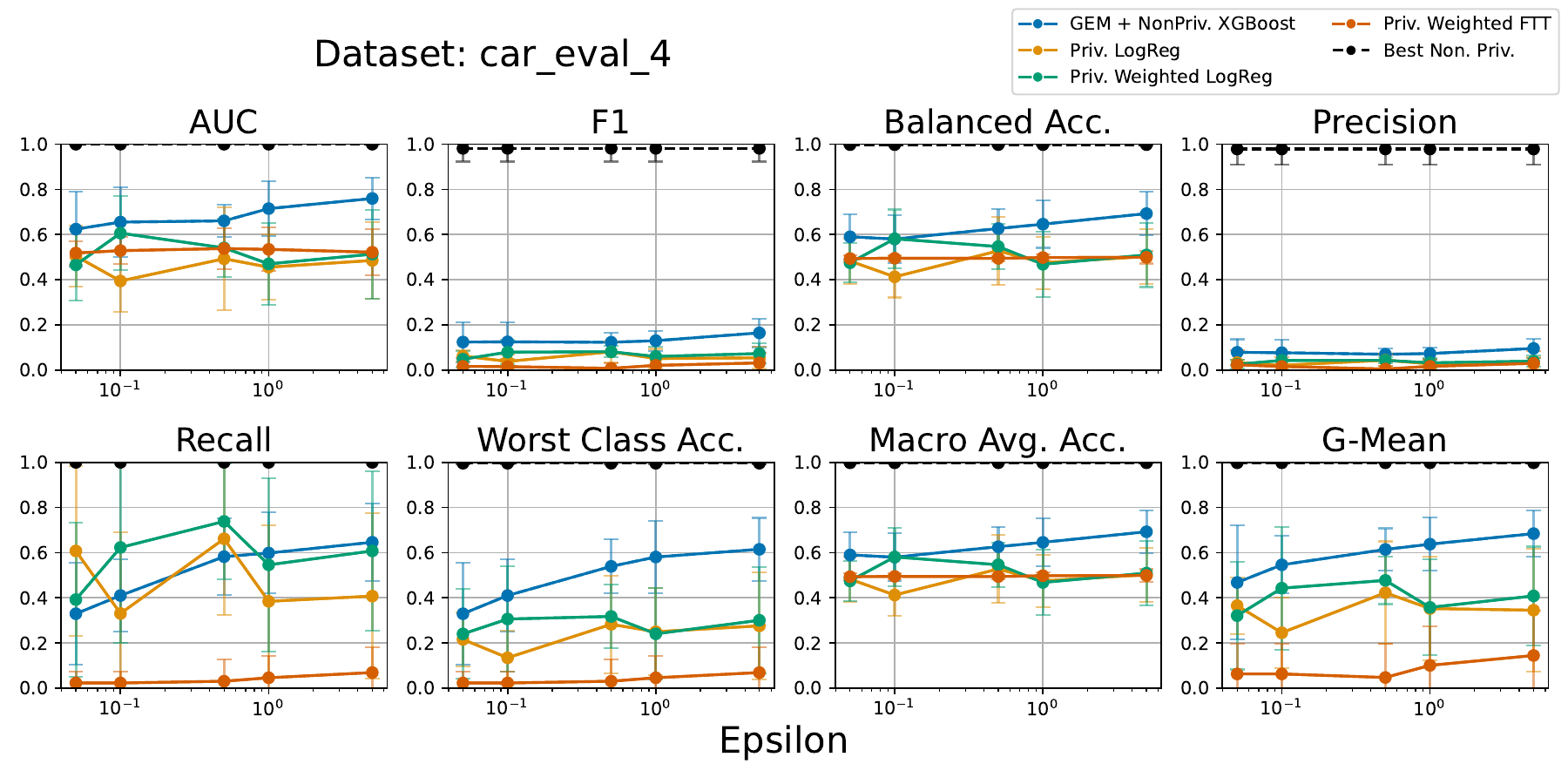}
    \caption{Privacy-preserving predictors across $\epsilon$ settings for $car\_eval\_4$ dataset.}
    \label{fig:car_eval_4_private}
\end{figure}

\begin{figure}[h!]
    \centering
    \includegraphics[width=0.9\linewidth]{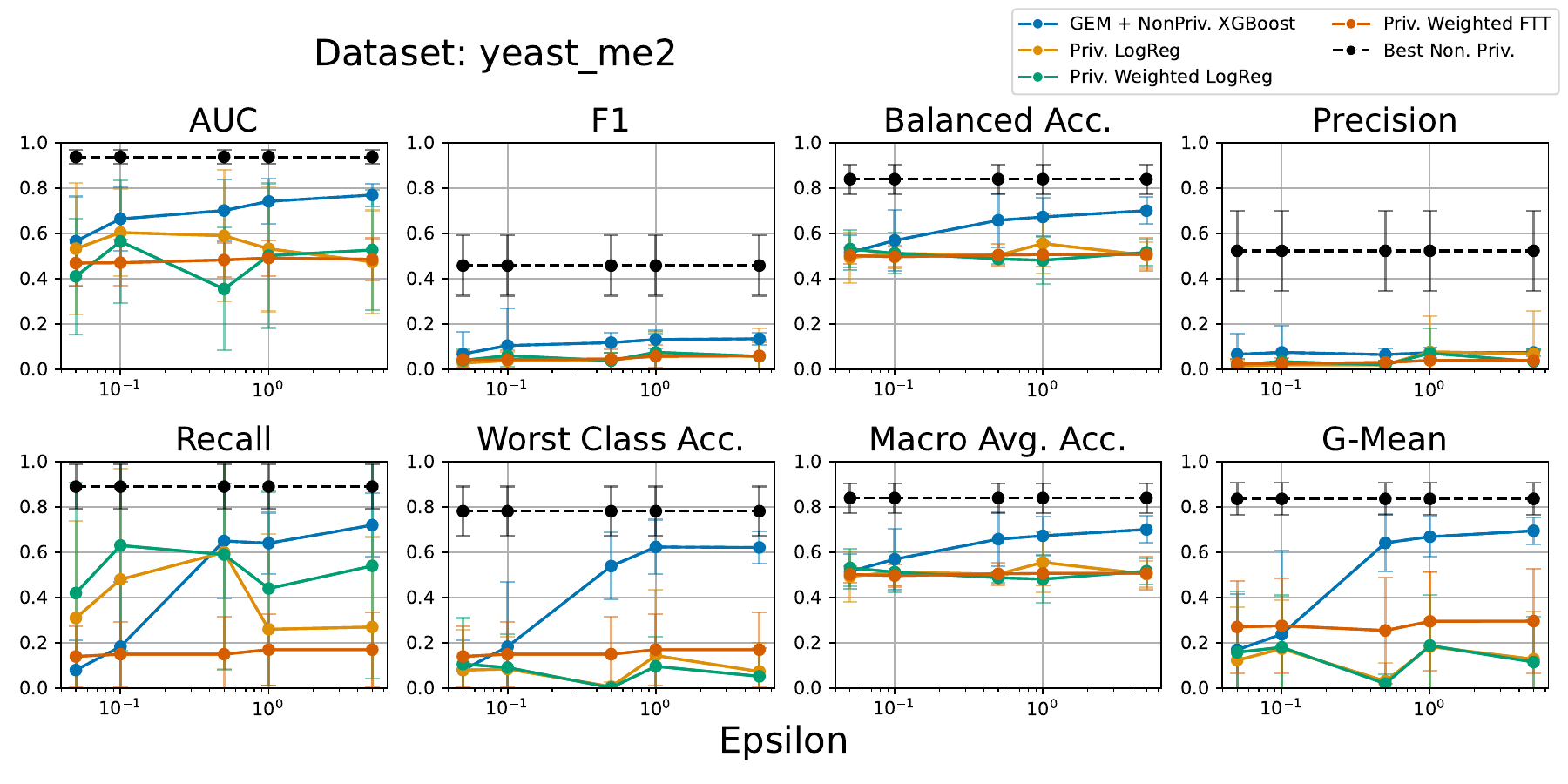}
    \caption{Privacy-preserving predictors across $\epsilon$ settings for $yeast\_me2$ dataset.}
    \label{fig:yeast_me2_private}
\end{figure}

\begin{figure}[h!]
    \centering
    \includegraphics[width=0.9\linewidth]{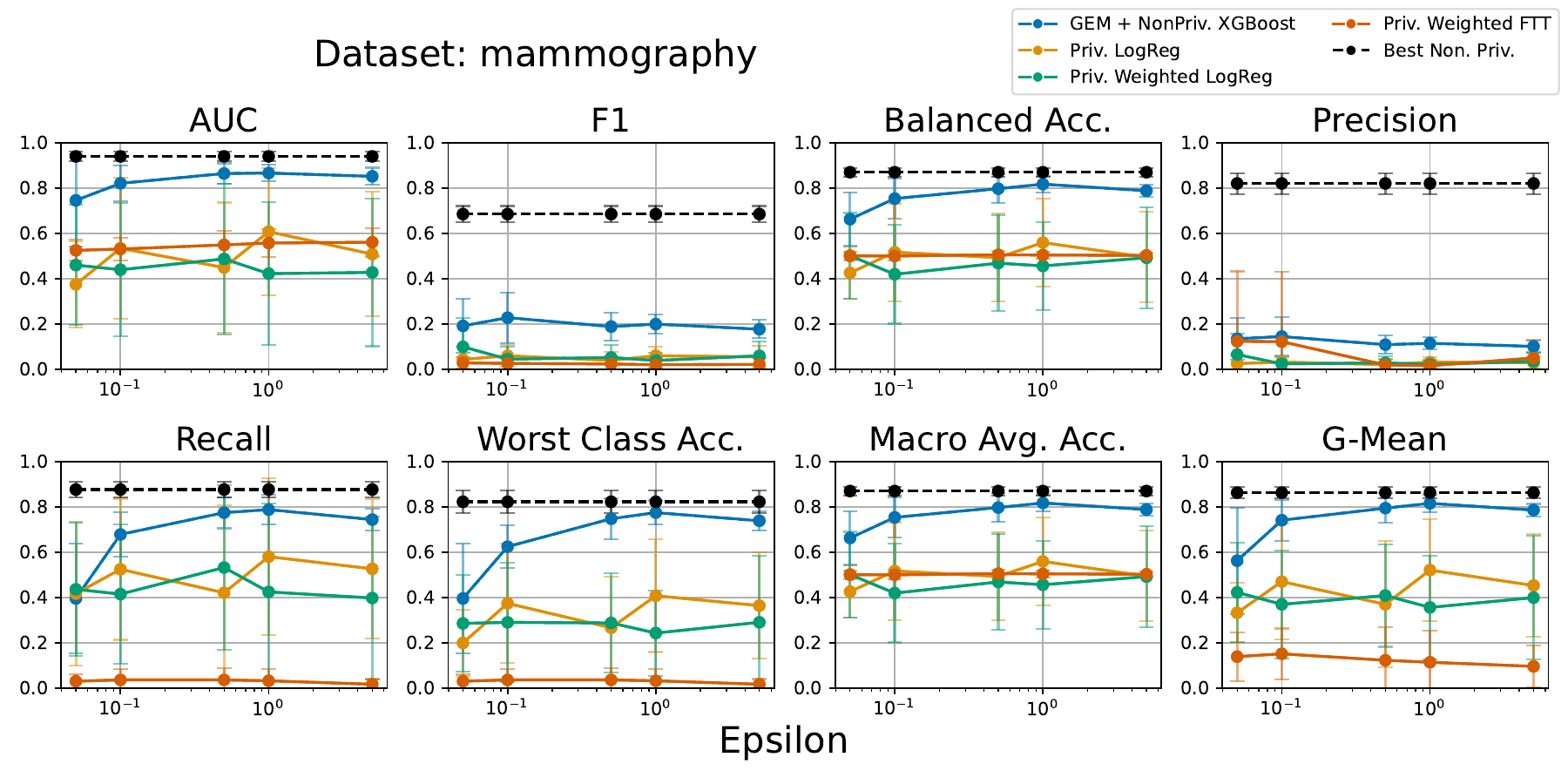}
    \caption{Privacy-preserving predictors across $\epsilon$ settings for $mammography$ dataset.}
    \label{fig:mammography_private}
\end{figure}

\begin{figure}[h!]
    \centering
    \includegraphics[width=0.9\linewidth]{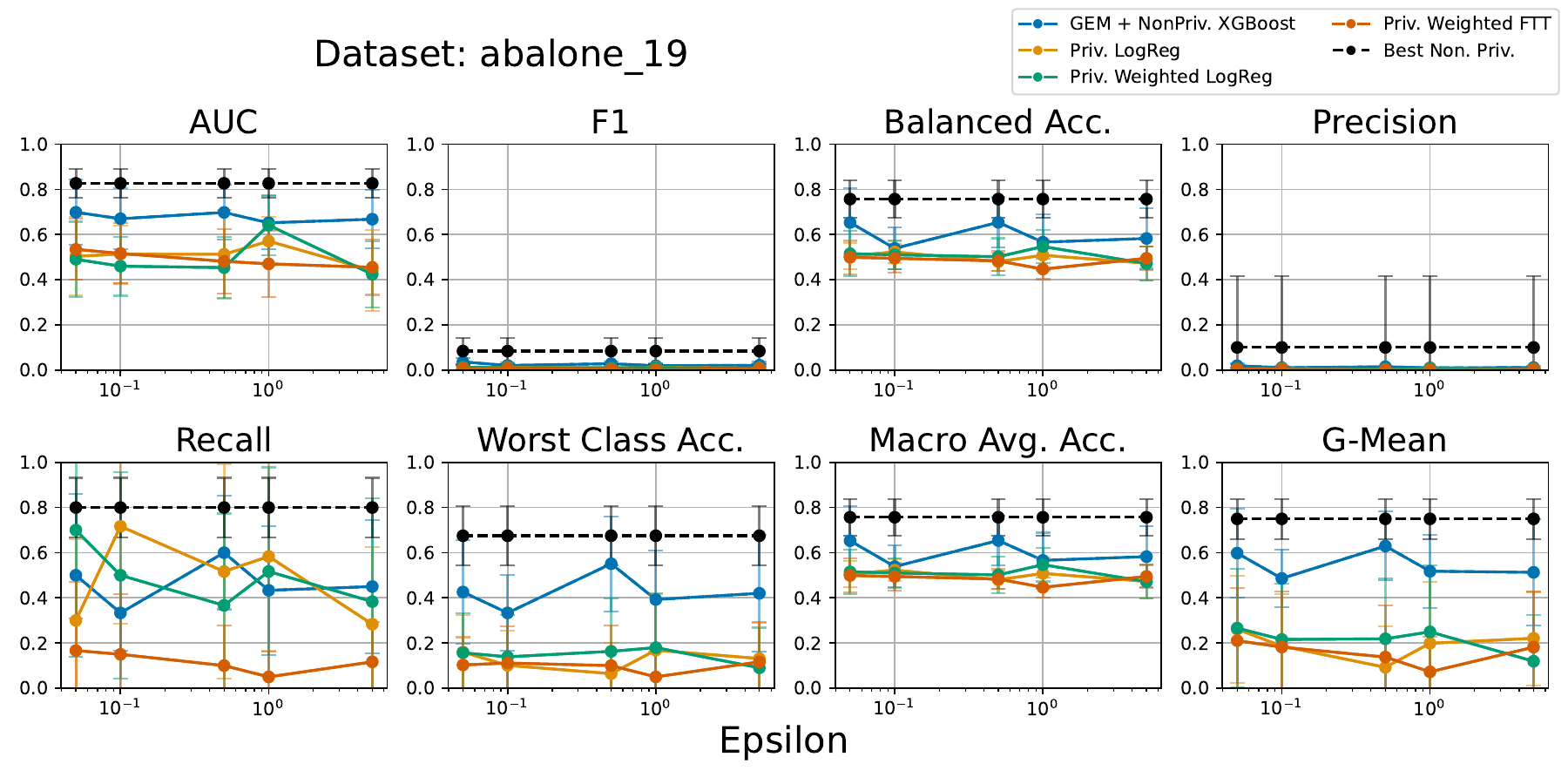}
    \caption{Privacy-preserving predictors across $\epsilon$ settings for $abalone\_19$ dataset.}
    \label{fig:abalone_private_19}
\end{figure}

\FloatBarrier
\newpage
\subsection{Complete non-private results}\label{sec:all_non_priv}
\begin{table}[H]
\centering
\caption{Ecoli Dataset}
\label{tab:bench_results_ecoli}
\resizebox{\columnwidth}{!}{
\begin{tabular}{lcccccccc}
\toprule
\multicolumn{1}{c}{\multirow{2}{*}{\textbf{Metrics}}} & \multicolumn{4}{c}{\textbf{Standard} $\uparrow$} & \multicolumn{4}{c}{\textbf{Imbalanced $\uparrow$}} \\
\cmidrule(lr){2-5} \cmidrule(lr){6-9}
\textbf{Approach} & \textbf{AUC} & \textbf{F1} & \textbf{Bal-ACC} & \textbf{Prec./Recall} & \textbf{Worst-ACC} & \textbf{Avg-ACC} & \textbf{G-Mean} & \textbf{MCC} \\
\midrule
\midrule
\textbf{\textbf{Non-Private} $\downarrow$} &  &  &  &  &  & &  & \\
\midrule
\midrule
Identity + NonPriv. LogReg & 0.94 $\pm$ 0.04 & 0.09 $\pm$ 0.12 & 0.53 $\pm$ 0.04 & 0.3 $\pm$ 0.42 / 0.06 $\pm$ 0.07 & 0.06 $\pm$ 0.07 & 0.53 $\pm$ 0.04 & 0.15 $\pm$ 0.19 & 0.11 $\pm$ 0.16 \\
Identity + NonPriv. Weighted LogReg & 0.93 $\pm$ 0.04 & 0.55 $\pm$ 0.11 & 0.77 $\pm$ 0.08 & 0.52 $\pm$ 0.12 / 0.61 $\pm$ 0.17 & 0.61 $\pm$ 0.17 & 0.77 $\pm$ 0.08 & 0.75 $\pm$ 0.1 & 0.51 $\pm$ 0.13 \\
Identity + NonPriv. XGBoost & 0.91 $\pm$ 0.07 & 0.63 $\pm$ 0.15 & 0.77 $\pm$ 0.09 & 0.77 $\pm$ 0.19 / 0.56 $\pm$ 0.17 & 0.56 $\pm$ 0.17 & 0.77 $\pm$ 0.09 & 0.73 $\pm$ 0.12 & 0.61 $\pm$ 0.16 \\
Identity + NonPriv. Weighted XGBoost & 0.91 $\pm$ 0.08 & 0.65 $\pm$ 0.16 & 0.8 $\pm$ 0.1 & 0.72 $\pm$ 0.18 / 0.63 $\pm$ 0.2 & 0.63 $\pm$ 0.2 & 0.8 $\pm$ 0.1 & 0.77 $\pm$ 0.13 & 0.63 $\pm$ 0.17 \\
SMOTE + NonPriv. LogReg & 0.94 $\pm$ 0.04 & 0.61 $\pm$ 0.09 & 0.88 $\pm$ 0.06 & 0.47 $\pm$ 0.1 / 0.89 $\pm$ 0.11 & 0.83 $\pm$ 0.07 & 0.88 $\pm$ 0.06 & 0.88 $\pm$ 0.06 & 0.59 $\pm$ 0.1 \\
SMOTE + NonPriv. Weighted LogReg & 0.94 $\pm$ 0.04 & 0.51 $\pm$ 0.07 & 0.86 $\pm$ 0.05 & 0.36 $\pm$ 0.06 / 0.91 $\pm$ 0.1 & 0.79 $\pm$ 0.06 & 0.86 $\pm$ 0.05 & 0.85 $\pm$ 0.05 & 0.5 $\pm$ 0.08 \\
SMOTE + NonPriv. Weighted XGB & 0.94 $\pm$ 0.04 & 0.68 $\pm$ 0.1 & 0.85 $\pm$ 0.07 & 0.65 $\pm$ 0.12 / 0.74 $\pm$ 0.15 & 0.74 $\pm$ 0.14 & 0.85 $\pm$ 0.07 & 0.84 $\pm$ 0.08 & 0.65 $\pm$ 0.11 \\
SMOTE + NonPriv. XGBoost & 0.94 $\pm$ 0.04 & 0.7 $\pm$ 0.11 & 0.85 $\pm$ 0.08 & 0.7 $\pm$ 0.16 / 0.74 $\pm$ 0.18 & 0.74 $\pm$ 0.17 & 0.85 $\pm$ 0.08 & 0.84 $\pm$ 0.1 & 0.68 $\pm$ 0.12 \\
Identity + NonPriv. Weighted FTTransformer & 0.51 $\pm$ 0.12 & 0.12 $\pm$ 0.09 & 0.51 $\pm$ 0.04 & 0.12 $\pm$ 0.10 / 0.13 $\pm$ 0.11 & 0.13 $\pm$ 0.11 & 0.51 $\pm$ 0.04 & 0.21 $\pm$ 0.21 & 0.09 $\pm$ 0.09 \\
\midrule
\bottomrule
\end{tabular}
}
\end{table}

\begin{table}[H]
\centering
\caption{Abolone Dataset}
\label{tab:bench_results_abalone}
\resizebox{\columnwidth}{!}{
\begin{tabular}{lcccccccc}
\toprule
\multicolumn{1}{c}{\multirow{2}{*}{\textbf{Metrics}}} & \multicolumn{4}{c}{\textbf{Standard} $\uparrow$} & \multicolumn{4}{c}{\textbf{Imbalanced $\uparrow$}} \\
\cmidrule(lr){2-5} \cmidrule(lr){6-9}
\textbf{Approach} & \textbf{AUC} & \textbf{F1} & \textbf{Bal-ACC} & \textbf{Prec./Recall} & \textbf{Worst-ACC} & \textbf{Avg-ACC} & \textbf{G-Mean} & \textbf{MCC} \\
\midrule
\midrule
\textbf{\textbf{Non-Private} $\downarrow$} &  &  &  &  &  & &  & \\
\midrule
\midrule
Identity + NonPriv. LogReg & 0.81 $\pm$ 0.02 & 0.0 $\pm$ 0.0 & 0.5 $\pm$ 0.0 & 0.0 $\pm$ 0.0 / 0.0 $\pm$ 0.0 & 0.0 $\pm$ 0.0 & 0.5 $\pm$ 0.0 & 0.0 $\pm$ 0.0 & 0.0 $\pm$ 0.0 \\
Identity + NonPriv. Weighted LogReg & 0.81 $\pm$ 0.02 & 0.38 $\pm$ 0.03 & 0.73 $\pm$ 0.03 & 0.27 $\pm$ 0.03 / 0.63 $\pm$ 0.04 & 0.63 $\pm$ 0.04 & 0.73 $\pm$ 0.03 & 0.72 $\pm$ 0.03 & 0.32 $\pm$ 0.04 \\
Identity + NonPriv. XGBoost & 0.84 $\pm$ 0.02 & 0.19 $\pm$ 0.05 & 0.55 $\pm$ 0.02 & 0.29 $\pm$ 0.08 / 0.14 $\pm$ 0.04 & 0.14 $\pm$ 0.04 & 0.55 $\pm$ 0.02 & 0.36 $\pm$ 0.05 & 0.15 $\pm$ 0.06 \\
Identity + NonPriv. Weighted XGBoost & 0.84 $\pm$ 0.02 & 0.35 $\pm$ 0.04 & 0.66 $\pm$ 0.03 & 0.3 $\pm$ 0.04 / 0.43 $\pm$ 0.05 & 0.43 $\pm$ 0.05 & 0.66 $\pm$ 0.03 & 0.62 $\pm$ 0.04 & 0.28 $\pm$ 0.05 \\
SMOTE + NonPriv. LogReg & 0.83 $\pm$ 0.02 & 0.36 $\pm$ 0.01 & 0.78 $\pm$ 0.02 & 0.22 $\pm$ 0.01 / 0.87 $\pm$ 0.03 & 0.69 $\pm$ 0.02 & 0.78 $\pm$ 0.02 & 0.77 $\pm$ 0.01 & 0.34 $\pm$ 0.02 \\
SMOTE + NonPriv. Weighted LogReg & 0.82 $\pm$ 0.02 & 0.31 $\pm$ 0.01 & 0.76 $\pm$ 0.02 & 0.19 $\pm$ 0.01 / 0.94 $\pm$ 0.03 & 0.58 $\pm$ 0.02 & 0.76 $\pm$ 0.02 & 0.74 $\pm$ 0.01 & 0.31 $\pm$ 0.02 \\
SMOTE + NonPriv. Weighted XGB & 0.84 $\pm$ 0.02 & 0.37 $\pm$ 0.04 & 0.69 $\pm$ 0.03 & 0.3 $\pm$ 0.04 / 0.49 $\pm$ 0.07 & 0.49 $\pm$ 0.07 & 0.69 $\pm$ 0.03 & 0.66 $\pm$ 0.04 & 0.3 $\pm$ 0.05 \\
SMOTE + NonPriv. XGBoost & 0.84 $\pm$ 0.02 & 0.32 $\pm$ 0.04 & 0.64 $\pm$ 0.03 & 0.29 $\pm$ 0.03 / 0.36 $\pm$ 0.06 & 0.36 $\pm$ 0.06 & 0.64 $\pm$ 0.03 & 0.57 $\pm$ 0.04 & 0.25 $\pm$ 0.04 \\
Identity + NonPriv. Weighted FTTransformer & 0.70 $\pm$ 0.03 & 0.06 $\pm$ 0.09 & 0.52 $\pm$ 0.04 & 0.19 $\pm$ 0.20 / 0.07 $\pm$ 0.14 & 0.07 $\pm$ 0.14 & 0.52 $\pm$ 0.04 & 0.19 $\pm$ 0.19 & 0.08 $\pm$ 0.08 \\
\midrule
\bottomrule
\end{tabular}
}
\end{table}

\begin{table}[H]
\centering
\caption{Car\_eval\_34 Dataset}
\label{tab:bench_results_car_eval}
\resizebox{\columnwidth}{!}{
\begin{tabular}{lcccccccc}
\toprule
\multicolumn{1}{c}{\multirow{2}{*}{\textbf{Metrics}}} & \multicolumn{4}{c}{\textbf{Standard} $\uparrow$} & \multicolumn{4}{c}{\textbf{Imbalanced $\uparrow$}} \\
\cmidrule(lr){2-5} \cmidrule(lr){6-9}
\textbf{Approach} & \textbf{AUC} & \textbf{F1} & \textbf{Bal-ACC} & \textbf{Prec./Recall} & \textbf{Worst-ACC} & \textbf{Avg-ACC} & \textbf{G-Mean} & \textbf{MCC} \\
\midrule
\midrule
\textbf{\textbf{Non-Private} $\downarrow$} &  &  &  &  &  & &  & \\
\midrule
\midrule
Identity + NonPriv. LogReg & 1.0 $\pm$ 0.0 & 0.86 $\pm$ 0.05 & 0.89 $\pm$ 0.04 & 0.95 $\pm$ 0.03 / 0.79 $\pm$ 0.08 & 0.79 $\pm$ 0.08 & 0.89 $\pm$ 0.04 & 0.89 $\pm$ 0.05 & 0.85 $\pm$ 0.05 \\
Identity + NonPriv. Weighted LogReg & 1.0 $\pm$ 0.0 & 0.85 $\pm$ 0.03 & 0.98 $\pm$ 0.0 & 0.74 $\pm$ 0.05 / 1.0 $\pm$ 0.0 & 0.97 $\pm$ 0.01 & 0.98 $\pm$ 0.0 & 0.98 $\pm$ 0.0 & 0.84 $\pm$ 0.03 \\
Identity + NonPriv. XGBoost & 1.0 $\pm$ 0.0 & 0.96 $\pm$ 0.02 & 0.98 $\pm$ 0.02 & 0.94 $\pm$ 0.03 / 0.97 $\pm$ 0.03 & 0.97 $\pm$ 0.03 & 0.98 $\pm$ 0.02 & 0.98 $\pm$ 0.02 & 0.95 $\pm$ 0.02 \\
Identity + NonPriv. Weighted XGBoost & 1.0 $\pm$ 0.0 & 0.94 $\pm$ 0.03 & 0.99 $\pm$ 0.0 & 0.89 $\pm$ 0.05 / 1.0 $\pm$ 0.0 & 0.99 $\pm$ 0.01 & 0.99 $\pm$ 0.0 & 0.99 $\pm$ 0.0 & 0.94 $\pm$ 0.03 \\
SMOTE + NonPriv. LogReg & 1.0 $\pm$ 0.0 & 0.85 $\pm$ 0.03 & 0.98 $\pm$ 0.0 & 0.74 $\pm$ 0.04 / 1.0 $\pm$ 0.0 & 0.97 $\pm$ 0.01 & 0.98 $\pm$ 0.0 & 0.98 $\pm$ 0.0 & 0.84 $\pm$ 0.03 \\
SMOTE + NonPriv. Weighted LogReg & 1.0 $\pm$ 0.0 & 0.73 $\pm$ 0.03 & 0.97 $\pm$ 0.0 & 0.58 $\pm$ 0.03 / 1.0 $\pm$ 0.0 & 0.94 $\pm$ 0.01 & 0.97 $\pm$ 0.0 & 0.97 $\pm$ 0.0 & 0.74 $\pm$ 0.02 \\
SMOTE + NonPriv. Weighted XGB & 1.0 $\pm$ 0.0 & 0.95 $\pm$ 0.02 & 0.99 $\pm$ 0.01 & 0.92 $\pm$ 0.05 / 0.99 $\pm$ 0.02 & 0.98 $\pm$ 0.01 & 0.99 $\pm$ 0.01 & 0.99 $\pm$ 0.01 & 0.95 $\pm$ 0.02 \\
SMOTE + NonPriv. XGBoost & 1.0 $\pm$ 0.0 & 0.96 $\pm$ 0.02 & 0.98 $\pm$ 0.01 & 0.94 $\pm$ 0.04 / 0.97 $\pm$ 0.03 & 0.97 $\pm$ 0.03 & 0.98 $\pm$ 0.01 & 0.98 $\pm$ 0.01 & 0.95 $\pm$ 0.02 \\
Identity + NonPriv. Weighted FTTransformer & 1.00 $\pm$ 0.00 & 0.92 $\pm$ 0.04 & 0.96 $\pm$ 0.03 & 0.91 $\pm$ 0.04 / 0.94 $\pm$ 0.06 & 0.93 $\pm$ 0.06 & 0.96 $\pm$ 0.03 & 0.03 $\pm$ 0.03 & 0.04 $\pm$ 0.04 \\
\midrule
\bottomrule
\end{tabular}
}
\end{table}

\begin{table}[H]
\centering
\caption{Solar\_flare\_m0 Dataset}
\label{tab:bench_results_solar_flare}
\resizebox{\columnwidth}{!}{
\begin{tabular}{lcccccccc}
\toprule
\multicolumn{1}{c}{\multirow{2}{*}{\textbf{Metrics}}} & \multicolumn{4}{c}{\textbf{Standard} $\uparrow$} & \multicolumn{4}{c}{\textbf{Imbalanced $\uparrow$}} \\
\cmidrule(lr){2-5} \cmidrule(lr){6-9}
\textbf{Approach} & \textbf{AUC} & \textbf{F1} & \textbf{Bal-ACC} & \textbf{Prec./Recall} & \textbf{Worst-ACC} & \textbf{Avg-ACC} & \textbf{G-Mean} & \textbf{MCC} \\
\midrule
\midrule
\textbf{\textbf{Non-Private} $\downarrow$} &  &  &  &  &  & &  & \\
\midrule
\midrule
Identity + NonPriv. LogReg & 0.79 $\pm$ 0.06 & 0.03 $\pm$ 0.05 & 0.51 $\pm$ 0.02 & 0.15 $\pm$ 0.34 / 0.01 $\pm$ 0.03 & 0.01 $\pm$ 0.03 & 0.51 $\pm$ 0.02 & 0.05 $\pm$ 0.11 & 0.04 $\pm$ 0.1 \\
Identity + NonPriv. Weighted LogReg & 0.79 $\pm$ 0.06 & 0.25 $\pm$ 0.08 & 0.67 $\pm$ 0.08 & 0.17 $\pm$ 0.05 / 0.44 $\pm$ 0.16 & 0.44 $\pm$ 0.16 & 0.67 $\pm$ 0.08 & 0.62 $\pm$ 0.11 & 0.22 $\pm$ 0.1 \\
Identity + NonPriv. XGBoost & 0.73 $\pm$ 0.04 & 0.09 $\pm$ 0.09 & 0.53 $\pm$ 0.03 & 0.17 $\pm$ 0.16 / 0.06 $\pm$ 0.06 & 0.06 $\pm$ 0.06 & 0.53 $\pm$ 0.03 & 0.19 $\pm$ 0.17 & 0.08 $\pm$ 0.09 \\
Identity + NonPriv. Weighted XGBoost & 0.74 $\pm$ 0.04 & 0.2 $\pm$ 0.05 & 0.61 $\pm$ 0.04 & 0.15 $\pm$ 0.04 / 0.31 $\pm$ 0.09 & 0.31 $\pm$ 0.09 & 0.61 $\pm$ 0.04 & 0.53 $\pm$ 0.07 & 0.15 $\pm$ 0.06 \\
SMOTE + NonPriv. LogReg & 0.76 $\pm$ 0.07 & 0.19 $\pm$ 0.04 & 0.67 $\pm$ 0.06 & 0.11 $\pm$ 0.02 / 0.58 $\pm$ 0.12 & 0.57 $\pm$ 0.1 & 0.67 $\pm$ 0.06 & 0.66 $\pm$ 0.07 & 0.17 $\pm$ 0.06 \\
SMOTE + NonPriv. Weighted LogReg & 0.75 $\pm$ 0.07 & 0.17 $\pm$ 0.03 & 0.69 $\pm$ 0.06 & 0.1 $\pm$ 0.02 / 0.75 $\pm$ 0.13 & 0.63 $\pm$ 0.03 & 0.69 $\pm$ 0.06 & 0.69 $\pm$ 0.06 & 0.17 $\pm$ 0.06 \\
SMOTE + NonPriv. Weighted XGB & 0.7 $\pm$ 0.04 & 0.09 $\pm$ 0.06 & 0.52 $\pm$ 0.04 & 0.09 $\pm$ 0.06 / 0.1 $\pm$ 0.08 & 0.1 $\pm$ 0.08 & 0.52 $\pm$ 0.04 & 0.27 $\pm$ 0.16 & 0.04 $\pm$ 0.07 \\
SMOTE + NonPriv. XGBoost & 0.7 $\pm$ 0.05 & 0.07 $\pm$ 0.06 & 0.52 $\pm$ 0.02 & 0.08 $\pm$ 0.07 / 0.06 $\pm$ 0.05 & 0.06 $\pm$ 0.05 & 0.52 $\pm$ 0.02 & 0.2 $\pm$ 0.15 & 0.03 $\pm$ 0.05 \\
Identity + NonPriv. Weighted FTTransformer & 0.76 $\pm$ 0.06 & 0.09 $\pm$ 0.12 & 0.53 $\pm$ 0.04 & 0.20 $\pm$ 0.27 / 0.06 $\pm$ 0.08 & 0.06 $\pm$ 0.08 & 0.53 $\pm$ 0.04 & 0.20 $\pm$ 0.20 & 0.14 $\pm$ 0.14 \\
\midrule
\bottomrule
\end{tabular}
}
\end{table}

\begin{table}[H]
\centering
\caption{Car\_eval\_4 Dataset}
\label{tab:bench_results_car_eval_4}
\resizebox{\columnwidth}{!}{
\begin{tabular}{lcccccccc}
\toprule
\multicolumn{1}{c}{\multirow{2}{*}{\textbf{Metrics}}} & \multicolumn{4}{c}{\textbf{Standard} $\uparrow$} & \multicolumn{4}{c}{\textbf{Imbalanced $\uparrow$}} \\
\cmidrule(lr){2-5} \cmidrule(lr){6-9}
\textbf{Approach} & \textbf{AUC} & \textbf{F1} & \textbf{Bal-ACC} & \textbf{Prec./Recall} & \textbf{Worst-ACC} & \textbf{Avg-ACC} & \textbf{G-Mean} & \textbf{MCC} \\
\midrule
\midrule
\textbf{\textbf{Non-Private} $\downarrow$} &  &  &  &  &  & &  & \\
\midrule
\midrule
Identity + NonPriv. LogReg & 1.0 $\pm$ 0.0 & 0.75 $\pm$ 0.1 & 0.82 $\pm$ 0.07 & 0.93 $\pm$ 0.08 / 0.64 $\pm$ 0.13 & 0.64 $\pm$ 0.13 & 0.82 $\pm$ 0.07 & 0.79 $\pm$ 0.08 & 0.76 $\pm$ 0.1 \\
Identity + NonPriv. Weighted LogReg & 1.0 $\pm$ 0.0 & 0.75 $\pm$ 0.06 & 0.99 $\pm$ 0.0 & 0.6 $\pm$ 0.08 / 1.0 $\pm$ 0.0 & 0.97 $\pm$ 0.01 & 0.99 $\pm$ 0.0 & 0.99 $\pm$ 0.0 & 0.76 $\pm$ 0.05 \\
Identity + NonPriv. XGBoost & 1.0 $\pm$ 0.0 & 0.98 $\pm$ 0.06 & 0.99 $\pm$ 0.03 & 0.98 $\pm$ 0.07 / 0.98 $\pm$ 0.05 & 0.98 $\pm$ 0.05 & 0.99 $\pm$ 0.03 & 0.99 $\pm$ 0.03 & 0.98 $\pm$ 0.06 \\
Identity + NonPriv. Weighted XGBoost & 1.0 $\pm$ 0.0 & 0.85 $\pm$ 0.06 & 0.99 $\pm$ 0.0 & 0.74 $\pm$ 0.09 / 1.0 $\pm$ 0.0 & 0.99 $\pm$ 0.01 & 0.99 $\pm$ 0.0 & 0.99 $\pm$ 0.0 & 0.85 $\pm$ 0.06 \\
SMOTE + NonPriv. LogReg & 1.0 $\pm$ 0.0 & 0.81 $\pm$ 0.06 & 0.99 $\pm$ 0.0 & 0.68 $\pm$ 0.08 / 1.0 $\pm$ 0.0 & 0.98 $\pm$ 0.01 & 0.99 $\pm$ 0.0 & 0.99 $\pm$ 0.0 & 0.82 $\pm$ 0.05 \\
SMOTE + NonPriv. Weighted LogReg & 1.0 $\pm$ 0.0 & 0.77 $\pm$ 0.06 & 0.99 $\pm$ 0.0 & 0.63 $\pm$ 0.08 / 1.0 $\pm$ 0.0 & 0.98 $\pm$ 0.01 & 0.99 $\pm$ 0.0 & 0.99 $\pm$ 0.0 & 0.78 $\pm$ 0.05 \\
SMOTE + NonPriv. Weighted XGB & 1.0 $\pm$ 0.0 & 0.96 $\pm$ 0.06 & 1.0 $\pm$ 0.0 & 0.92 $\pm$ 0.1 / 1.0 $\pm$ 0.0 & 1.0 $\pm$ 0.01 & 1.0 $\pm$ 0.0 & 1.0 $\pm$ 0.0 & 0.96 $\pm$ 0.06 \\
SMOTE + NonPriv. XGBoost & 1.0 $\pm$ 0.0 & 0.97 $\pm$ 0.05 & 0.99 $\pm$ 0.01 & 0.95 $\pm$ 0.09 / 0.99 $\pm$ 0.02 & 0.99 $\pm$ 0.02 & 0.99 $\pm$ 0.01 & 0.99 $\pm$ 0.01 & 0.97 $\pm$ 0.05 \\
Identity + NonPriv. Weighted FTTransformer & 0.99 $\pm$ 0.01 & 0.82 $\pm$ 0.10 & 0.94 $\pm$ 0.07 & 0.78 $\pm$ 0.14 / 0.90 $\pm$ 0.14 & 0.89 $\pm$ 0.13 & 0.94 $\pm$ 0.07 & 0.07 $\pm$ 0.07 & 0.10 $\pm$ 0.10 \\
\midrule
\bottomrule
\end{tabular}
}
\end{table}

\begin{table}[H]
\centering
\caption{Yeast\_me2 Dataset}
\label{tab:bench_results_yeast}
\resizebox{\columnwidth}{!}{
\begin{tabular}{lcccccccc}
\toprule
\multicolumn{1}{c}{\multirow{2}{*}{\textbf{Metrics}}} & \multicolumn{4}{c}{\textbf{Standard} $\uparrow$} & \multicolumn{4}{c}{\textbf{Imbalanced $\uparrow$}} \\
\cmidrule(lr){2-5} \cmidrule(lr){6-9}
\textbf{Approach} & \textbf{AUC} & \textbf{F1} & \textbf{Bal-ACC} & \textbf{Prec./Recall} & \textbf{Worst-ACC} & \textbf{Avg-ACC} & \textbf{G-Mean} & \textbf{MCC} \\
\midrule
\midrule
\textbf{\textbf{Non-Private} $\downarrow$} &  &  &  &  &  & &  & \\
\midrule
\midrule
Identity + NonPriv. LogReg & 0.88 $\pm$ 0.06 & 0.0 $\pm$ 0.0 & 0.5 $\pm$ 0.0 & 0.0 $\pm$ 0.0 / 0.0 $\pm$ 0.0 & 0.0 $\pm$ 0.0 & 0.5 $\pm$ 0.0 & 0.0 $\pm$ 0.0 & 0.0 $\pm$ 0.0 \\
Identity + NonPriv. Weighted LogReg & 0.88 $\pm$ 0.06 & 0.27 $\pm$ 0.07 & 0.66 $\pm$ 0.04 & 0.22 $\pm$ 0.07 / 0.37 $\pm$ 0.08 & 0.37 $\pm$ 0.08 & 0.66 $\pm$ 0.04 & 0.59 $\pm$ 0.07 & 0.25 $\pm$ 0.07 \\
Identity + NonPriv. XGBoost & 0.93 $\pm$ 0.03 & 0.33 $\pm$ 0.13 & 0.62 $\pm$ 0.05 & 0.52 $\pm$ 0.18 / 0.25 $\pm$ 0.11 & 0.25 $\pm$ 0.11 & 0.62 $\pm$ 0.05 & 0.49 $\pm$ 0.11 & 0.34 $\pm$ 0.13 \\
Identity + NonPriv. Weighted XGBoost & 0.94 $\pm$ 0.03 & 0.46 $\pm$ 0.13 & 0.77 $\pm$ 0.08 & 0.38 $\pm$ 0.13 / 0.58 $\pm$ 0.15 & 0.58 $\pm$ 0.15 & 0.77 $\pm$ 0.08 & 0.74 $\pm$ 0.1 & 0.45 $\pm$ 0.14 \\
SMOTE + NonPriv. LogReg & 0.9 $\pm$ 0.04 & 0.29 $\pm$ 0.04 & 0.84 $\pm$ 0.07 & 0.18 $\pm$ 0.03 / 0.81 $\pm$ 0.14 & 0.78 $\pm$ 0.11 & 0.84 $\pm$ 0.07 & 0.84 $\pm$ 0.07 & 0.34 $\pm$ 0.06 \\
SMOTE + NonPriv. Weighted LogReg & 0.9 $\pm$ 0.04 & 0.2 $\pm$ 0.01 & 0.82 $\pm$ 0.04 & 0.11 $\pm$ 0.01 / 0.89 $\pm$ 0.1 & 0.74 $\pm$ 0.03 & 0.82 $\pm$ 0.04 & 0.82 $\pm$ 0.04 & 0.26 $\pm$ 0.03 \\
SMOTE + NonPriv. Weighted XGB & 0.92 $\pm$ 0.04 & 0.4 $\pm$ 0.09 & 0.71 $\pm$ 0.04 & 0.39 $\pm$ 0.12 / 0.45 $\pm$ 0.08 & 0.45 $\pm$ 0.08 & 0.71 $\pm$ 0.04 & 0.66 $\pm$ 0.07 & 0.39 $\pm$ 0.09 \\
SMOTE + NonPriv. XGBoost & 0.92 $\pm$ 0.04 & 0.42 $\pm$ 0.12 & 0.71 $\pm$ 0.06 & 0.41 $\pm$ 0.15 / 0.45 $\pm$ 0.12 & 0.45 $\pm$ 0.12 & 0.71 $\pm$ 0.06 & 0.66 $\pm$ 0.1 & 0.41 $\pm$ 0.12 \\
Identity + NonPriv. Weighted FTTransformer & 0.51 $\pm$ 0.10 & 0.03 $\pm$ 0.04 & 0.50 $\pm$ 0.05 & 0.02 $\pm$ 0.03 / 0.09 $\pm$ 0.16 & 0.09 $\pm$ 0.16 & 0.50 $\pm$ 0.05 & 0.22 $\pm$ 0.22 & 0.05 $\pm$ 0.05 \\
\midrule
\bottomrule
\end{tabular}
}
\end{table}

\begin{table}[H]
\centering
\caption{Mammography Dataset}
\label{tab:bench_results_mammo}
\resizebox{\columnwidth}{!}{
\begin{tabular}{lcccccccc}
\toprule
\multicolumn{1}{c}{\multirow{2}{*}{\textbf{Metrics}}} & \multicolumn{4}{c}{\textbf{Standard} $\uparrow$} & \multicolumn{4}{c}{\textbf{Imbalanced $\uparrow$}} \\
\cmidrule(lr){2-5} \cmidrule(lr){6-9}
\textbf{Approach} & \textbf{AUC} & \textbf{F1} & \textbf{Bal-ACC} & \textbf{Prec./Recall} & \textbf{Worst-ACC} & \textbf{Avg-ACC} & \textbf{G-Mean} & \textbf{MCC} \\
\midrule
\midrule
\textbf{\textbf{Non-Private} $\downarrow$} &  &  &  &  &  & &  & \\
\midrule
\midrule
Identity + NonPriv. LogReg & 0.9 $\pm$ 0.02 & 0.53 $\pm$ 0.05 & 0.7 $\pm$ 0.03 & 0.8 $\pm$ 0.07 / 0.4 $\pm$ 0.06 & 0.4 $\pm$ 0.06 & 0.7 $\pm$ 0.03 & 0.63 $\pm$ 0.05 & 0.55 $\pm$ 0.05 \\
Identity + NonPriv. Weighted LogReg & 0.91 $\pm$ 0.02 & 0.4 $\pm$ 0.02 & 0.84 $\pm$ 0.02 & 0.27 $\pm$ 0.02 / 0.73 $\pm$ 0.05 & 0.73 $\pm$ 0.05 & 0.84 $\pm$ 0.02 & 0.84 $\pm$ 0.03 & 0.43 $\pm$ 0.03 \\
Identity + NonPriv. XGBoost & 0.94 $\pm$ 0.02 & 0.69 $\pm$ 0.04 & 0.79 $\pm$ 0.03 & 0.82 $\pm$ 0.05 / 0.59 $\pm$ 0.06 & 0.59 $\pm$ 0.06 & 0.79 $\pm$ 0.03 & 0.77 $\pm$ 0.04 & 0.69 $\pm$ 0.03 \\
Identity + NonPriv. Weighted XGBoost & 0.94 $\pm$ 0.02 & 0.65 $\pm$ 0.04 & 0.87 $\pm$ 0.03 & 0.57 $\pm$ 0.04 / 0.75 $\pm$ 0.06 & 0.75 $\pm$ 0.06 & 0.87 $\pm$ 0.03 & 0.86 $\pm$ 0.04 & 0.65 $\pm$ 0.04 \\
SMOTE + NonPriv. LogReg & 0.91 $\pm$ 0.02 & 0.29 $\pm$ 0.01 & 0.86 $\pm$ 0.02 & 0.17 $\pm$ 0.01 / 0.82 $\pm$ 0.05 & 0.82 $\pm$ 0.05 & 0.86 $\pm$ 0.02 & 0.86 $\pm$ 0.02 & 0.35 $\pm$ 0.02 \\
SMOTE + NonPriv. Weighted LogReg & 0.92 $\pm$ 0.02 & 0.19 $\pm$ 0.01 & 0.85 $\pm$ 0.02 & 0.11 $\pm$ 0.01 / 0.88 $\pm$ 0.03 & 0.82 $\pm$ 0.01 & 0.85 $\pm$ 0.02 & 0.85 $\pm$ 0.02 & 0.27 $\pm$ 0.01 \\
SMOTE + NonPriv. Weighted XGB & 0.92 $\pm$ 0.02 & 0.66 $\pm$ 0.04 & 0.87 $\pm$ 0.02 & 0.58 $\pm$ 0.06 / 0.75 $\pm$ 0.04 & 0.75 $\pm$ 0.04 & 0.87 $\pm$ 0.02 & 0.86 $\pm$ 0.02 & 0.65 $\pm$ 0.04 \\
SMOTE + NonPriv. XGBoost & 0.93 $\pm$ 0.02 & 0.65 $\pm$ 0.04 & 0.86 $\pm$ 0.03 & 0.59 $\pm$ 0.05 / 0.73 $\pm$ 0.05 & 0.73 $\pm$ 0.05 & 0.86 $\pm$ 0.03 & 0.85 $\pm$ 0.03 & 0.65 $\pm$ 0.04 \\
Identity + NonPriv. Weighted FTTransformer & 0.88 $\pm$ 0.03 & 0.21 $\pm$ 0.16 & 0.59 $\pm$ 0.08 & 0.40 $\pm$ 0.34 / 0.19 $\pm$ 0.17 & 0.19 $\pm$ 0.17 & 0.59 $\pm$ 0.08 & 0.23 $\pm$ 0.23 & 0.16 $\pm$ 0.16 \\
\midrule
\bottomrule
\end{tabular}
}
\end{table}

\begin{table}[H]
\centering
\caption{Abolone\_19 Dataset}
\label{tab:bench_results_abalone_19}
\resizebox{\columnwidth}{!}{
\begin{tabular}{lcccccccc}
\toprule
\multicolumn{1}{c}{\multirow{2}{*}{\textbf{Metrics}}} & \multicolumn{4}{c}{\textbf{Standard} $\uparrow$} & \multicolumn{4}{c}{\textbf{Imbalanced $\uparrow$}} \\
\cmidrule(lr){2-5} \cmidrule(lr){6-9}
\textbf{Approach} & \textbf{AUC} & \textbf{F1} & \textbf{Bal-ACC} & \textbf{Prec./Recall} & \textbf{Worst-ACC} & \textbf{Avg-ACC} & \textbf{G-Mean} & \textbf{MCC} \\
\midrule
\midrule
\textbf{\textbf{Non-Private} $\downarrow$} &  &  &  &  &  & &  & \\
\midrule
\midrule
Identity + NonPriv. LogReg & 0.75 $\pm$ 0.11 & 0.0 $\pm$ 0.0 & 0.5 $\pm$ 0.0 & 0.0 $\pm$ 0.0 / 0.0 $\pm$ 0.0 & 0.0 $\pm$ 0.0 & 0.5 $\pm$ 0.0 & 0.0 $\pm$ 0.0 & 0.0 $\pm$ 0.0 \\
Identity + NonPriv. Weighted LogReg & 0.72 $\pm$ 0.11 & 0.06 $\pm$ 0.05 & 0.59 $\pm$ 0.1 & 0.03 $\pm$ 0.03 / 0.25 $\pm$ 0.18 & 0.25 $\pm$ 0.18 & 0.59 $\pm$ 0.1 & 0.42 $\pm$ 0.25 & 0.07 $\pm$ 0.07 \\
Identity + NonPriv. XGBoost & 0.72 $\pm$ 0.12 & 0.03 $\pm$ 0.09 & 0.51 $\pm$ 0.03 & 0.1 $\pm$ 0.32 / 0.02 $\pm$ 0.05 & 0.02 $\pm$ 0.05 & 0.51 $\pm$ 0.03 & 0.04 $\pm$ 0.13 & 0.04 $\pm$ 0.13 \\
Identity + NonPriv. Weighted XGBoost & 0.76 $\pm$ 0.11 & 0.04 $\pm$ 0.04 & 0.55 $\pm$ 0.07 & 0.02 $\pm$ 0.02 / 0.15 $\pm$ 0.15 & 0.15 $\pm$ 0.15 & 0.55 $\pm$ 0.07 & 0.29 $\pm$ 0.26 & 0.04 $\pm$ 0.06 \\
SMOTE + NonPriv. LogReg & 0.82 $\pm$ 0.06 & 0.05 $\pm$ 0.01 & 0.76 $\pm$ 0.08 & 0.02 $\pm$ 0.0 / 0.72 $\pm$ 0.18 & 0.68 $\pm$ 0.13 & 0.76 $\pm$ 0.08 & 0.75 $\pm$ 0.09 & 0.11 $\pm$ 0.03 \\
SMOTE + NonPriv. Weighted LogReg & 0.83 $\pm$ 0.06 & 0.03 $\pm$ 0.0 & 0.73 $\pm$ 0.06 & 0.02 $\pm$ 0.0 / 0.8 $\pm$ 0.13 & 0.66 $\pm$ 0.02 & 0.73 $\pm$ 0.06 & 0.73 $\pm$ 0.05 & 0.08 $\pm$ 0.02 \\
SMOTE + NonPriv. Weighted XGB & 0.79 $\pm$ 0.1 & 0.08 $\pm$ 0.06 & 0.57 $\pm$ 0.06 & 0.06 $\pm$ 0.04 / 0.17 $\pm$ 0.11 & 0.17 $\pm$ 0.11 & 0.57 $\pm$ 0.06 & 0.36 $\pm$ 0.2 & 0.09 $\pm$ 0.07 \\
SMOTE + NonPriv. XGBoost & 0.78 $\pm$ 0.09 & 0.08 $\pm$ 0.07 & 0.57 $\pm$ 0.07 & 0.06 $\pm$ 0.05 / 0.15 $\pm$ 0.15 & 0.15 $\pm$ 0.15 & 0.57 $\pm$ 0.07 & 0.31 $\pm$ 0.23 & 0.08 $\pm$ 0.08 \\
Identity + NonPriv. Weighted FTTransformer & 0.56 $\pm$ 0.09 & 0.01 $\pm$ 0.01 & 0.53 $\pm$ 0.10 & 0.00 $\pm$ 0.01 / 0.20 $\pm$ 0.38 & 0.12 $\pm$ 0.22 & 0.53 $\pm$ 0.10 & 0.29 $\pm$ 0.29 & 0.04 $\pm$ 0.04 \\
\midrule
\bottomrule
\end{tabular}
}
\end{table}

\end{document}